\RequirePackage{fix-cm}
\documentclass[letterpaper,10pt]{article}
\usepackage{graphicx}

\usepackage{hyperref}
\renewcommand{\url}[1]{\href{#1}{#1}} 

\usepackage{color}

\newcommand{\todo}[1]{{#1}}
\newcommand{\finaltodo}[1]{{#1}}

\usepackage{natbib}
\usepackage{amsmath}
\usepackage{amssymb}
\usepackage{url, units}

\usepackage{authblk}

\usepackage[caption=false, font=footnotesize,subrefformat=parens,labelformat=parens]{subfig}
\graphicspath{{./Figures/}}

\usepackage{mathtools}

\usepackage{algpseudocode,algorithm,algorithmicx}

\algrenewcommand\algorithmicrequire{\textbf{Precondition:}}

\usepackage{soul}   %strike through text

\usepackage{graphicx}

\usepackage{color}

\usepackage{tabularx}
\usepackage{multirow}  

\usepackage{amsthm}
\newtheorem{theorem}{Theorem}
\newtheorem{corollary}{Corollary}

\begin{document}

\title{Turing learning: a metric-free approach to inferring behavior and its application to swarms\let\thefootnote\relax\footnotetext{All authors have contributed equally to this work.}}

\author[1]{Wei Li\thanks{wei.li11@sheffield.ac.uk}}
\author[2]{Melvin Gauci\thanks{mgauci@g.harvard.edu}}
\author[1]{Roderich Gro{\ss}\thanks{r.gross@sheffield.ac.uk (corresponding author)}}

\affil[1]{Department of Automatic Control and Systems Engineering, The~University of Sheffield,
% Mappin Street, S1 3JD, 
Sheffield, UK}
\affil[2]{Wyss Institute for Biologically Inspired Engineering, Harvard~University, 
%3 Blackfan Cir, Boston, MA 02115, 
Boston, MA}

\renewcommand\Authands{ and }

%\thanks{All authors contributed equally to this work.}
%\thanks{Wei Li $\cdot$ Roderich Gro{\ss}, 
%              Department of Automatic Control and Systems Engineering, The University of Sheffield, Mappin Street, S1 3JD, Sheffield, UK \\
%              \email{\{wei.li11, r.gross\}}@sheffield.ac.uk       
%           \and
%             Melvin Gauci, 
%             Wyss Institute for Biologically Inspired Engineering, Harvard University, 3 Blackfan Cir, Boston, MA 02115, USA \\
%             \email{mgauci@seas.harvard.edu}
%}
%}
\date{}

\maketitle

\begin{abstract}
We propose~\textit{Turing Learning}, a novel \todo{system identification} method for inferring the behavior of natural or artificial systems. \textit{Turing Learning} simul\-ta\-neous\-ly optimizes two populations of computer programs, one representing \textit{models} of the behavior of the system under investigation, and the other representing \textit{classifiers}.  By observing the behavior of the system as well as the behaviors produced by the models, two sets of data samples are obtained. The classifiers are rewarded for discriminating between these two sets, that is, for correctly categorizing data samples as either genuine or counterfeit.  Conversely, the models \todo{are rewarded for `tricking'} the classifiers into categorizing their data samples as genuine. Unlike other methods for system identification, \textit{Turing Learning} does not require \todo{predefined} metrics to quantify the difference between the system and its models. We present two case studies with swarms of simulated robots and prove that the underlying behaviors cannot be inferred by a metric-based system identification method. By contrast, \textit{Turing Learning} infers the behaviors with high accuracy.  It also produces a useful by-product---the classifiers---that can be used to detect abnormal behavior in the swarm. Moreover, we show that \textit{Turing Learning} also successfully infers the behavior of physical robot swarms. The results show that collective behaviors can be directly inferred from motion trajectories of \textcolor{black}{individuals} in the swarm, which may have significant implications for the study of animal collectives. \textcolor{black}{Furthermore, \textit{Turing Learning} could prove useful whenever a behavior is not easily characterizable using metrics, making it suitable for a wide range of applications.}
\\[0.3cm]
\textbf{Keywords}: system identification; Turing test; collective behavior; swarm robotics; coevolution; machine learning
\end{abstract}

\section{Introduction}\label{sec:introduction}

System identification is the process of modeling natural or artificial systems through observed data. It has drawn a large interest among researchers for decades~\citep{Ljung2010, Billings2013}. 
%One application of system identification is the reverse engineering of agent behavior (biological organisms or artificial agents).
A limitation of current system identification methods is that they rely on predefined metrics, such as the sum of square errors, to measure the difference between the output of the models and that of the system under investigation. Model optimization then proceeds by minimizing the measured differences. However, for complex systems, defining a metric can be non-trivial and case-dependent. It may require prior information about the systems. Moreover, an unsuitable metric may not distinguish well between good and bad models, or even bias the identification process. This paper overcomes these problems by introducing a system identification method that does not rely on predefined metrics.

A promising application of such a metric-free method is the identification of collective behaviors, which are emergent behaviors that arise from the interactions of numerous simple individuals~\citep{Camazine03}. Inferring collective behaviors 
is particularly challenging, as the individuals not only interact with the environment but also with each other. Typically, their motion appears stochastic and is hard to predict~\citep{Dirk2011}. For instance, given a swarm of simulated fish, one would have to evaluate how close its behavior is to that of a real fish swarm, or how close the individual behavior of a simulated fish is to that of a real fish. Characterizing the behavior at the level of the swarm is difficult~\citep{Harvey:SI:2015}. Such a metric may require domain-specific knowledge; moreover, it may not be able to discriminate among distinct individual behaviors that lead to similar collective dynamics~\citep{Weitz2012}. Characterizing the behavior at the level of individuals is also difficult, as even the same individual fish in the swarm is likely to exhibit a different trajectory every time it is being looked at. 
%It may require domain-specific knowledge and not discriminate among alternative individual rules that exhibit similar collective dynamics~\citep{Weitz2012}.

In this paper, we propose \textit{Turing Learning}, a novel system identification method that allows a machine to \textcolor{black}{autonomously} infer the behavior of a natural or artificial system.
%. The system
% \textcolor{black}{The agent 
%under investigation could be 
%natural or artificial.
\textit{Turing Learning} simul\-ta\-neous\-ly optimizes two populations of computer programs, one representing \textit{models} of the behavior, the other representing \textit{classifiers}. 
% The \todo{method optimizes two populations simultaneously}: one of \textit{models}, and the other of \textit{classifiers}. 
The purpose of the models is to imitate the behavior of the system under investigation. The purpose of the classifiers is to discriminate between the behaviors produced by the system and any of the models. In \textit{Turing Learning}, all behaviors are observed for a period of time. This generates two sets of data samples. The first set consists of \textit{genuine} data samples, which originate from the system. The second set consists of \textit{counterfeit} data samples, which originate from the models.
% ones), and those originating from the models (the counterfeit. 
The classifiers are rewarded for discriminating between samples of these two sets: Ideally, they should recognize any data sample from the system as genuine, and any data sample from the models as  counterfeit. 
%n other words, given a data sample from an observation, the classifiers shall correctly categorize it as either genuine or counterfeit.
%the system and the models are observed for a period of time, during which data samples are gathered about their behavior. These data samples are in turn fed to the classifiers. For each sample, a classifier} outputs a Boolean value indicating \textcolor{black}{whether} the sample is believed to \textcolor{black}{originate} from \textcolor{black}{the agent or a model.}
%\todo{The classifier receives a reward} if and only if it makes the correct judgment. 
%\todo{Their quality} thus depends solely on their ability to discriminate between the system and \todo{models}. 
Conversely, the models are rewarded for their ability to `trick' the classifiers into categorizing their data samples as genuine.

\textit{Turing Learning} does not rely on predefined metrics for measuring how close the models reproduce the behavior of the system under investigation;
%the similarity between behavior produced by the system and by its models; 
rather, the metrics (classifiers) are produced as a by-product of the identification process. The method is inspired by the Turing test~\citep{Turing1950,Pinar-etal2000:minds_machine,Harnad2000}, which machines can pass if behaving indistinguishably from humans. Similarly, the models could pass the tests by the classifiers if behaving indistinguishably from the system under investigation. We hence call our method~\textit{Turing Learning}.

In the following, we examine the ability of \textit{Turing Learning} to infer the behavioral rules of a swarm of mobile agents. The agents are either simulated or physical robots. They execute known behavioral rules.  
%part of a homogenous swarm. The agents are artificial---they are the robots of} an existing swarm robotic system. 
This allows us to compare the inferred models to the ground truth. To obtain the data samples, we record the motion trajectories of all the agents. In addition, we record the motion trajectories of an agent \textit{replica}, which is mixed into the group of agents. The replica executes the rules defined by the models---one at a time.
As will be shown, by observing the motion trajectories of agents and of the agent replica,
%\todo{(but not knowing the sensory values that produced these)}, 
\textit{Turing Learning} automatically infers the behavioral rules of the agents.
% underpinning the collective behaviors. 
The behavioral rules examined here 
%The behavioral rules defined by the model \textcolor{black}{are executed in an agent \textit{replica}}, and \textcolor{black}{the replica} also produces data samples (motion trajectories). 
%We present two case studies that  
relate to two canonical problems in swarm robotics: self-organized aggregation~\citep{Gauci2014_ijrr}, and object clustering~\citep{Melvin2014_aamas}. 
They are reactive; in other words, each agent maps its inputs (sensor readings) directly onto the outputs (actions). 
The problem of inferring the mapping is challenging, as the inputs are not known.
%\footnote{Strictly speaking, the outputs---used to control the agent's actuators---are not known either. However, reasonable estimates of these can be obtained from the observed motion trajectories of the agents.}
Instead, \textit{Turing Learning} has to infer the mapping indirectly, from the observed motion trajectories of the agents and of the replica.

We originally presented the basic idea of \textit{Turing Learning}, along with preliminary simulations, in~\citep{Li-etal2013:proc_gecco,Li-etal2014:proc_gecco}. This paper extends our prior work as follows:
%Some preliminary results were presented using computer simulation in an earlier work~\citep{Li-etal2014:proc_gecco}. \todo{The contributions of this paper are as follows:}
\begin{itemize}
\renewcommand{\labelitemi}{\scriptsize$\bullet$} 
\item It presents an algorithmic description of \textit{Turing Learning};
\item It shows that \textit{Turing Learning} outperforms a metric-based system identification method in terms of model accuracy;
%that the classifiers are a useful byproduct of systematically investigates the evolution of classifiers, showing that the classifiers can be used as a byproduct;
\item It proves that the metric-based method is fundamentally flawed, as the globally optimal solution differs from the solution that should be inferred;
\item It demonstrates, to the best of our knowledge for the first time, that system identification can infer the behavior of swarms of physical robots;
\item It examines in detail the usefulness of the classifiers;
\item It examines through simulation how \textit{Turing Learning} can simultaneously infer the agent's brain (controller) and \textcolor{black}{an aspect of its} morphology that determines the agent's field of view;
\item \textcolor{black}{It demonstrates through simulation that \textit{Turing Learning} can infer the behavior even if the agent's control system structure is unknown.}
\end{itemize}

This paper is organized as follows. Section~\ref{sec:related_work} discusses related work. Section~\ref{sec:methodology} describes \textit{Turing Learning} and the general methodology of the two case studies. Section~\ref{sec:results_simulation} investigates the ability of \textit{Turing Learning} to infer two behaviors of swarms of simulated robots. It also presents a mathematical analysis, proving that these behaviors cannot be inferred by a metric-based system identification method. Section~\ref{sec:physical_implementation} presents a real-world validation of~\textit{Turing Learning} with a swarm of physical robots. Section~\ref{sec:conclusion} concludes the paper. 

\section{Related work}\label{sec:related_work}
\textcolor{black}{This section is organized as follows. First, we outline our previous work on \textit{Turing Learning}, and review a similar line of research, which has appeared since its publication. As the \textit{Turing Learning} implementation uses coevolutionary algorithms, we then overview work using coevolutionary algorithms (but with predefined metrics), as well as work on the evolution of physical systems. Finally, works using replicas in ethological studies are presented.}

\textit{Turing Learning} is a system identification method that simultaneously optimizes a population of models and a population of classifiers. The objective for the models is to be indistinguishable from the system under investigation. The objective for the classifiers is to distinguish between the models and the system. The idea of \textit{Turing Learning} was first proposed in~\citep{Li-etal2013:proc_gecco}; this work presented a coevolutionary approach for inferring the behavioral rules of a single agent. The agent moved in a simulated, one-dimensional environment. Classifiers were rewarded for distinguishing between the models and the agent. In addition, they were able to control the stimulus that influenced the behavior of the agent. This allowed the classifiers to interact with the agent during the learning process. 
%In this paper, we extend the framework in~\citep{Li-etal2013:proc_gecco} to learn collective behaviors of groups of agents in both simulated and physical worlds. 
\textit{Turing Learning} was subsequently investigated with swarms of simulated robots~\citep{Li-etal2014:proc_gecco}.

\citet{Goodfellow:NIPS:2014} proposed \textit{generative adversarial nets} (GANs). GANs, while independently invented, are essentially based on the same idea as \textit{Turing Learning}. The authors
% optimizing a population of models and a population of classifiers to was originally presented in~\cite{Li-etal2013:proc_gecco,Li-etal2014:proc_gecco}. It was independetly rediscovered in a different context by~\citep{Goodfellow:NIPS:2014}, who 
used GANs to train models for generating counterfeit images that resemble real images, for example, from the Toronto Face Database (for further examples, see~\citep{Radford-etal:iclr2016}).
% A gradiend-descent method was used to tune both classifiers as well as models. The input data were images of real hotels (or real human faces).
They simultaneously optimized a generative model \todo{(producing counterfeit images)} and a discriminative model that estimates the probability of an image to \todo{be real}. \todo{The optimization was done using a stochastic gradient descent method}. 
%\todo{The optimization of the model types needs} to be synchronized, which is manually achieved. 
%By contrast, in \textit{Turing Learning}, no humans are needed; instead the artificial observers (classifiers) coevolve with the models.} 
% TODO: synchronisation also needed for coevolution, so not mention? Theoretical results from NIPS 2014?

In a work reported in \citep{Herbert-Read2015}, humans were asked to discriminate between the collective motion of real and simulated fish. The authors reported that the humans could do so even though the data from the model were consistent with the real data according to predefined metrics. Their results ``highlight a limitation of fitting detailed models to real-world data''. They argued that ``observational tests [...] could be used to cross-validate models'' (see also~\citet{Harel2005:nat_biotechnol}). This is in line with \textit{Turing Learning}. Our method, however,  automatically generates both the models and the classifiers, and thus does not require human observers.

While \textit{Turing Learning} can in principle be used with any optimization algorithm, our implementation relies on coevolutionary algorithms. Metric-based coevolutionary algorithms have already \todo{proven} effective for system identification~\citep{ Bongard_remote_robot_2004,Bongard_function_recovery_2004, Bongard2005,Bongard2007PNAS,Koos2009,Mirm2011,Ly2014}. A range of work has been performed on simulated agents. \citet{Bongard_remote_robot_2004} proposed the \emph{estimation-exploration algorithm}, a nonlinear system identification method to coevolve inputs and models in a way that minimizes the number of inputs to be tested on the system. In each generation, the input (test) that led, in simulation, to the highest disagreement between the models' predicted outputs was carried out on the real system. The models' predictions were then compared with 
%quality of models was evaluated through quantitatively comparing their predictions with 
the actual output of the system. The method was applied to evolve morphological parameters of a simulated quadrupedal robot after it had undergone physical damage. In a later work~\citep{Bongard_function_recovery_2004}, the authors reported that ``in many cases the simulated robot would exhibit wildly different behaviors even when it very closely approximated the damaged `physical' robot. This result is not surprising due to the fact that the robot is a highly coupled, non-linear system: Thus similar initial conditions [...] are expected to rapidly diverge in behavior over time''. \todo{The authors} addressed this problem by using a more refined comparison metric reported in~\citep{Bongard_function_recovery_2004}. In~\citep{Koos2009}, an algorithm which also coevolves models and inputs (tests) was presented to model a simulated quadrotor and improve the control quality. The tests were selected based on multiple criteria:
%objective performances (e.g., 
to provide disagreement between models as in~\citep{Bongard_remote_robot_2004}, and to evaluate the control quality in a given task. Models were then refined by comparing the predicted trajectories with those of the real system. In these works, predefined metrics were critical for evaluating the performance of models. Moreover, the algorithms are not applicable to the scenarios we consider here, as the system's inputs are assumed to be unknown
% are no exogenous inputs in the swarm systems under observation 
(the same would typically also be the case for biological systems). %Moreover, the model optimization is based on predefined metrics (explicit or implicit).
% TODO: This seems to be true for animals. But for our model? Yes, if we say we are just observing?

%TODO: arXiv.org submission?
Some studies also investigated the implementation of evolution directly in physical environments, on either a single robot~\citep{Floreano1996, Zykov2004, Bongard-etal2006:science, Koos2013, Cully2015} or multiple robots~\citep{Watson2002, Alan2014, Heinerman:GECCO:2015}. In~\citep{Bongard-etal2006:science}, a four-legged robot was built to study how it can infer its own morphology through a process of continuous self-modeling. The robot ran a coevolutionary algorithm on its onboard processor. One population evolved models for the robot's morphology, while the other evolved actions (inputs) to be conducted on the robot for gauging the quality of these models. Note that this approach required knowledge of the robot's inputs (sensor data). \citet{Alan2014} presented a distributed approach to coevolve onboard simulators and controllers for a swarm of ten robots. Each robot used its simulators to evolve controllers for performing foraging behavior.
%The evolution of each robot's simulator was driven by comparing the real-world foraging efficiency 
%of the robot and its nearby neighbors each executing the best controller generated by their own simulators. 
%Each robot had a population of controllers, which evolved according to the robot's on-board simulator. 
The best performing controller was then used to control the physical robot.
%for performing real-world foraging. 
The foraging performances of the robot and of its neighbors were then compared to inform the evolution of simulators.
This physical/embodied evolution helped reduce the \textit{reality gap} between the simulated and physical environments~\citep{Jakobi95}. In all these approaches, the model optimization was based on predefined metrics (explicit or implicit).
%The common feature for the above approaches is that the optimization of models (simulators) are based on predefined metrics, which are thus task dependent. However, when studying an unknown agent behavior, which is the scenario that we consider here, it is difficult to quantitatively compare the behaviors of models and agents due to the `stochastic' feature of the agents. Therefore, in our coevolutionary approach, the classifiers are only evolved to make judgment (like a human observer) rather than quantitative comparison.
%In~\citep{Watson2002}, a parallel evolutionary algorithm was distributively implemented on eight robots and the goal was evolving a reactive controller to perform a simple light-seeking behavior. Each robot carried a genome (controller) and the exchange of genes was realized through local broadcasts using the robots' infrared sensors. 
%In all of the above approaches, the model optimization is based on predefined metrics (explicit or implicit), which are task dependent.
  
The use of replicas can be found in ethological studies in which researchers use robots that interact with animals~\citep{ Vaughan2000, J.Halloy2007, Faria2010, Halloy2013, Thomas2013}. Robots can be created and systematically controlled in such a way that they are accepted as conspecifics or heterospecifics by the animals in the group~\citep{Krause2011}. For example,  in~\citep{Faria2010}, a replica fish, which resembled sticklebacks in appearance, was created to investigate two types of interaction: recruitment and leadership. In~\citep{J.Halloy2007}, autonomous robots, which executed a model, were mixed into a group of cockroaches to modulate their decision-making of selecting a shelter. The robots behaved in a similar way to the cockroaches. Although the robots' appearance was different from that of the cockroaches, the robots released a specific odor such that the cockroaches would \todo{perceive} them as conspecifics. In these works, the models were manually derived and the robots were only used for model validation. We believe that this robot-animal interaction framework could be enhanced through \textit{Turing Learning}, which autonomously infers the collective behavior.

\section{\todo{Methodology}}\label{sec:methodology}

In this section, we present the~\textit{Turing Learning} method and \todo{show how it can be applied to two case studies in swarm robotics.}
%: self-organized aggregation~\citep{Gauci2014_ijrr} and object clustering~\citep{Melvin2014_aamas}. The description of the case studies illustrates how \textit{Turing Learning} can be applied
%In both case studies, the agents execute simple behavioral rules that lead to meaningful emergent behaviors on a global level. 

\subsection{Turing learning}\label{sec:turing_learning}

\textit{Turing Learning} is a system identification method for inferring \textcolor{black}{the} behavior of natural or artificial systems. % a system \textcolor{black}{under investigation. It} 
\textit{Turing Learning} needs data samples of the behavior---we refer to these data samples as \textit{genuine}. For example, if the behavior of interest \textcolor{black}{were} to shoal like fish, genuine data samples could be trajectory data from fish. If the behavior \textcolor{black}{were} to \textcolor{black}{produce paintings} in a particular style (e.g., Cubism), genuine data samples could be existing paintings \textcolor{black}{in} this style.

\textit{Turing Learning} simul\-ta\-neous\-ly optimizes two populations of computer programs, one representing \textit{models} of the behavior, the other representing \textit{classifiers}. 
% models
The purpose of the models is to imitate the behavior of the system under investigation. 
The models are used to produce data samples---we refer to these data samples as \textit{counterfeit}. There are thus two sets of data samples: one containing genuine data samples, the other containing counterfeit ones. 
% classifiers
The purpose of the classifiers is to discriminate between these two sets. 
%In other words, behaviors produced by the system and by any model. 
%Simultaneously to \textcolor{black}{generating models}, \textit{Turing Learning} \textcolor{black}{also} generates classifiers. 
Given a data sample, the classifiers need to judge where it comes from.
Is it genuine, and thus originating from the system under investigation? Or is it counterfeit, and thus originating from a model?
%or a model fish? Does it originate from a Cubist painter or a model painter? 
This setup is akin of a Turing test; hence the name \textit{Turing Learning}.

The models and classifiers are competing. The models are rewarded for `tricking' the classifiers into categorizing their data samples as genuine, 
whereas the classifiers are rewarded for correctly categorizing data samples as either genuine or counterfeit. \textit{Turing Learning} thus optimizes models for producing behaviors that are seemingly genuine, in other words, \textit{indistinguishable} from the behavior of interest. This is in contrast to other system identification methods, which optimize models for producing behavior that is \textit{as similar as possible} to the behavior of interest. The Turing test inspired setup allows for model generation \todo{irrespective of whether suitable similarity metrics are known.}

The model can be any computer program that can be used to produce data samples.
%In principle, the model can take any form. The model
It must however be expressive enough to produce data samples that---from an observer's per\-spec\-tive---are indistinguishable from those of the system. 

The classifier can be any computer program that takes a sequence of data as input and produces a binary output. The classifier must be fed with sufficient information about the behavior of the system. If it has access to only a subset of the behavioral information, any system characteristic not influencing this subset cannot be learned. In principle, classifiers with non-binary
%more complex classifier 
outputs (e.g., probabilities or confidence levels) could also be considered.

\begin{algorithm}[t]
  \caption{Turing Learning 
    \label{alg:turing_learning}}
  \begin{algorithmic}[1]
%    \Statex
    \Procedure{\textsc{Turing learning}}{}
     \State initialize population of $M$ models and population of $N$ classifiers
      \While{termination criterion not met}
		 \ForAll {classifiers $i \in \{1, 2, \ldots, N\}$}
		  \State obtain genuine data samples (system, \textcolor{black}{classifier $i$}) 
		  \State for each sample, obtain and store output of classifier $i$
%		  \State $\textsc{specificity}_i$ = true negative rate of classifier $i$ on system samples

		 	\ForAll {models $j \in \{1, 2, \ldots, M\}$}
		 	 \State obtain counterfeit data samples (model $j$, \textcolor{black}{classifier $i$}) 
		 	 \State	for each sample, obtain and store output of classifier $i$
		 	\EndFor
		 	
%		 \State $\textsc{sensitivity}_i$ = true positive rate of classifier $i$ on model samples	
%		 \State $c_i =  (\textsc{sensitivity}_i + \textsc{specificity}_i)/2$		
		\EndFor
		
%		 \ForAll {models $j \in \left[1, M\right]$}
%		   \State $m_j$ = proportion of classifiers with negative output for model $j$ 
%                \State assign rewards to models and classifiers (classifier outputs)
                \State \todo{reward models ($r_m$)} for misleading classifiers (classifier outputs)
                \State \todo{reward classifiers ($r_c$)} for making correct judgements (classifier outputs)
	    %\EndFor
	   \State improve model and classifier populations based \todo{on $r_m$ and $r_c$}
%      \State optimize models% based on rewards $m$
      \EndWhile
    \EndProcedure
  \end{algorithmic}
\end{algorithm}

Algorithm~\ref{alg:turing_learning} provides a description of \textit{Turing Learning}. We assume a population of \textcolor{black}{$M > 1$} models and a population of \textcolor{black}{$N > 1$} classifiers. \todo{After an initialization stage,
%At the beginning these populations are initialized. 
\textit{Turing Learning} proceeds in an iterative manner until a termination criterion is met.}

In each iteration cycle, data samples are obtained from observations of both the system and the models.   
In the case studies considered here, the classifiers do not influence the sampling process. Therefore, the same set of data samples is provided to all classifiers of an iteration cycle.\footnote{In general, the classifiers may influence the sampling process. In this case, independent data samples should be generated for each classifier. In particular, the classifiers could change the stimuli that influence the behavior of the system under investigation. This would enable a classifier to interact with the system by choosing the conditions under which the behavior is observed~\citep{Li-etal2013:proc_gecco}. The classifier could then extract hidden information about the system, which may not be revealed through passive observation alone~\citep{LiThesis2016}.}
For simplicity, we assume that each of the $N$ classifiers is provided with \textcolor{black}{$K \geq 1$} data samples for the system and with one data sample for every model. 
%\textcolor{black}{Suppose the replica that executes a model is mixed into a group of $n$ agents; then $K$ would be equal to $n \times M$, where $M$ is the total number of models.}  
%\footnote{Where classifiers are provided with multiple data samples for each model, the reward assignment gives equal weight to each of them.} 
%The outputs of all classification tests are stored. 

A model's quality is determined by its ability of misleading classifiers to judge its data samples as genuine. Let $m_{ij}=1$ if classifier $i$ \textcolor{black}{wrongly} classified the data sample of model $j$, and $m_{ij}=0$ otherwise. 
%The quality of each model in a sample is determined by each of the $N$ classifiers in the competing population. 
%For every classifier that wrongly judges the model as an agent, the model's quality increases by $1$. The model's quality is then normalized to $[0, 1]$. 
The quality of model $j$ is then given by:
\begin{equation}\label{eq:model_fitness_definition}
{r_m(j)} = \frac{1}{N} \sum \limits_{i=1}^{N} m_{ij}.
\end{equation}

A classifier's quality is determined by how well it judges data samples from both the system and its models. 
%By default, the reward is defined by its specificity and sensitivity. 
The quality of classifier $i$ is given by:
\begin{equation}\label{eq:classifier_fitness_definition}
{r_c(i)} = \frac{1}{2} (\mbox{specificity}_{i}+\mbox{sensitivity}_{i}). 
\end{equation}

\textcolor{black}{The \textit{specificity} of a classifier (in statistics, also called the true-negative rate) denotes the percentage of genuine data samples that it correctly identified as such.} Formally,
%The specificity of a classifier denotes its true negative rate: the percentage of agent data samples that were correctly identified as such (i.e., the tests were negative).
%The quality of each classifier is determined by its judgments for all models and agents. It comprises two values: \textit{sensitivity} and \textit{specificity}. For each correct (positive) judgment of the model in a sample, the classifier's quality increases by $\frac{1}{n_r}$; after being evaluated against all models, the classifier's \textit{sensitivity} value, which is its true positive judgment rate, is calculated. For each correct (negative) judgment of the agent, the classifier's quality increases by $\frac{1}{n_a}$. The classifier's \textit{specificity} value, which is its true negative judgment rate is also calculated. The quality value of each classifier is then normalized to $[0, 1]$. For classifier $i$, its quality is defined as follows:
\begin{equation}\label{eq:classifier_specificit_definition}
\mbox{specificity}_{i} = \frac{1}{K} \sum \limits_{k=1}^{K} a_{ik},
\end{equation}
where, $a_{ik} = 1$ if classifier $i$ \textcolor{black}{correctly} classified the $k$th data sample of the system, and $a_{ik} = 0$ otherwise.

%The sensitivity of a classifier denotes its true positive rate: the percentage of model data samples that were correctly identified as such (i.e., the tests were positive).
\textcolor{black}{The \textit{sensitivity} of a classifier (in statistics, also called the true-positive rate) denotes the percentage of counterfeit data samples that it correctly identified as such.} Formally,
\begin{equation}\label{eq:classifier_sensitivity_definition}
\mbox{sensitivity}_{i} = \frac{1}{M} \sum \limits_{j=1}^{M} (1-m_{ij}).
\end{equation}
%r_{i} = \frac{1}{M} \sum \limits_{j=1}^{M} f_{ij}}

%+ \frac{1}{M \cdot n_{a}} \sum \limits_{j=1}^{M} \sum \limits_{k=1}^{n_{a}} f_{a_{jk}} \right\}.

% The sensitivity of a classifier denotes its true positive rate: the percentage of model data samples that were correctly identified as such (test positive). Based on these reward measures, the classifier and model populations are optimized. By default, equal weight is given to specificity and sensitivity. 

Using the solution qualities, $r_m$ and $r_c$, the model and classifier populations are improved. In principle, any population-based optimization method can be used.

\subsection{Case studies}\label{sec:case_studies}

In the following, we examine the ability of \textit{Turing Learning} to infer the behavioral rules of swarming \textit{agents}. The swarm is assumed to be homogeneous; it comprises a set of identical agents of known capabilities. The identification task thus reduces to inferring the behavior of a single agent. The agents are robots, either simulated or physical. The agents have inputs (corresponding to sensor reading values) and outputs (corresponding to motor commands). The input and output values
are not known. However, the agents are observed and their motion trajectories are recorded. The trajectories are provided to \textit{Turing Learning} using a reactive control architecture~\citep{Brooks1991:artif_int}. Evidence indicates that reactive behavioral rules are sufficient to produce a range of complex collective behaviors in both groups of natural and artificial agents~\citep{Braitenberg1984,ArkinBehBasedRob1986,Camazine03}. Note that although reactive architectures are conceptually simple, learning their parameters is not trivial if the agent's inputs are not available, as is the case in our problem setup. In fact, as shown in Section~\ref{sec:metric-based_EA}, a conventional (metric-based) system identification method fails in this respect.

%\textcolor{black}{Unless otherwise stated we assume that the agent control architecture is known (gray box approach), and that a set of model parameters---determining how the agent's inputs are mapped onto its outputs---is to be inferred.
%\textcolor{black}{This approach is believed to be sufficient to reproduce many  of the complex swarm behaviors observed in nature~\citep{Scott04}.} 
% and they can be represented by a set of parameters. 

\subsubsection{\todo{Agents}}\label{sec:problem_formulation}

\begin{figure}
	\centering
	\includegraphics[width=2.5in]{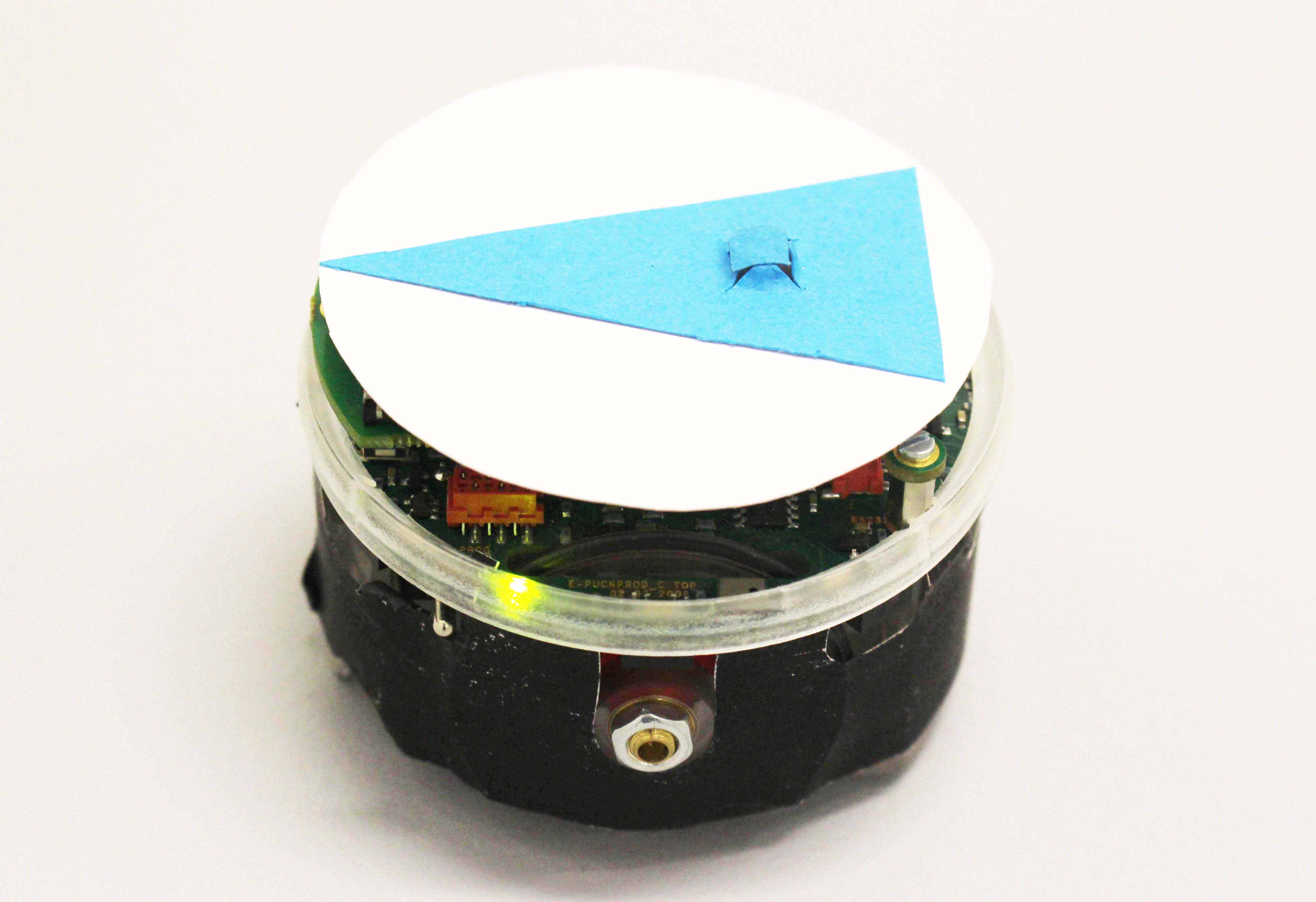}  %2.6
	\caption{An e-puck robot fitted with a black `skirt' and a top marker for motion tracking.}
	\label{fig:e-puck_body}
\end{figure}

The \todo{agents move} in a two-dimensional, continuous space. \todo{They are differential-wheeled robots. The speed of each wheel can be independently set to $\left[-1,1\right]$, where $-1$ and $1$ correspond to the wheel rotating backwards and forwards, respectively, with maximum speed. Fig.~\ref{fig:e-puck_body} shows the agent platform, the e-puck~\citep{e-puck}, which is used in the experiments.}
%Their embodiment is based on the e-puck~\citep{e-puck}, which is a miniature, differential-wheeled 
%robot. 
%The agents are embodied and move in a two-dimensional, continuous space. The agents' embodiment is based on the e-puck~\citep{e-puck}, which is a miniature, differential-wheeled robot. Fig.~\ref{fig:e-puck_body} shows an e-puck robot used in the physical experiments. 

Each agent is equipped with a line-of-sight sensor that \todo{detects the type of item in front of it}. We assume that there are $n$ types (e.g., background, other agent, object~\citep{Gauci2014_ijrr, Melvin2014_aamas}). The state of the sensor is denoted by $I\in\{0,1,\ldots,n-1\}$.
%In the object clustering case study, the objects to be clustered can also be distinguished by the agents.
%In the case of object clustering, the objects are also embodied, but passive. They are of such size to be movable by a single agent. line-of-sight sensor that makes it able to distinguish between types of items in the environment (e.g., the background and other agents)

% as found in many biological systems~\citep{Hedwig2004}. %researchers have found that the complex sound localization behavior in female crickets emerge from simple reactive steering responses to specific sound pulses. In social behaviors, ants simply follow the pheromone trails when foraging~\citep{Carroll1973}. so the behavior of each agent only depends on its current sensory information. The agents directly map their sensor states to motor actions and thus need no memory. \footnote{For example, researchers have found that the complex auditory orientation behavior of female crickets is derived from simple reactive motor responses to specific sound pulses~\citep{Hedwig2004}. }

\textcolor{black}{Each agent implements a reactive behavior by mapping} the input ($I$) onto the outputs, that is, a pair of predefined speeds for the left and right wheels,
%The motion of each agent solely depends on the state of its line-of-sight sensor ($I$). Each possible sensor state, $I\in\{0,1,\cdots,n-1\}$, is mapped onto a pair of predefined speeds for the left and right wheels, 
$(v_{\ell I}, v_{rI})$, $v_{\ell I}, v_{rI} \in \left[-1,1\right]$.
% represent the normalized left and right wheel speeds, respectively, where $1$ ($-1$) corresponds to the wheel rotating forwards (backwards) with maximum speed. 
Given $n$ sensor states, \todo{the mapping} can be represented using $2n$ system parameters, which we denote as:
% In the remainder of this paper, we describe the corresponding controllers by writing the $2n$ parameters as a tuple in the following order:
\begin{equation}\label{controller:form}
\mathbf{p} = (v_{\ell 0}, v_{r0}, v_{\ell1}, v_{r1}, \cdots, v_{\ell (n-1)}, v_{r (n-1)}).
\end{equation}

Using $\mathbf{p}$, any reactive behavior for the above agent can be expressed. In the following, we \todo{consider} two example behaviors \todo{in detail}. 
%\textcolor{black}{\st{We investigate the ability of \textit{Turing Learning} to infer these behaviors as well as 1000 randomly-generated reactive behaviors.}}

\textbf{Aggregation:}\label{sec:aggregation_behavior} 
In this behavior, the sensor is binary, that is, $n=2$. It gives a reading of $I=1$ if there is an agent in the line of sight, and $I=0$ otherwise. The environment is free of obstacles. The objective of the agents is to aggregate into a single compact cluster as fast as possible. Further details, including a validation with 40 physical e-puck robots, are reported in~\citep{Gauci2014_ijrr}. %The agents are homogeneous: they all execute the same behavior. 

The aggregation controller was found by performing a grid search over the space of possible controllers~\citep{Gauci2014_ijrr}. The controller exhibiting the highest performance was:
\begin{equation}\label{eq:aggregation_optimal_controller}
\mathbf{p} = \left(-0.7, -1.0, 1.0, -1.0\right). 
\end{equation}

\captionsetup[subfigure]{labelformat=empty}  
\begin{figure}
	\centering
	\subfloat[initial configuration]{  %\scriptsize{initial configuration}
		\includegraphics[width = 1.1 in]{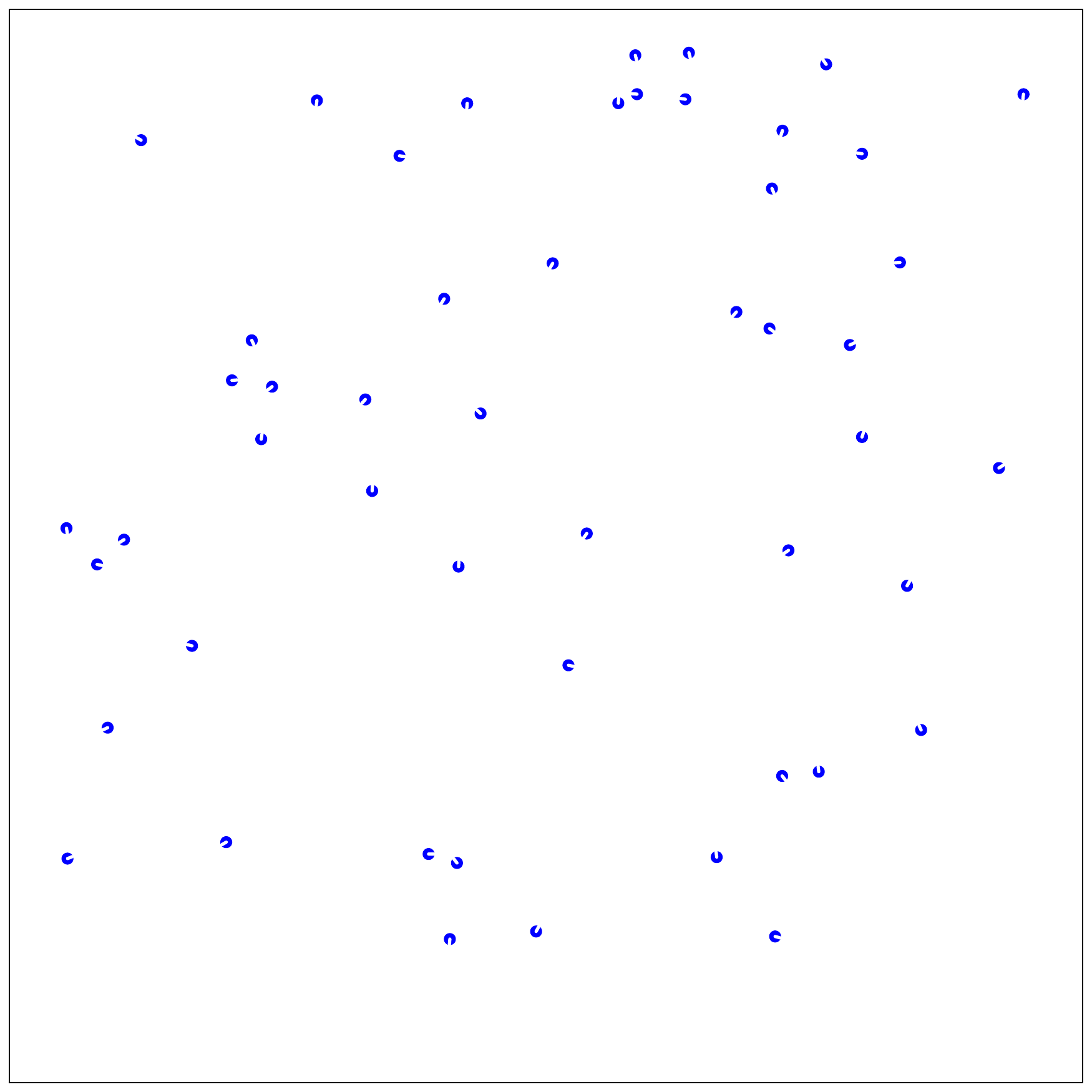}  %1.25
	}
	\subfloat[after $60$ $\unit{s}$]{
		\includegraphics[width = 1.1 in]{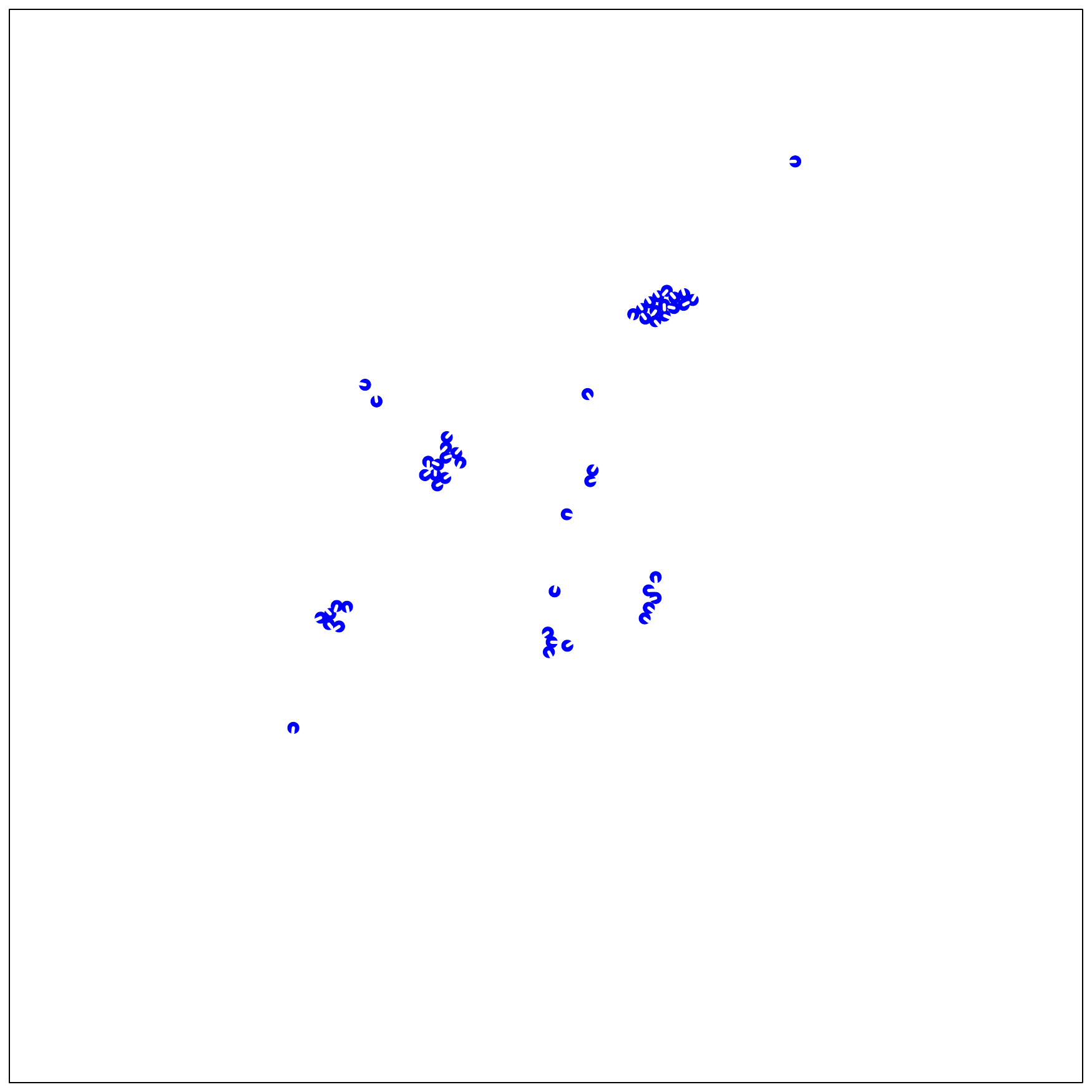}
	}
	\subfloat[after $180$ $\unit{s}$]{
		\includegraphics[width = 1.1 in]{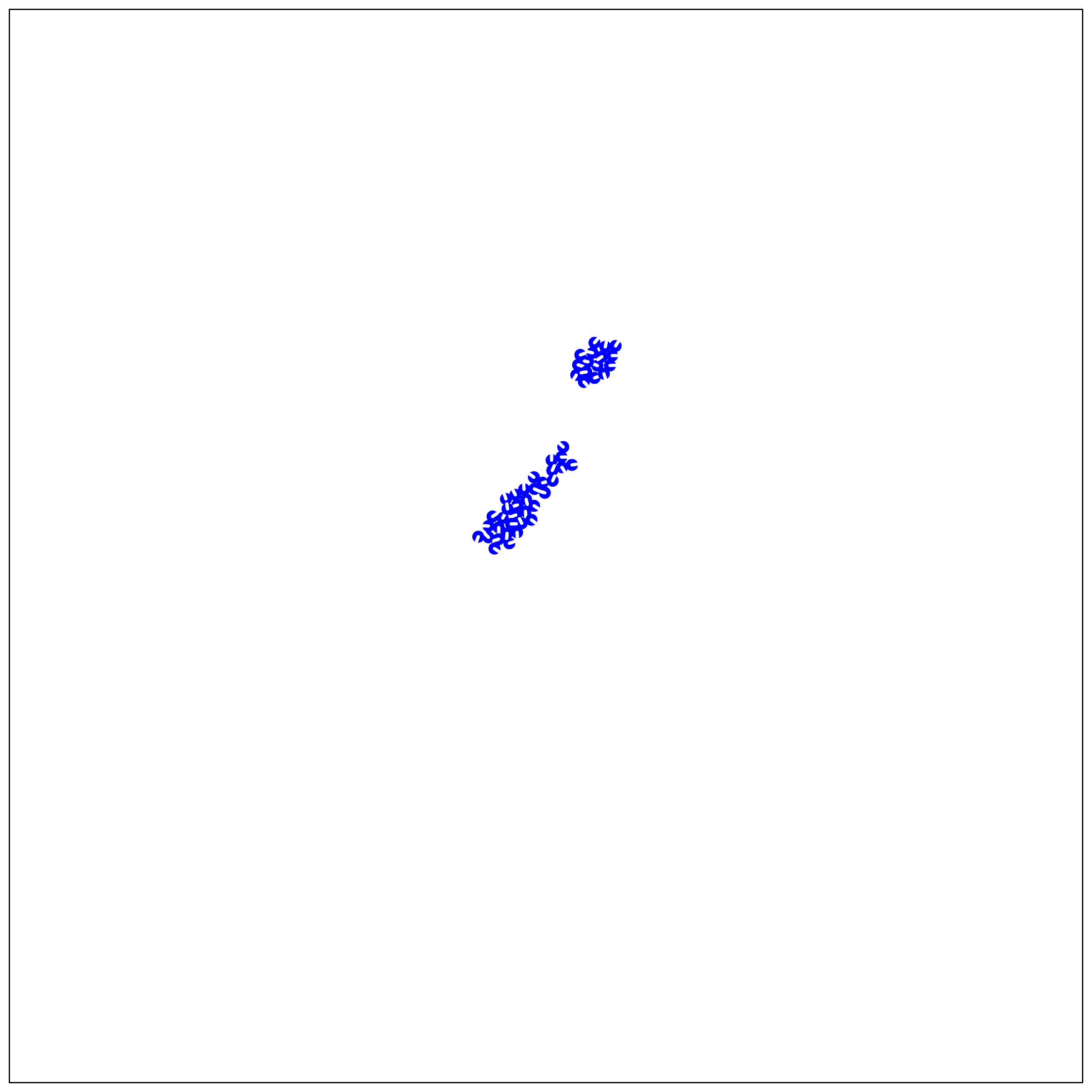}
	}
	\subfloat[after $300$ $\unit{s}$]{
		\includegraphics[width = 1.1 in]{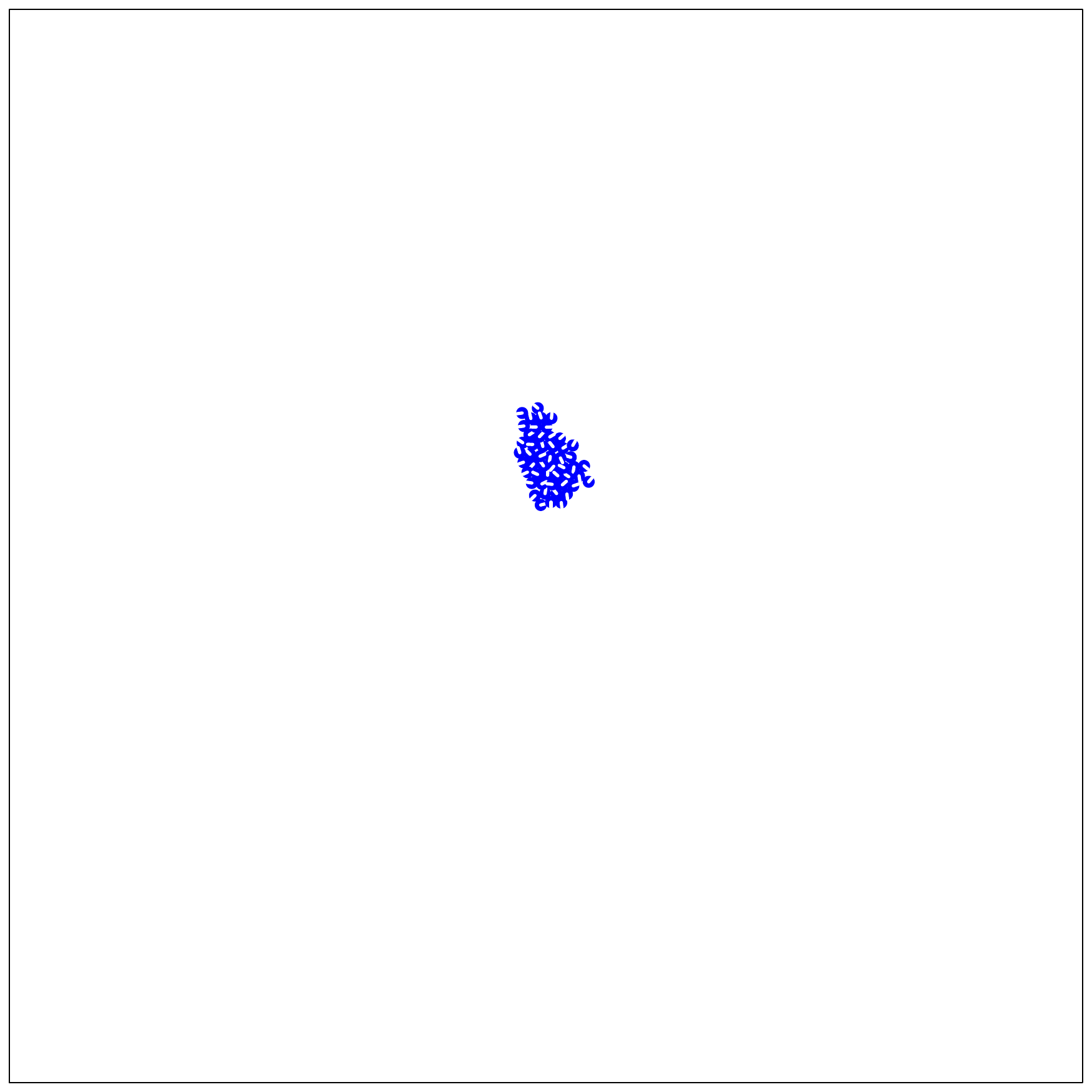}
	}
	\caption{Snapshots of the aggregation behavior of $50$ agents in simulation. }
	\label{fig:aggregation_snapshoot}
\end{figure}

When $I=0$, an agent moves backwards along a clockwise circular trajectory ($v_{\ell0} = -0.7$ and $v_{r0} = -1.0$). When $I=1$, an agent rotates clockwise on the spot with maximum angular speed ($v_{\ell1} = 1.0$ and $v_{r1} = -1.0$). Note that, rather counterintuitively, an agent never moves forward, regardless of $I$. With this controller, an agent provably aggregates with another agent or with a quasi-static cluster of agents~\citep{Gauci2014_ijrr}. Fig.~\ref{fig:aggregation_snapshoot} shows snapshots from a simulation trial with $50$ agents.

\textbf{Object Clustering:}  
In this behavior, the sensor is ternary, that is, $n=3$. It gives a reading of $I=2$ if there is an agent in the line of sight, $I=1$ if there is an object in the line of sight, and $I=0$ otherwise. 
%This behavior uses $n=3$ sensor states: $I=0$ if the sensor is pointing at the background (e.g., the wall of the environment, if the latter is bounded), $I=1$ if the sensor is pointing at an object, and $I=2$ if it is pointing at another agent. 
The objective of the agents is to arrange the objects into a single compact cluster as fast as possible. Details of this behavior, including a validation using 5 physical e-puck robots and 20 cylindrical objects, are presented in~\citep{Melvin2014_aamas}.

The controller's parameters, found using an evolutionary algorithm~\citep{Melvin2014_aamas}, are:
\begin{equation}\label{eq:clustering_optimal_controller}
\mathbf{p} = \left( 0.5, 1.0, 1.0, 0.5, 0.1, 0.5 \right).
\end{equation} 

\captionsetup[subfigure]{labelformat=empty}  
\begin{figure}
	\centering
	\subfloat[initial configuration]
	{
		\includegraphics[width = 1.1 in]{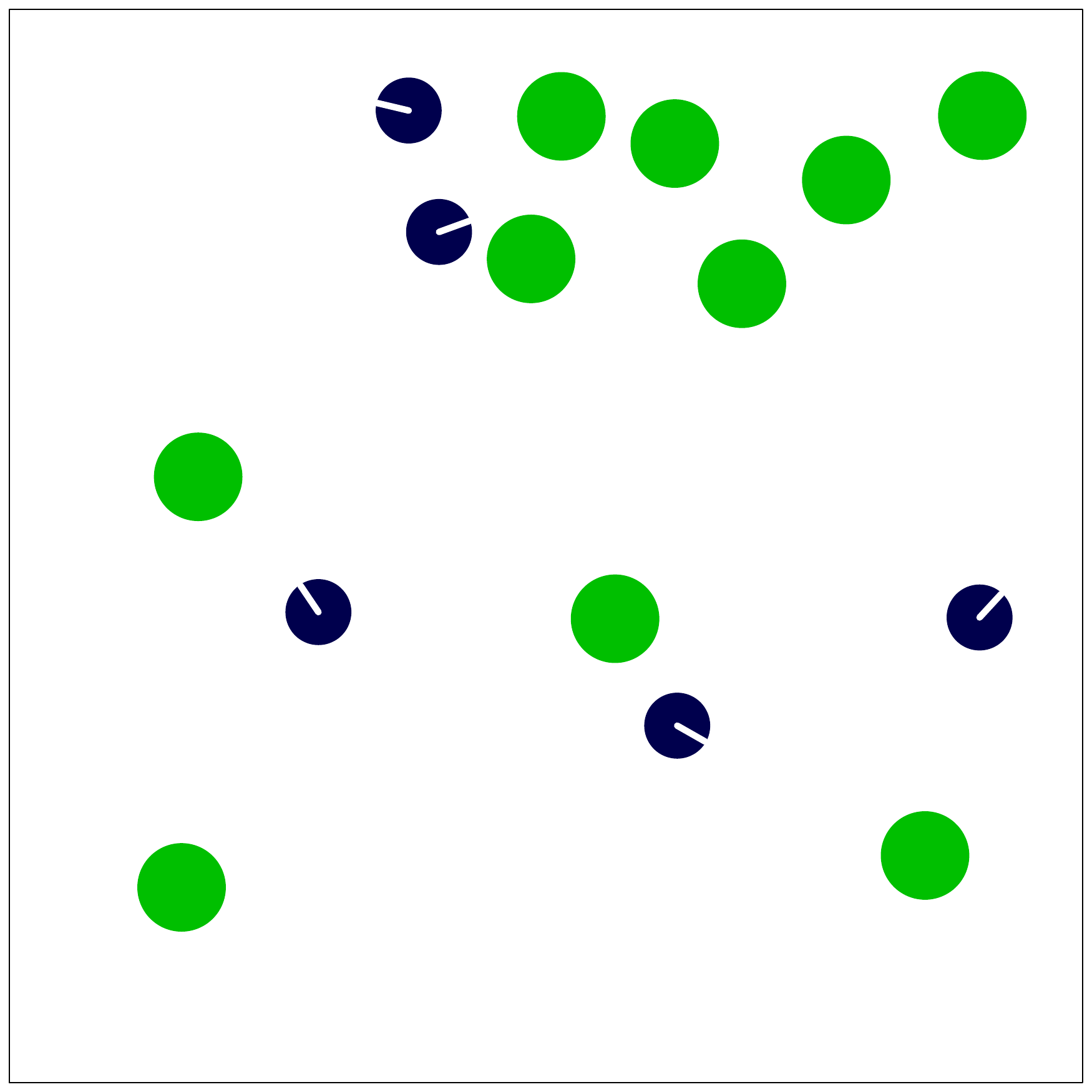}  %1.25
	}
	\subfloat[after $20$ $\unit{s}$]{
		\includegraphics[width = 1.1 in]{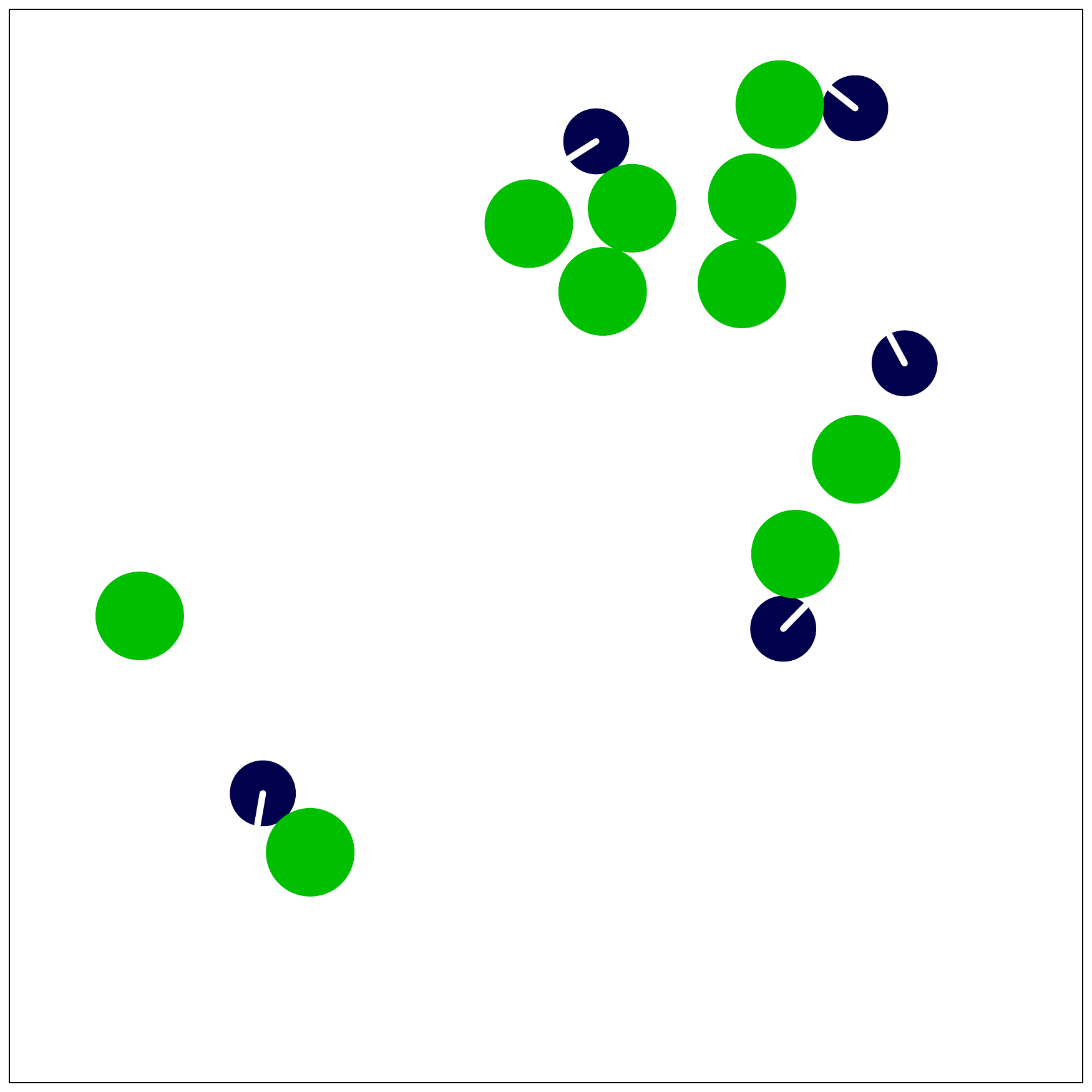}
	}
	\subfloat[after $40$ $\unit{s}$]{
		\includegraphics[width = 1.1 in]{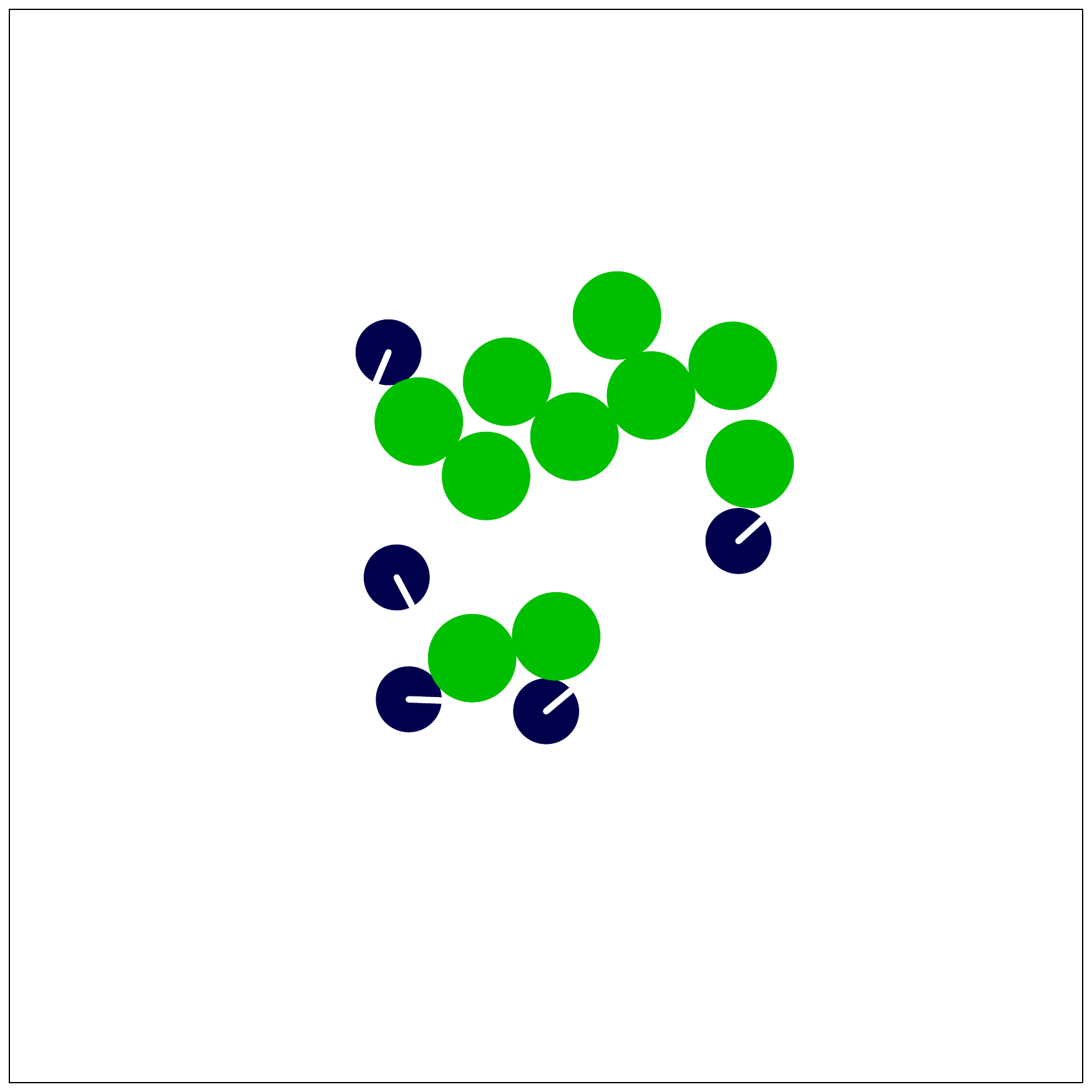}
	}
	\subfloat[after $60$ $\unit{s}$]{
		\includegraphics[width = 1.1 in]{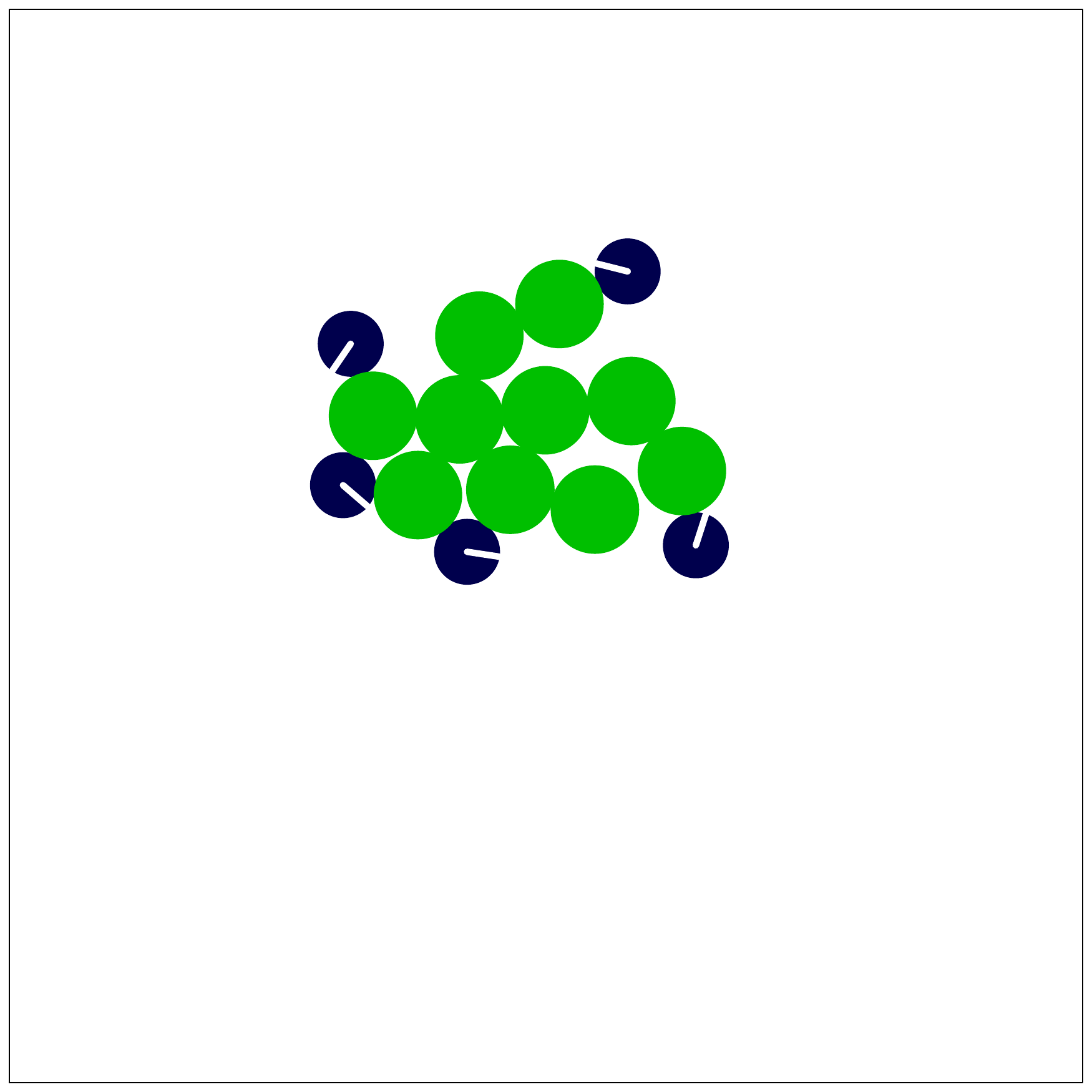}
	}\\
	\caption{Snapshots of the object clustering behavior in simulation. There are $5$ agents (dark blue) and 10 objects (green).}
	\label{fig:clustering_snapshoot}
\end{figure}

When $I=0$ and $I=2$, the agent moves forward along a counterclockwise circular trajectory, but with different linear and angular speeds. When $I=1$, it moves forward along a clockwise circular trajectory. Fig.~\ref{fig:clustering_snapshoot} shows snapshots from a simulation trial with $5$ agents and $10$ objects.

\subsubsection{Models and replicas} 
%We assume that one or more replicas, which have actuation and sensing abilities that are equivalent to those of the agents under investigation, are available. In this paper, the replica(s) will be mixed into a group of homogeneous agents. They should therefore be perceived by the agents as conspecifics~\citep{J.Halloy2007}. 

%In our case studies, the system identification task is to infer the behavior of a swarming agent. 
%we examine the ability of \textit{Turing Learning} to infer the behavior of swarming agents. In this application, i
% as it interacts with its environment.

%In our case studies, 

\textcolor{black}{We assume the availability of \textit{replicas}, which must have the potential to produce data samples that---to an external observer (classifier)---are indistinguishable from those of the agent. In our case, the replicas have the same morphology as the agent, including identical line-of-sight sensors and differential drive mechanisms.\footnote{In Section~\ref{sec:evolving_control_and_morphology}, we show that this assumption can be relaxed by also inferring some aspect of the agent's morphology.}} 
%Therefore, the sensing ability, actuating mechanism and control system structure of the replica need to be considered. In our case studies, the replica has the same line-of-sight sensor and differential drive as the agent.

\textcolor{black}{The replicas execute behavioral rules defined by the model. We adopt two model representations: gray box and black box.} In both cases, note that the classifiers, which determine the quality of the models, have no knowledge about the \textcolor{black}{agent/model} representation or the \textcolor{black}{agent/model} inputs.
\begin{itemize}
\item \textcolor{black}{In a gray box representation, the agent's control system structure is assumed to be known. In other words, the model and the agent share the same control system structure, as defined in~Eq.~\eqref{controller:form}. This representation reduces the complexity of the identification process, in the sense that only the parameters of Eq.~\eqref{controller:form} need to be inferred. Additionally, this allows for an objective evaluation of how well the identification process performs, because one can compare the inferred parameters directly with the ground truth.} 
%We will show later that a conventional metric-based system identification method does not infer the behavior well, even though the behavioral rules of the swarming agent is simple.
%The replica uses a reactive controller and has the same number of parameters as the agents. In this paper, the aim of \textit{Turing Learning} is to infer the parameters of Eq.~\eqref{controller:form}, that is, the input-output mapping. 
\item \textcolor{black}{In a black box representation, the agent's control system structure is assumed to be unknown, and the model has to be represented in a general way. In particular, we use a control system structure with memory, in the form of a neural network with recurrent connections (see Section~\ref{sec:evolve_neural_network_model}).}
%we do not assume knowledge of whether agent has memory or not.
%We investigate this approach in Section~\ref{sec:evolve_neural_network_model} using artificial neural networks.}
\end{itemize}
 
%This makes it possible to objectively measure the quality of the obtained models in post-evaluation analysis. In this paper (except in  Sections~\ref{sec:evolving_control_and_morphology} and~\ref{sec:evolve_neural_network_model}),

\textcolor{black}{The replicas can be mixed into a group of agents or separated from them.} By default, we consider the situation that one or multiple replicas are mixed into a group of agents. The case of studying groups of agents and groups of replicas in isolation is investigated in Section~\ref{sec:separating_replicas_agents}.

%The replicas can be mixed into or separated from the agents. We assume that the replicas have physical abilities that are equivalent to those of the agents under investigation, which means the replicas have the potential of producing the behavior that is indistinguishable from the agents under the observation of the classifiers. The models to be executed on the replicas can be represented explicitly (e.g., parameters) or implicitly (e.g., artificial neural networks). 

\subsubsection{Classifiers}

% TODO - here the classifiers have no direct access to the outputs. 

The classifiers need to discriminate between data samples originating from the agents and ones originating from the replicas. We use the term \textit{individual} to refer to either the agent or a replica executing a model.

\todo{A data sample comes from the motion trajectory of an individual \finaltodo{observed} for the duration of a trial. \todo{We assume that it is possible to track both the individual's position and orientation. The sample comprises the} linear speed ($s$) and angular speed ($\omega$).\footnote{We define the linear speed to be positive when the angle between the individual's orientation and its direction of motion is smaller than $\unit[\pi / 2]{rad}$, and negative otherwise.} \textcolor{black}{Full details (e.g., trial duration) are provided in Sections 4.2 and 5.3 for the cases of simulation and physical experiments respectively.}}
%In simulation, the tracking is assumed to be noise-free (but see~\citep{Li-etal2014:proc_gecco}). In the physical experiments (Section~\ref{sec:physical_implementation}), noise is inherently present.

\todo{The classifier is represented} as an Elman neural network~\citep{Elman1990}. The network has $i=2$ inputs ($s$ and $\omega$), $h=5$ hidden neurons and one output neuron. 
%The input neurons \todo{provide} the linear speed ($v$) and angular speed ($\omega$) of the individual. The data is provided by tracking the individual's positions and orientations. We define the linear speed to be positive when the angle between the individual's orientation and its direction of motion is smaller than $\unit[\pi \small/ 2]{rad}$, and negative otherwise. %This number was chosen arbitrarily and we did not attempt to optimize the number required. 
%In simulation, the tracking is assumed to be noise-free (but see~\citep{Li-etal2014:proc_gecco}). In the physical experiments (Section~\ref{sec:physical_implementation}), noise is inherently present.
Each neuron of the hidden and output layers has a bias. The network thus has a total of $(i+1) h + h^2 + (h+1) = 46$ parameters, which all assume values in $\mathbb{R}$. 
The activation function used in the hidden and the output neurons is the logistic sigmoid function, which has the range $\left(0,1\right)$ and is defined as: 
\begin{equation}\label{equ:logistic_sigmoid}
\textrm{sig}\,(x) = \frac{1}{1+e^{-x}}, \quad\forall x \in \mathbb{R}.
\end{equation}
\textcolor{black}{The data sample consists of a time series, which is fed sequentially into the classifier neural network.} The final value of the output neuron is used to make the judgment: \todo{model, if its value is less than $0.5$, and agent otherwise. 
The network's} memory (hidden neurons) is reset after each judgment.

%The classifier %makes judgments about all individuals---both agents and models. 
%\todo{observes each individual over a set period of time (i.e., a trial). Only the sequence of motion data is fed into the neural network.} 

%The classifier does not have any prior knowledge about the individual under investigation. It is fed with the sequence of motion data of the individual. It has two input neurons ($i=2$), five hidden neurons ($h=5$) and one output neuron. 

\subsubsection{Optimization algorithm}\label{sec:optimization_algorithm}

The optimization of models and classifiers is realized using an evolutionary algorithm. We use a ($\mu+\lambda$) evolution strategy with self-adaptive mutation strengths~\citep{Eiben2003} to optimize either population. As a consequence, the optimization consists of two processes, one for the model population, and another for the classifier population. The two processes synchronize whenever the \todo{solution qualities} described in Section~\ref{sec:turing_learning} are computed. 
%The optimization algorithm can be thought of consisting of two sub-algorithms: one for the model population, and another for the classifier population. The sub-algorithms do not interact with each other, except for the \todo{calculation of the solution quality} described in Section~\ref{sec:turing_learning}. 
% TODO: need to mention synchronous execution
The implementation of the evolutionary algorithm is detailed in~\citep{Li-etal2013:proc_gecco}.

For the remainder of this paper, we adopt terminology used in evolutionary computing, and refer to the quality of solutions as their \textit{fitness} and to iteration cycles as \textit{generations}. \textcolor{black}{Note that in coevolutionary algorithms, each population's fitness depends on the performance of the other populations, and is hence referred to as the \textit{subjective} fitness. By contrast, the fitness measure as used in conventional evolutionary algorithms is referred to as the \textit{objective} fitness.}

\subsubsection{Termination criterion}

%Two termination criteria could be used for the \textit{Turing Learning} method: 1) the coevolutionary algorithm runs for a fixed number of generations; 2) The fitness of the models and classifiers maintains a steady state. In this paper, we use the first termination criteria. 
The algorithm stops after running for a fixed number of iterations. 
%\todo{For the remainder of this paper, we refer to these as generations.}
%\todo{An ideal case is the models behave like the agents and indistinguishable by the classifiers. All the classifiers (suppose they are also trained well) get a fitness of $0.5$.}

\section{Simulation experiments}\label{sec:results_simulation}
 
In this section, we present the simulation experiments for the two case studies. \textcolor{black}{ Sections~\ref{sec:simulation_platform} and~\ref{sec:simulation_setup} describe the simulation platform and setups. Sections \ref{sec:analysis_evolved_models} and \ref{sec:analysis_of_evolved_classifiers_simulation},  respectively, analyze the inferred models and classifiers. Section \ref{sec:metric-based_EA} compares \textit{Turing Learning} with a metric-based identification method and mathematically analyzes this latter method. Section~\ref{sec:generality_turing_learning} presents further results of testing the generality of \textit{Turing Learning} through exploring different scenarios, which include: (i) simultaneously inferring the control of the agent and an aspect of its morphology; (ii) using artificial neural networks as a model representation, thereby removing the assumption of a known agent control system structure; (iii) separating the replicas and the agents, thereby allowing for a potentially simpler experimental setup; and (iv) inferring arbitrary reactive behaviors.}
%s \ref{sec:evolving_control_and_morphology}, \ref{sec:infer_other_behaviors},  \ref{sec:separating_replicas_agents} and \ref{sec:evolve_neural_network_model}

\subsection{Simulation platform}\label{sec:simulation_platform}

We use the open-source Enki library~\citep{Enki}, which models the kinematics and dynamics of rigid objects, and handles collisions. Enki has a built-in 2D model of the e-puck. The robot is represented as a disk of diameter $\unit[7.0]{cm}$ and mass $\unit[150]{g}$. The inter-wheel distance is $\unit[5.1]{cm}$. The speed of each wheel can be set independently. Enki induces noise on each wheel speed by multiplying the set value by a number in the range $(0.95, 1.05)$ chosen randomly with uniform distribution. The maximum speed of the e-puck is $\unit[12.8]{\textrm{cm/s}}$, forwards or backwards. The line-of-sight sensor is simulated by casting a ray from the e-puck's front and checking the first item with which it intersects (if any). The range of this sensor is unlimited in simulation. 
%Although Enki models the e-puck's real sensors, we choose to implement the line-of-sight sensor independently, for the sake of accuracy (in the physical experiments, this sensor is realized using the e-puck's camera; see Section~\ref{sec:robot_platform_sensor_implementation}).

In the object clustering case study, we model objects as disks of diameter $\unit[10]{cm}$ with mass $\unit[35]{g}$ and a coefficient of static friction with the ground of $0.58$, which makes it movable by a single e-puck.

The robot's control cycle is updated every $\unit[0.1]{s}$, and the physics is updated every $\unit[0.01]{s}$.

\subsection{Simulation setups}\label{sec:simulation_setup}

In all simulations, we use an unbounded environment. For the aggregation case study, we use groups of $11$ individuals---$10$ agents and $1$ replica that executes a model. The initial positions of individuals are generated randomly in a square region of sides $\unit[331.66]{cm}$, following a uniform distribution (average area per individual = $\unit[10000]{cm^2}$). For the object clustering case study, we use groups of $5$ individuals---$4$ agents and $1$ replica that executes a model---and $10$ cylindrical objects. The initial positions of individuals and objects are generated randomly in a square region of sides $\unit[100]{cm}$, following a uniform distribution (average area per object = $\unit[1000]{cm^2}$). In both case studies, individual starting orientations are chosen randomly in $[-\pi,\pi)$ with uniform distribution. %The initial configuration\footnote{Throughout this paper, `initial configuration' refers to initial positions and orientations of the individuals or objects, if any.} of the individuals is randomly generated in each trial. 

We performed 30 \todo{runs of \textit{Turing Learning}} for each case study. Each run lasted 1000 generations. The model and classifier populations each consisted of $100$ solutions ($\mu = 50$,  $\lambda = 50$). In each trial, classifiers observed individuals for $\unit[10]{s}$ at $\unit[0.1]{s}$ intervals ($100$ data points). \textcolor{black}{In both setups, the total number of samples for the agents in each generation was equal to $n_t \times n_a$, where $n_t$ is the number of trials performed (one per model) and $n_a$ is the number of agents in each trial.} 

\subsection{Analysis of inferred models}\label{sec:analysis_evolved_models}

In order to objectively measure the quality of the models obtained through~\textit{Turing Learning}, we define two metrics. Given a \todo{candidate model (candidate controller)} $\mathbf{x}$ and the \todo{agent (original controller)} $\mathbf{p}$, where $\mathbf{x}\in\mathbb{R}^{2n}$ and $\mathbf{p}\in[-1,1]^{2n}$, we define the absolute error (AE) in a particular parameter $i\in\{1,2,\dots,2n\}$ as: 
\begin{equation}\label{eq:AE}
\mathrm{AE}_i = |x_i-p_i|. 
\end{equation}

We define the mean absolute error (MAE) over all parameters as: 
\begin{equation}\label{eq:MAE}
\mathrm{MAE} = \frac{1}{2n}\sum_{i=1}^{2n} \mathrm{AE}_i.
\end{equation}
%\footnote{Throughout the main text, we exclusively use the acronyms AE and MAE. Therefore ``mean AE'' is not to be confused with MAE; rather, it refers to the mean AE in one parameter over a number of coevolution runs.}

\begin{figure}[!t]%htbp
	\centering
		\subfloat[(a) Aggregation \label{fig:model_parameters_box_aggregation}]{%
			\includegraphics[width=2.25 in]{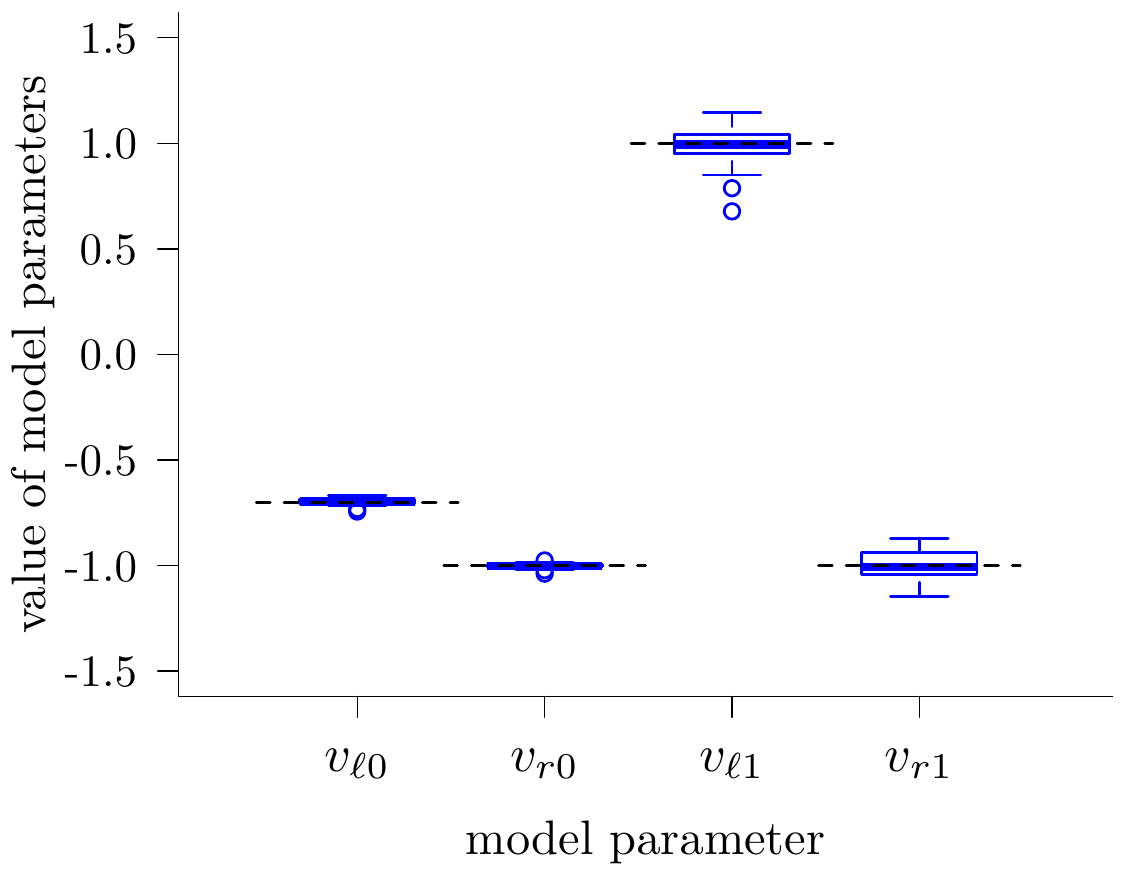}
		}
		\subfloat[(b) Object Clustering\label{fig:model_parameters_box_clustering}]{%
			\includegraphics[width=2.25 in]{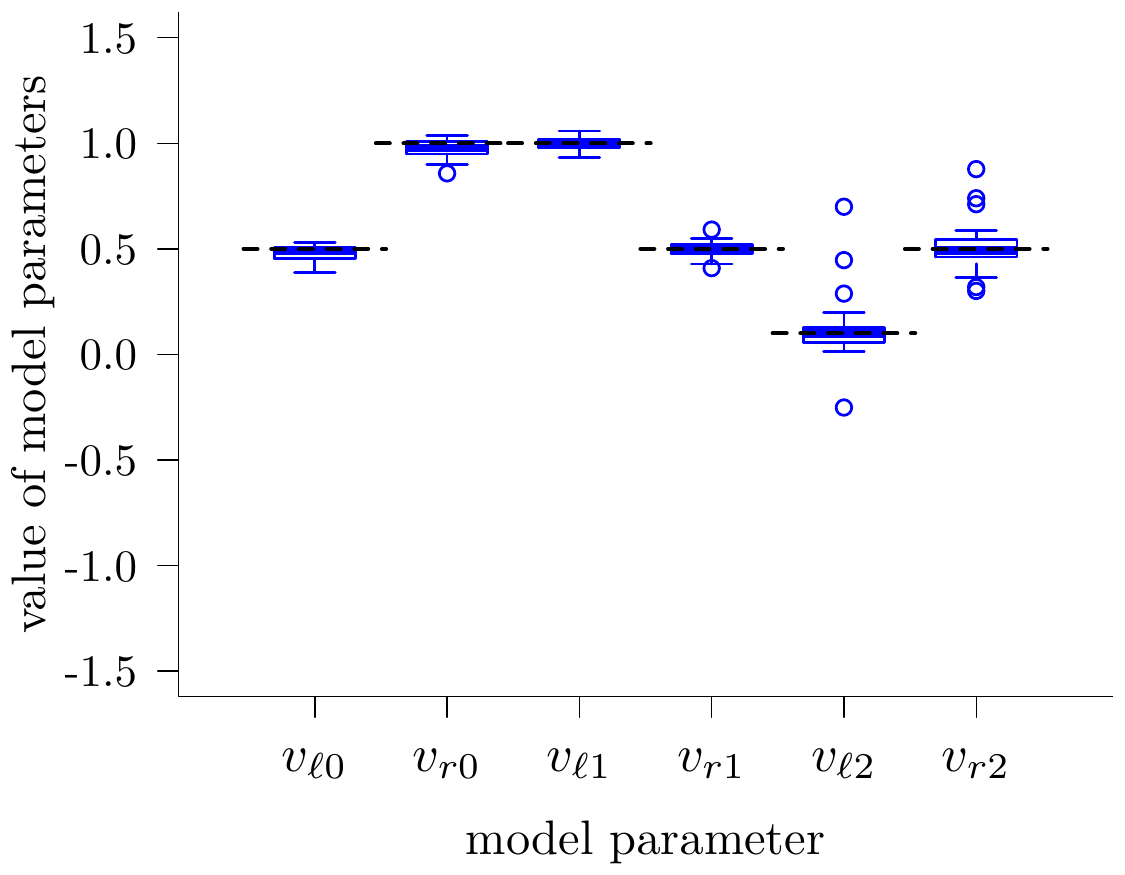}
		}
		\caption{Model parameters \textit{Turing Learning} inferred from swarms of simulated agents performing (a) aggregation and (b) object clustering. Each box corresponds to the models with the highest subjective fitness in the $1000$th generation of 30 runs. The dashed black lines correspond to the values of the parameters that the system is expected to learn (i.e., those of the agent).\label{fig:model_parameters_box}}
\end{figure}
%
%\todo{The fitness of the models depends on the solutions (classifiers) from the competing population, and is hence referred to as the \textit{subjective} fitness.}

Fig.~\ref{fig:model_parameters_box} shows a box plot\footnote{The box plots presented here are all as follows. The line inside the box represents the median of the data. The edges of the box represent the lower and the upper quartiles of the data, whereas the whiskers represent the lowest and the highest data points that are within $1.5$ times the range from the lower and the upper quartiles, respectively. Circles represent outliers.\label{fn:boxplot}} \todo{of the parameters of the inferred models with the highest \textcolor{black}{subjective} fitness value} in the final generation. %\textcolor{black}{\st{Note that the fitness of the models depends on the solutions (classifiers) from the competing population, and is hence referred to as the \textit{subjective} fitness.}} 
It can be seen that~\textit{Turing Learning} identifies the parameters for both behaviors with good accuracy (dashed black lines represent the ground truth, that is, the parameters of the observed swarming agents). In the case of aggregation, the means (standard deviations) of the AEs in the parameters are (from left to right in Fig.~\subref*{fig:model_parameters_box_aggregation}): $0.01$ ($0.01$), $0.01$ ($0.01$), $0.07$ ($0.07$), and $0.06$ ($0.04$). In the case of object clustering, these values are as follows: $0.03$ ($0.03$), $0.04$ ($0.03$), $0.02$ ($0.02$), $0.03$ ($0.03$), $0.08$ ($0.13$), and $0.08$ ($0.09$).
%In the rest of this paper, unless otherwise stated, `evolved model' refers to the model with the highest subjective fitness in a generation.

\begin{figure}[!t]%
	\centering
		\subfloat[(a) Aggregation \label{fig:model_parameters_convergence_aggregation}]{%
			\includegraphics[width=2.25 in]{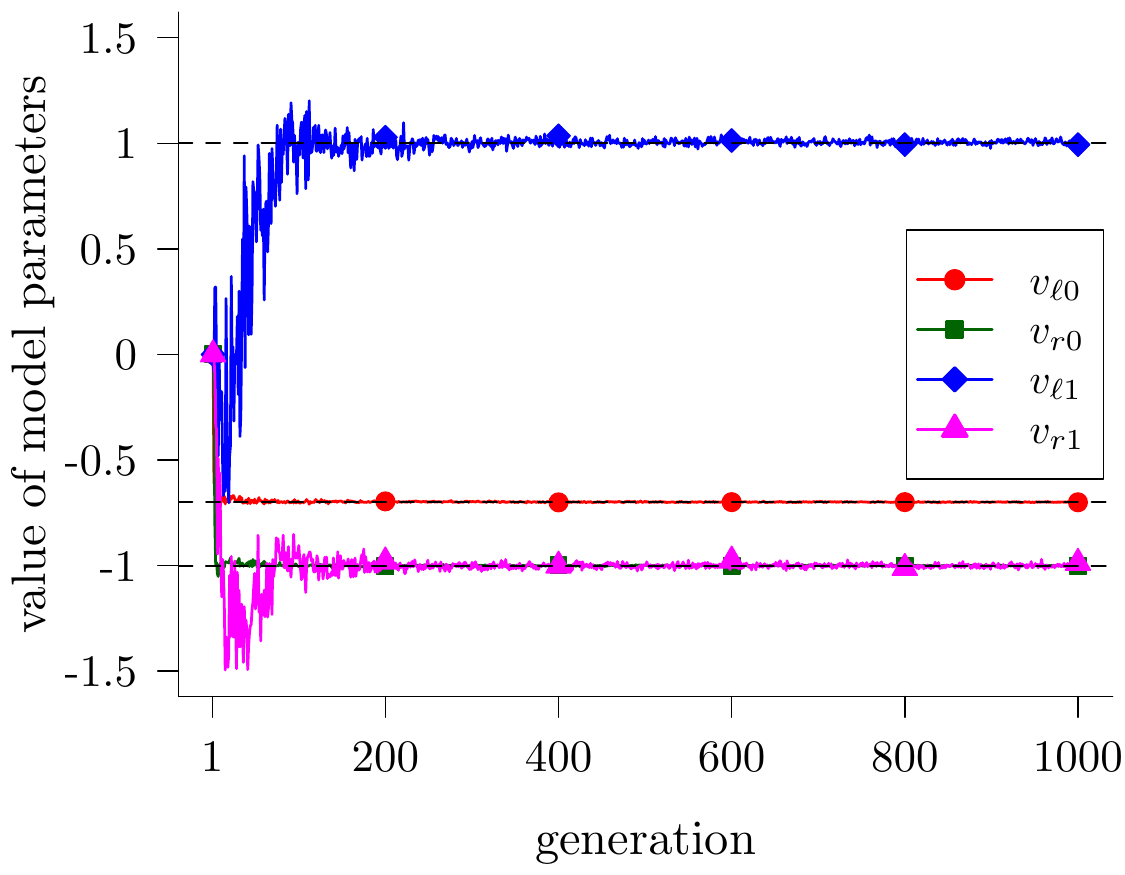}
		}
		\subfloat[(b) Object Clustering\label{fig:model_parameters_convergence_clustering}]{%
			\includegraphics[width=2.25 in]{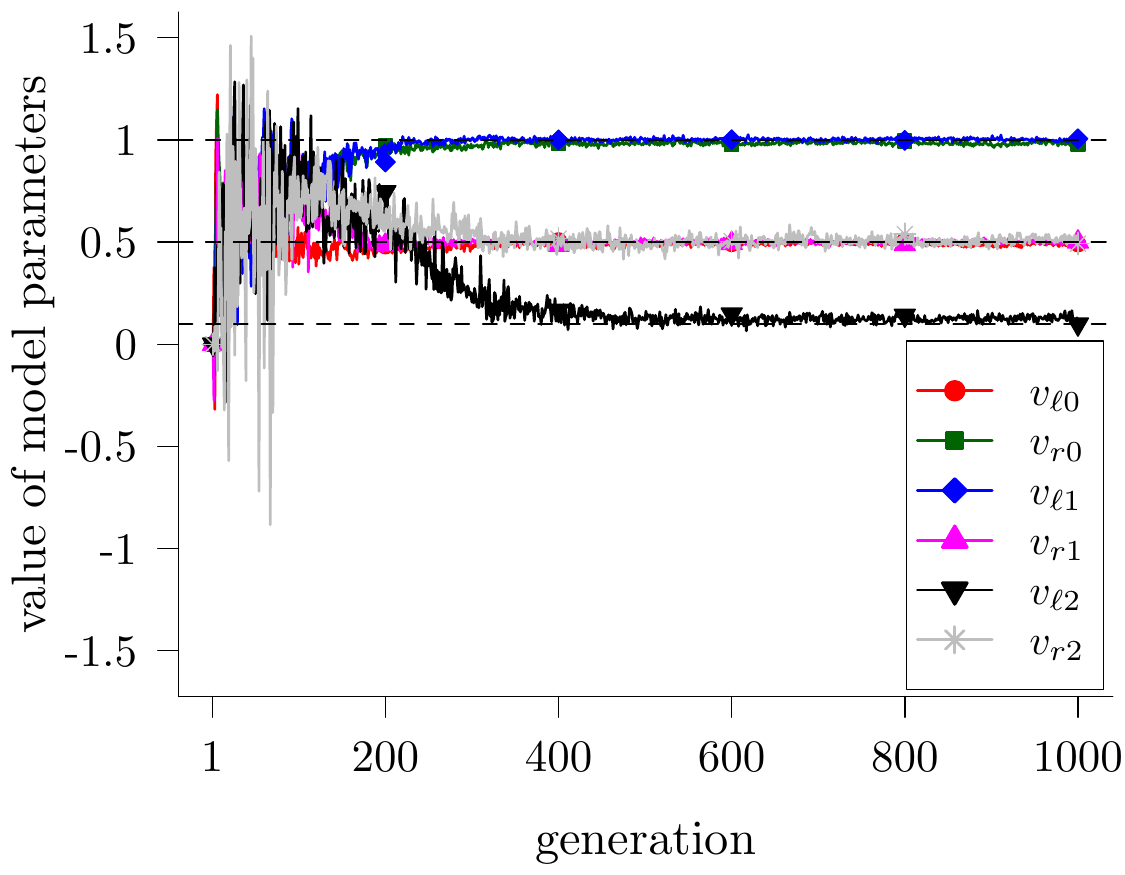}
		}
		\caption{Evolutionary \todo{dynamics of} model parameters for the (a) aggregation and (b) object clustering \todo{case studies}. Curves represent median \todo{parameter values of the models with the highest subjective fitness} across 30 \todo{runs of \textit{Turing Learning}}. Dashed black lines indicate \todo{the ground truth}. \label{fig:model_parameters_convergence}}
\end{figure}

We also investigate the evolutionary dynamics. Fig.~\ref{fig:model_parameters_convergence} shows how the model parameters converge over generations. In the aggregation case study (see Fig.~\subref*{fig:model_parameters_convergence_aggregation}), the parameters corresponding to $I=0$ are learned first. After around $50$ generations, both $v_{\ell0}$ and $v_{r0}$ closely approximate their true values ($-0.7$ and $-1.0$). For $I=1$, it takes about $200$ generations for both $v_{\ell1}$ and $v_{r1}$ to converge. A likely reason for this effect is that an agent spends a larger proportion of its time seeing nothing ($I=0$) than seeing other agents ($I=1$)---simulations revealed these percentages to be $91.2\%$ and $8.8\%$ respectively (mean values over $100$ trials). 

In the object clustering case study (see Fig.~\subref*{fig:model_parameters_convergence_clustering}), the parameters corresponding to $I=0$ and $I=1$ are learned faster than the parameters corresponding to $I=2$. After about $200$ generations, $v_{\ell0}$, $v_{r0}$, $v_{\ell1}$ and $v_{r1}$ start to converge; however it takes about $400$ generations for $v_{\ell2}$ and $v_{r2}$ to approximate their true values. Note that an agent spends the highest proportion of its time seeing nothing ($I=0$), followed by seeing objects ($I=1$) and seeing other agents ($I=2$)---simulations revealed these proportions to be $53.2\%$, $34.2\%$ and $12.6\%$ respectively (mean values over $100$ trials).

\begin{figure}[!t]
	\centering
		\subfloat[(a) Aggregation \label{fig:model_validation_aggregation_simulation}]{%
			\includegraphics[width=2.25 in]{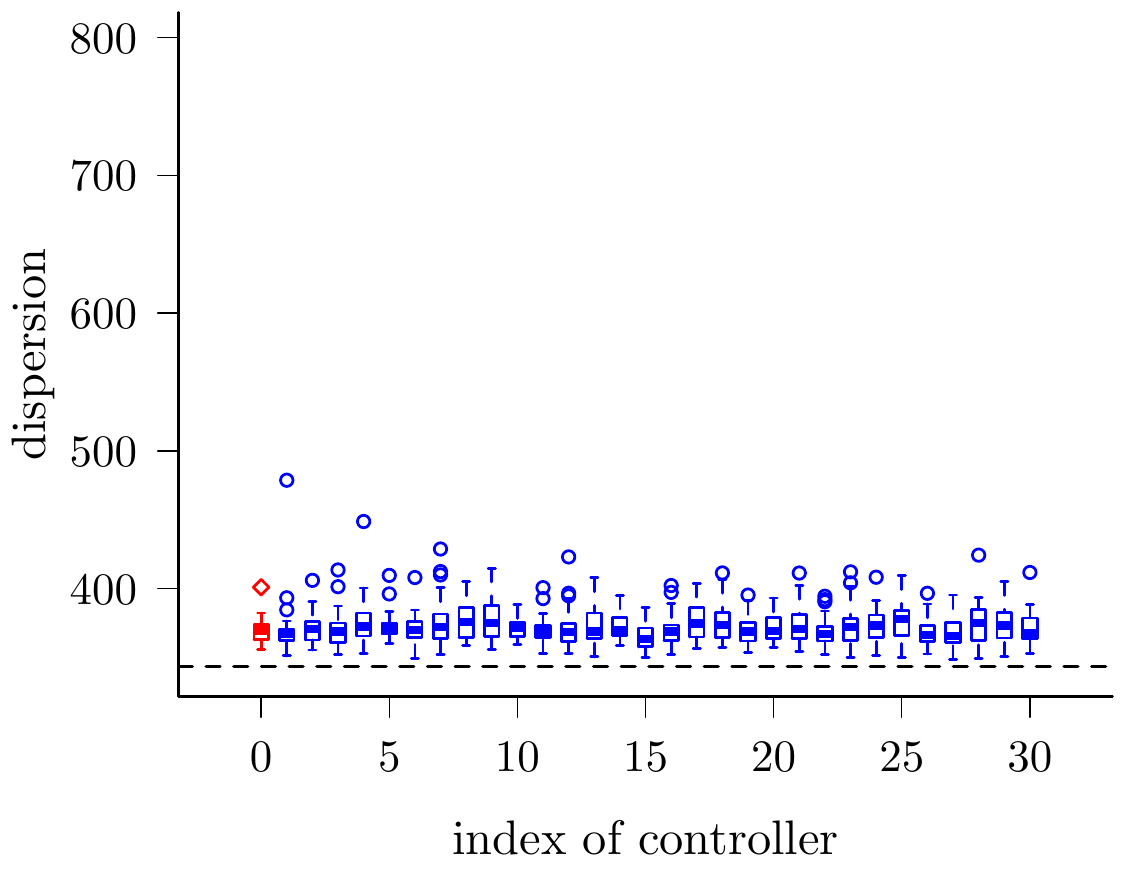}
		}
		\subfloat[(b) Object Clustering\label{fig:model_validation_clustering_simulation}]{%
			\includegraphics[width=2.25 in]{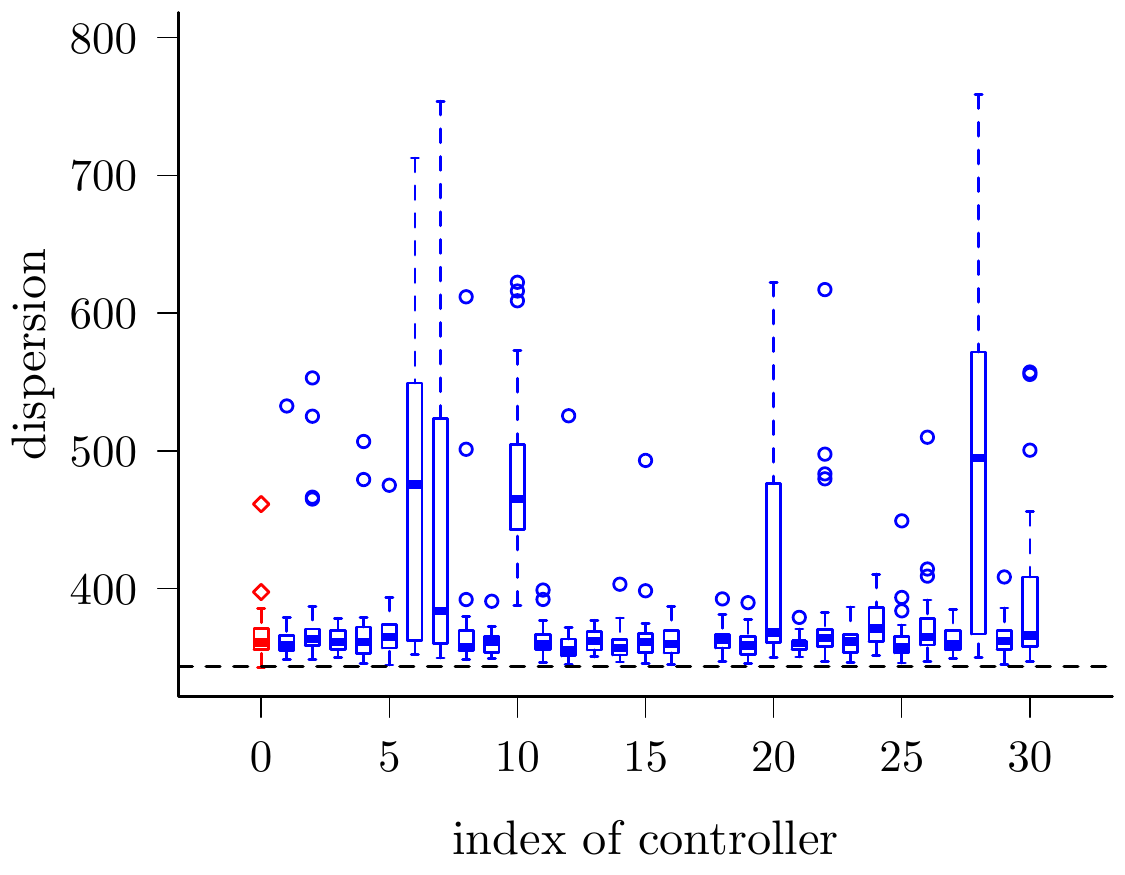}
		}
		\caption{(a) Dispersion of $50$ simulated agents (red box) or replicas (blue boxes), executing one of the $30$ inferred models in the aggregation case study. (b) Dispersion of $50$ objects when using a swarm of $25$ simulated agents (red box) or replicas (blue boxes), executing one of the $30$ inferred models in the object clustering case study. In both (a) and (b), boxes show the distributions obtained after $\unit[400]{s}$ over $30$ trials. The models are from the $1000$th generation. The dashed black lines indicate the minimum dispersion that $50$ individuals/objects can possibly achieve~\citep{Graham1990}. See Section~\ref{sec:analysis_evolved_models} for details.\label{fig:model_validation_simulation}}
\end{figure}

Although the \todo{inferred} models approximate the agents well in terms of parameters, it \finaltodo{is not uncommon} in swarm systems that small changes in individual behavior lead to vastly different emergent behaviors, especially when using large numbers of agents~\citep{Paul2010}. For this reason, we evaluate the quality of the emergent behaviors that the models give rise to. In the case of aggregation, we measure the dispersion of the swarm after some elapsed time as defined in~\citep{Gauci2014_ijrr}\footnote{The measure of dispersion is based on the robots'/objects' distances from their centroid. For a formal definition, see Eq. (5) of \citep{Gauci2014_ijrr}, Eq. (2) of~\citep{Melvin2014_aamas} and~\citep{Graham1990}.}. For each of the $30$ models with the highest subjective fitness in the final generation, we performed $30$ trials with $50$ \todo{replicas executing} the model. For comparison, we also performed $30$ trials using the \todo{agent} (see Eq.~\eqref{eq:aggregation_optimal_controller}). The set of initial configurations was the same for the replicas and the agents. Fig.~\subref*{fig:model_validation_aggregation_simulation} shows the dispersion of agents and replicas after $\unit[400]{s}$. All models led to aggregation. We performed a statistical test\footnote{Throughout this paper, the statistical test used is a two-sided Mann-Whitney test with a $5\%$ significance level.} on the final dispersion of the individuals between the agents and replicas for each model. There is no statistically significant difference in \textcolor{black}{$30$ out of $30$ cases (tests with Bonferroni correction).} 
%in individual agent behaviors can lead to vastly different emergent behaviors, especially with large numbers of agents~\citep{Paul2010, Luca2014}

%TODO: 400s correct for clustering too? At least the figure caption says it was the same, so the text could too? Or change both, if not the case
In the case of object clustering, we use the dispersion of the objects after $\unit[400]{s}$ as a measure of the emergent behavior. \todo{We performed $30$ trials with $25$ individuals and $50$ objects for the agent and each model.}
% With the original controller (see Eq.~\eqref{eq:clustering_optimal_controller}) and each of the models,. 
The results are shown in Fig.~\subref*{fig:model_validation_clustering_simulation}. In the final dispersion of \todo{objects by the agent or any of the models (replicas)}, there is no statistically significant difference in \textcolor{black}{$26$ out of $30$ cases (tests with Bonferroni correction).}

\subsection{Analysis of \todo{generated} classifiers}\label{sec:analysis_of_evolved_classifiers_simulation}

The primary outcome of the \textit{Turing Learning} method (and of any system identification method) is the model, which has been discussed in the previous section. However, the \todo{generated} classifiers can also be considered as a useful by-product. For instance they could be used to detect abnormal agents in a swarm. We now analyze the performance of the classifiers. For the remainder of this paper, we consider only the aggregation case study.

To assess the performance of the classifiers, we measure the percentage of correct judgments over agents and a wide range of models. The models are uniformly distributed \textcolor{black}{over the four-dimensional parameter space of the agents (see Eqs.~\eqref{controller:form} and~\eqref{eq:aggregation_optimal_controller}); each element takes a value in $[-1,1]$}. To keep the analysis of classifiers within a reasonable computation time, we discretize this space using $11$ settings per parameter; \textcolor{black}{thus,  each parameter is chosen from the set $\{-1.0, -0.8, ..., 0.8, 1.0\}$.} This discretized space is a grid consisting of $|\mathcal{X}|=11^4=14641$ points (i.e., models). The model is executed by a replica mixed into a group of $10$ agents (as in the \todo{\textit{Turing Learning} runs}). The motion data are fed to each classifier, which makes $10$ judgments per individual. If the classifier consistently judges the individual as a model (i.e., not an agent) in $10$ out of $10$ trials, it outputs a ``model'' decision. Otherwise, it outputs ``agent''. This conservative approach is used to minimize the risk of false-positive detection of abnormal behavior. \textcolor{black}{The classifier's performance (i.e., decision accuracy) is computed by combining the percentage of correct judgments about models ($50\%$ weight) with the percentage of correct judgments about agents ($50\%$ weight), analogous to the solution quality definition in Eq.~\eqref{eq:classifier_fitness_definition}.} 
%\subsubsection{Using a Single Classifier}

\begin{figure}[!t]
	\centering
	\includegraphics[width=3.0 in]{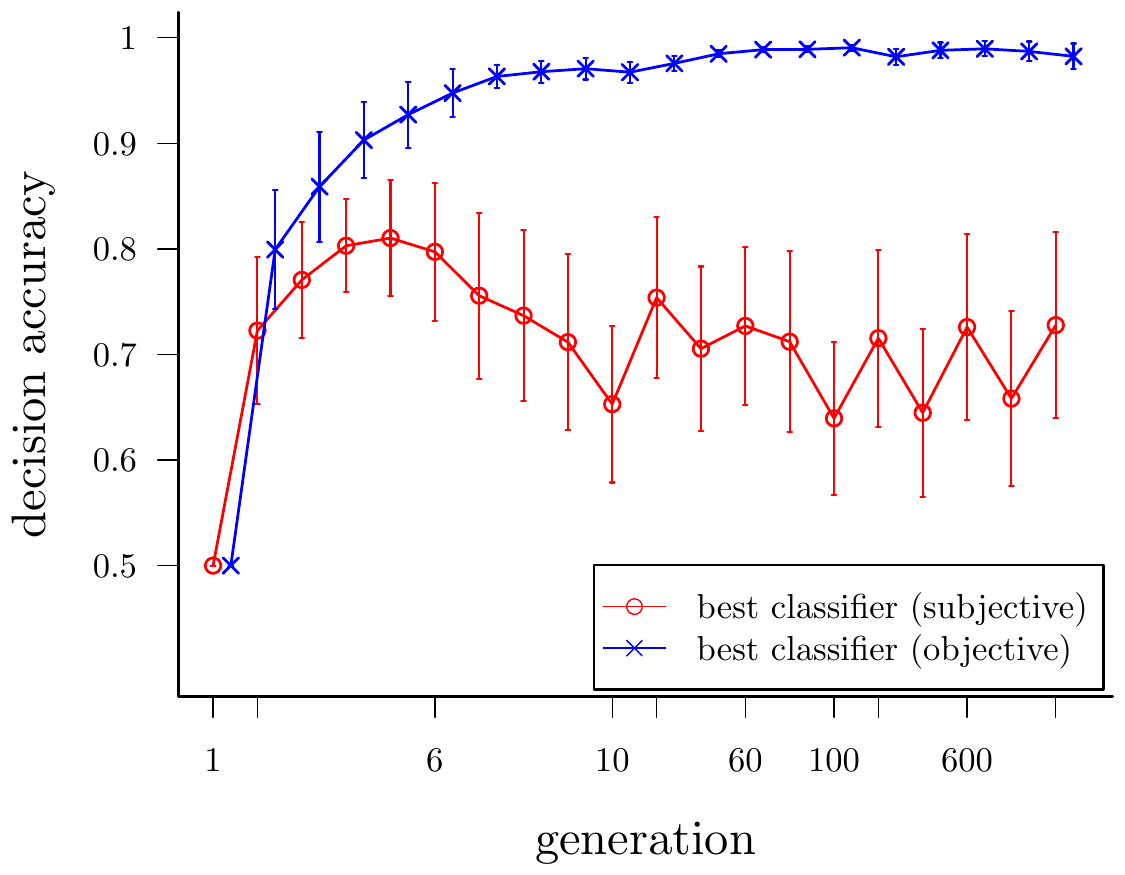}
	\caption{Average decision accuracy of the best classifiers over $1000$ generations (nonlinear scale) in $30$ runs of \textit{Turing Learning}. The error bars show standard deviations. See text for details.}
% TODO: Wei - please replace "decision" with "objective" in the figure legend?
	\label{fig:classifier_decision_accuracy_simulation}
\end{figure}

We performed $10$ trials using a set of initial configurations common to all classifiers. Fig.~\ref{fig:classifier_decision_accuracy_simulation} shows the average decision accuracy of the classifier with the highest subjective fitness during the evolution (\textit{best classifier (subjective)}) in $30$ runs of \textit{Turing Learning}. The accuracy of the classifier increases in the first $5$ generations, then drops, and fluctuates within range $62\%$--$80\%$.
%There is some fluctuation during the evolutionary process, but the accuracy is always at a low value over generations. 
%An alternative strategy is to select the classifier that achieves the highest fitness when evaluated on the whole historical tracking data (not just those of the current generation). The decision accuracy of this classifier is also shown in Fig.~\ref{fig:classifier_decision_accuracy_simulation} (\textit{best classifier (archive)}). The trend is similar to that of \textit{best classifier (subjective)}. The accuracy increases in the first $10$ generations,  and then starts \emph{decaying}, dropping to around $65\%$ by the $1000^\textrm{th}$ generation. However, in the earlier generations, the accuracy of the \textit{best classifier (archive)} is higher than that of the \textit{best classifier (subjective)}. 
For a comparison, we also plot the highest decision accuracy that a single classifier achieved during the post-evaluation for each generation. This classifier is referred to \textit{best classifier (objective)}. Interestingly, the accuracy of the \textit{best classifier (objective)}
%, which is shown in  Fig.~\ref{fig:classifier_decision_accuracy_simulation},
increases almost monotonically, reaching a level above $95\%$. \textcolor{black}{To select the \textit{best classifier (objective)}, all the classifiers were post-evaluated using the aforementioned $14641$ models.} 
%\textcolor{blue}{\st{Note that to select the \textit{best classifier (objective)}, one needs to perform additional trials ($146410$ in this case).}} 

At first sight, it seems counterintuitive that the \textit{best classifier (subjective)} has a low decision accuracy. This phenomenon, however, can be explained when considering the model population. We have shown in the previous section (see Fig.~\subref*{fig:model_parameters_convergence_aggregation}) that the models converge rapidly at the beginning of the coevolutions. As a result, when classifiers are evaluated in later generations, the trials are likely to include models very similar to each other. \todo{Classifiers that become overspecialized to this small set of models (the ones dominating the later generations) have a higher chance of being selected during the evolutionary process.} These classifiers may however have a low performance when evaluated across the entire model space.
%The selected classifiers thus become overspecialized to a small set of models: the ones dominating the later generations. 
%At first sight, it is counterintuitive that selecting the best classifier according to the historical data still leads to low decision accuracy. 

\textcolor{black}{
%The performance of \textit{Turing Learning} may be limited (bottlenecked) by the performance of the classifiers. 
Note that our analysis does not exclude the potential existence of models for which the performance of the classifiers degenerates substantially. As reported by~\citep{nguyen2015deep}, well-trained classifiers, which in their case are represented by deep neural networks, can be easily fooled. For instance, the classifiers may label a random-looking image as a guitar with high confidence. However, in this degenerate case, the image was obtained through evolutionary learning, while the classifiers remained static. By contrast, in \textit{Turing Learning},  the classifiers are coevolving with the models, and hence have the opportunity to adapt to such a situation.}
\subsection{A metric-based system identification method: mathematical analysis and comparison with \textit{Turing Learning}}\label{sec:metric-based_EA}

In order to compare \textit{Turing Learning} against a metric-based method, we employ the commonly used least-square approach.
%used an evolutionary algorithm with a single population of models. 
%The algorithm was identical to the \finaltodo{model optimization} sub-algorithm in \textit{Turing Learning} except for the fitness calculation. 
%We also kept the simulation setup the same. 
%In each generation, 
The objective is to minimize the differences between the observed outputs of the agents and of the models, respectively. Two outputs are considered---an individual's linear and angular speed. Both outputs are considered over the whole duration of a trial. 
%The metric is the
%the metric measures the reciprocal of the
%cummulative square error between a model's and the agent's observed output. 
Formally,
%linear and angular speed sequences over a  trial. Formally,
%The square error of a model is defined as follows:
\begin{equation}\label{eq:square_error_metric}
		\todo{e_{m} = \sum \limits_{i=1}^{n_a} \sum \limits_{t=1}^{T} \left\{ (s_m^{(t)} - s_i^{(t)})^2 + (\omega_m^{(t)} - \omega_i^{(t)})^2 \right\},}
\end{equation}
where $s_m^{(t)}$ and $s_i^{(t)}$ are the linear speed of the model and of agent $i$, respectively, at time step $t$; $\omega_m^{(t)}$ and $\omega_i^{(t)}$ are the angular speed of the model and of agent $i$, respectively, at time step $t$; $n_a$ is the number of agents in the group; $T$ is the total number of time steps in a trial. 

\begin{figure}[!t]
	\centering
		\subfloat[(a) Aggregation \label{fig:model_parameters_box_aggregation_evolution}]{%
			\includegraphics[width=2.25 in]{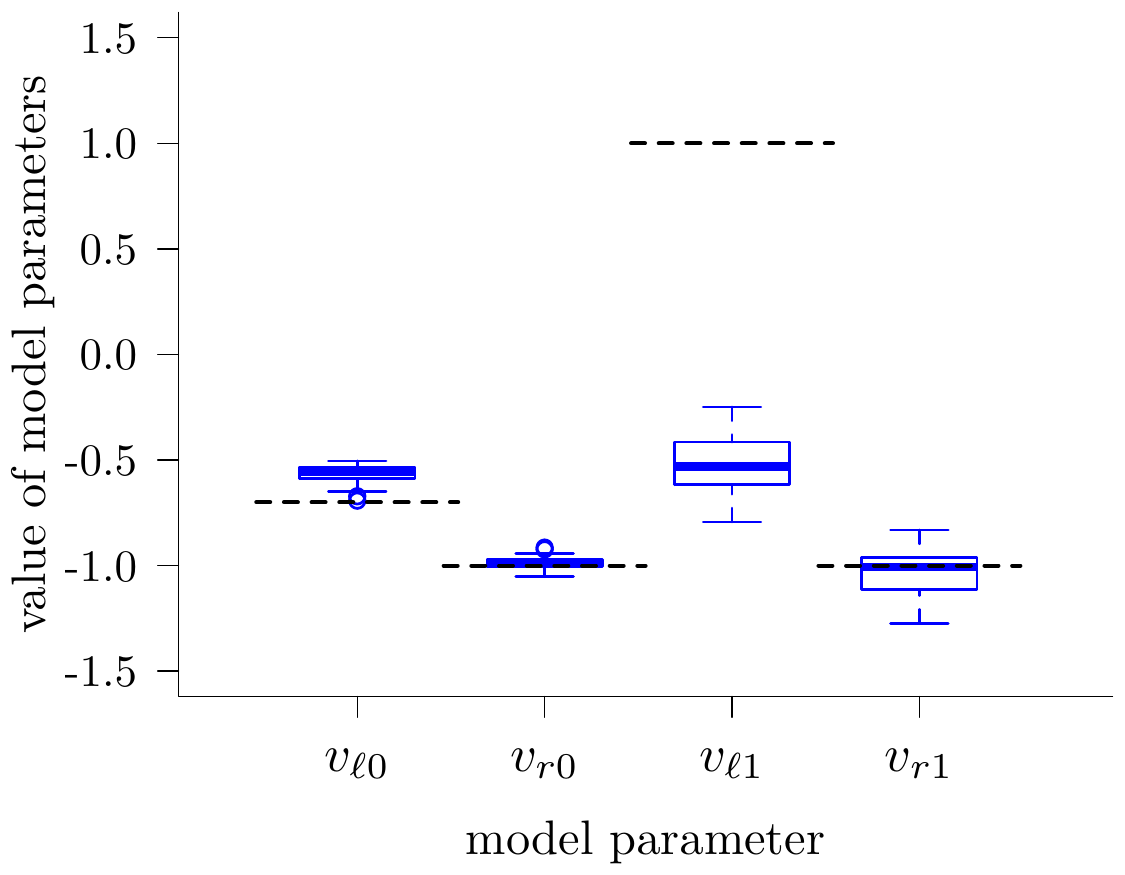}
		}
		\subfloat[(b) Object Clustering\label{fig:model_parameters_box_clustering_evolution}]{%
			\includegraphics[width=2.25 in]{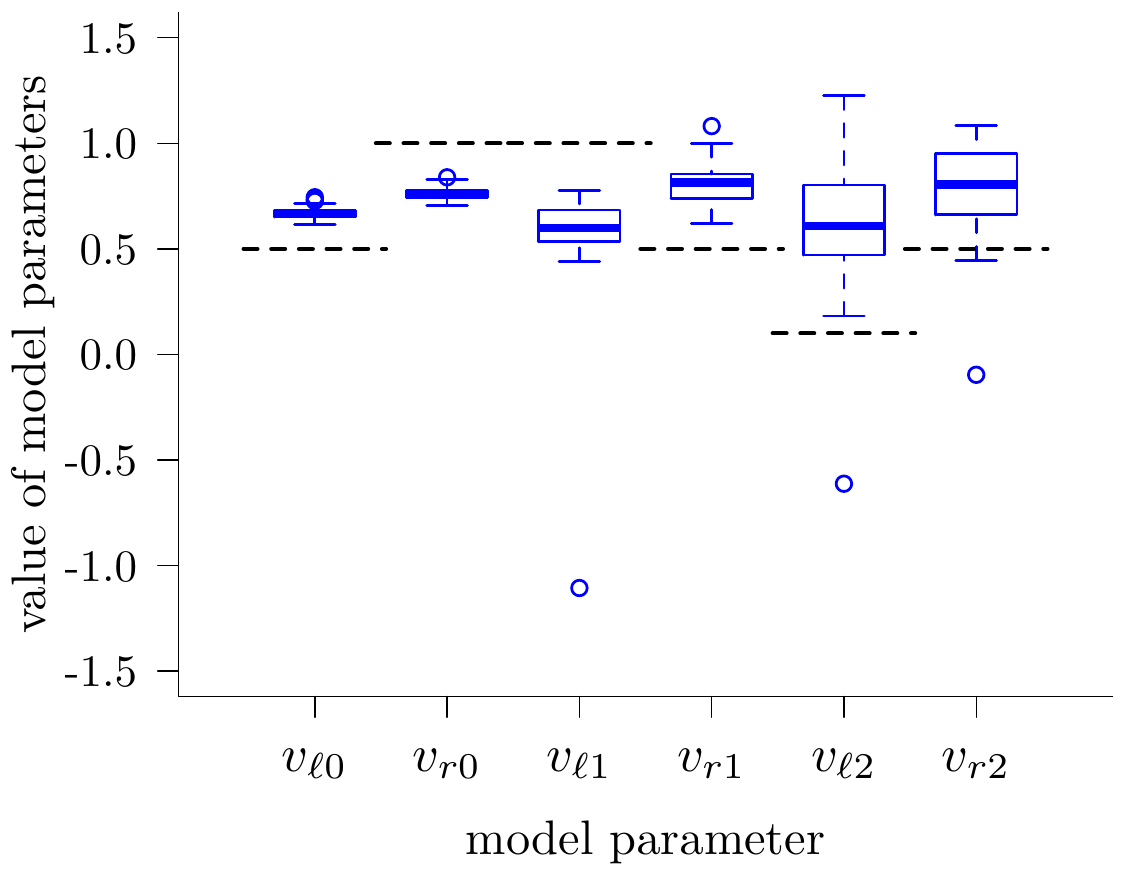}
		}
	\caption{\label{fig:model_parameters_box_evolution}\todo{Model parameters a metric-based evolutionary method inferred from swarms of simulated agents performing (a) aggregation and (b) object clustering}. Each box corresponds to the models with the highest fitness in the $1000$th generation of 30 runs. The dashed black lines correspond to the values of the parameters that the system is expected to learn (i.e., those of the agent).}
\end{figure}

%Appendix A (TODO) contains a formal analysis of Eq.~\ref{eq:square_error_metric}. It proves that the globally optimal model with respect to this metric differs substantially from the agent. In other words, the metric-based approach is deemed to fail; it will be unable to infer the correct behavior even if the global optimum of the metric is found.

\subsubsection{Mathematical analysis}

We begin our analysis by first analyzing an abstract version of the problem.
\begin{theorem}
Consider two binary random variables $X$ and $Y$. Variable $X$ takes value $x_1$ with probability $p$, and value $x_2$, otherwise.  Variable $Y$ takes value $y_1$ with the same probability $p$, and value $y_2$, otherwise. 
%in $\mathcal{X}=(x_1,x_2)$, value $x_1$ with probability $p$, and value $x_2$, otherwise. Variable $Y$ takes values in $\mathcal{Y}=(y_1,y_2)$, value $y_1$ with the same probability $p$, and value $y_2$, otherwise. 
Variables $X$ and $Y$ are assumed to be independent of each other. Assuming $y_1$ and $y_2$ are given, then the metric $D=\mathbb{E}\{(X-Y)^2\}$ has a global minimum at $X^*$ with $x_1^*=x_2^*=\mathbb{E}\{Y\}$.
% with $x^*_1 = x^*_2 = py_1+(1-p)y_2$. 
If $p\in(0,1)$, the solution is unique.
\end{theorem}

\begin{proof}
The probability of (i) both $x_1$ and $y_1$ being observed is $p^2$; (ii) both $x_1$ and $y_2$ being observed is $p (1 - p)$; (iii) both $x_2$ and $y_1$ being observed is $(1 - p) p$; (iv) both $x_2$ and $y_2$ being observed is $(1-p)^2$. The expected error value, $D$, is then given as
\begin{equation}
    D = p^2\left(x_1 - y_1\right)^2 + p (1 - p)\left(x_1 - y_2\right)^2 + (1-p) p \left(x_2 - y_1\right)^2 + (1-p)^2\left(x_2 - y_2\right)^2.
    \label{eq:expected_error}
\end{equation}

%\begin{aligned} \end{aligned}
To find the minimum expected error value, we set the partial derivatives w.r.t. $x_1$ and $x_2$ to $0$. For $x_1$, we have:
\begin{equation}
    \frac{\partial D}{\partial x_1} = 2p^{2}\left(x_1 - y_1\right) + 2 p (1 - p) \left(x_1 - y_2\right) = 0,
    \label{eq:expected_error_derivative}
\end{equation}
from which we obtain $x^*_1 = p y_1 + (1-p) y_2=\mathbb{E}\{Y\}$. Similarly, setting $\frac{\partial D}{\partial x_2} = 0$ we obtain $x^*_2 = p y_1 + (1-p) y_2=\mathbb{E}\{Y\}$. Note that at these values of $x_1$ and $x_2$, the second-order partial derivatives are both positive (assuming $p\in(0,1)$). Therefore, the (global) minimum of $D$ is at this stationary point.
%The stationary point is thus the (global) minimum.
\end{proof}

% Isn't that beautifully compact? (hopefully correct too..)
\begin{corollary}
If $p\in (0,1)$ and $y_1\ne y_2$, then $X^*\ne Y$. 
\end{corollary}
\begin{proof}
As $p\in (0,1)$, the only global minimum exists at $X^*$. As $x^*_1 = x^*_2$ and $y_1\ne y_2$, it follows that $X^*\ne Y$.
\end{proof}

\begin{corollary}
Consider two discrete random variables $X$ and $Y$ with values $x_1$, $x_2, \ldots, x_n$, and $y_1$, $y_2, \ldots, y_n$ respectively, $n>1$. 
%that that take values in $\mathcal{X}=(x_1,x_2,\ldots,x_n)$ and $\mathcal{Y}=(y_1,y_2,\ldots,y_n)$ respectively; $n>1$. 
Variable $X$ takes value $x_i$ with probability $p_i$ and variable $Y$ takes value $y_i$ with the same probability $p_i$, $i=1,2,\dots,n$, where $\sum\limits_{i=1}^{n} p_i = 1$ and $\exists i,j:y_i\ne y_j$. Variables $X$ and $Y$ are assumed to be independent of each other. Metric $D$ has a global minimum at
$X^*\ne Y$ with $x_1^*= x_2^* = \ldots = x_n^*=\mathbb{E}\{Y\}$. %x}_1 = \tilde{x}_2 = \cdots = \tilde{x}_n = \sum\limits_{i=1}^{n} p_i y_i$.
If all $p_i\in(0,1)$, then $X^*$ is unique.
\end{corollary}

\begin{proof}
This proof, which is omitted here, can be obtained by examining the first and second derivatives of a generalized version of Eq.~\eqref{eq:expected_error}. Rather than four ($2^2$) cases, there are $n^2$ cases to be considered.
\end{proof}
%\begin{proof}
%This result can be easily extended to the general case with $n$ sensor values; in this case:
%\begin{equation}
%    \tilde{x}_1 = \tilde{x}_2 = \cdots = \tilde{x}_n = \sum\limits_{k=1}^{n} p_k y_k,
%\end{equation}
%Where, $\sum\limits_{k=1}^{n} p_k = 1$. Note that the simulation results for both the aggregation ($n=2$) and object clustering ($n=3$) case studies show good agreement with this result.    
%\end{proof}

\begin{corollary}
Consider a sequence of pairs of binary random variables ($X_t$, $Y_t$), $t=1,\ldots,T$. Variable $X_t$ takes value $x_1$ with probability $p_t$, and value $x_2$, otherwise. Variable $Y_t$ takes value $y_1$ with the same probability $p_t$, and value $y_2\ne y_1$, otherwise. For all $t$, variables $X_t$ and $Y_t$ are assumed to be independent of each other. If all $p_t\in(0,1)$, then the metric $D=\mathbb{E}\left\{\sum_{t=1}^T(X_t-Y_t)^2\right\}$ has one global minimum at $X^*\ne Y$.
%with $x_1 = x_2 = py_1+(1-p)y_2$.
%; and if $y_1\ne y_2$, then $\tilde{X}\ne Y$.
\end{corollary}

\begin{proof}
The case $T=1$ has already been considered (Theorem 1 and Corollary 1). For the case $T=2$, we extend Eq.~\eqref{eq:expected_error_derivative} to take into account $p_1$ and $p_2$, and obtain
\begin{equation}
x_1(p_1^2+p_1-p_1^2+p_2^2+p_2-p_2^2) = y_1(p_1^2+p_2^2)+y_2(p_1-p_1^2+p_2-p_2^2).
\end{equation}
This can be rewritten as:
\begin{equation}
x_1=\frac{p_1^2+p_2^2}{p_1+p_2}y_1+\frac{p_1(1-p_1)+p_2(1-p_2)}{p_1+p_2}y_2.
\label{eq:two_probabilities}
\end{equation}
As $y_2\ne y_1$, $x_1$ can only be equal to $y_1$ if $p_1^2+p_2^2 = p_1 + p_2$, which is equivalent to $p_1(1-p_1)+p_2(1-p_2) = 0$. This is however not possible for any $p_1,p_2\in(0,1)$. Therefore, $X^*\ne Y$.\footnote{Note that in the case of $p_1=p_2$, Eq.~\eqref{eq:two_probabilities} simplifies to $x^*_1 = py_1+(1-p)y_2$, as already shown by Theorem 1. For $p_1\ne p_2$, it can be shown that $x_1^*$ and $x_2^*$ are not necessarily equal to $\mathbb{E}\{Y\}$.}

For the general case, $T\ge 1$, the following equation can be obtained (proof omitted).
\begin{equation}
x_1=\frac{\sum_{t=1}^{T}p_t^2}{\sum_{t=1}^{T}p_t}y_1+\frac{\sum_{t=1}^{T}p_t(1-p_t)}{\sum_{t=1}^{T}p_t}y_2.
\end{equation}
The same argument applies---$x^*_1$ cannot be equal to $y_1$. Therefore, $X^*\ne Y$.
\end{proof}

%\begin{remark}
\textbf{Implications for our scenario:} The metric-based approach considered in this paper is unable to infer the correct behavior of the agent. In particular, the model that is globally optimal w.r.t. the expected value for the error function defined by Eq.~\eqref{eq:square_error_metric} is different from the agent. This observation follows from Corollary 1 for the aggregation case study (two sensor states), and from Corollary 2 for the object clustering case study (three sensor states).
%it can be shown that the model that is globally optimal w.r.t. the expected value for the error function defined by Eq.~\eqref{eq:square_error_metric} is different from the agent. This follows from Corollary 1. It is applicable to both case studies (aggregation and object clustering), as the findings extend to the case where the sensors have more than 2 states (Corollary 2). 
It exploits the fact that the error function is of the same structure as the metric in Corollary 3---a sum of square error terms. The summation over time is not of concern---as was shown in Corollary 3, the distributions of sensor reading values (inputs) of the agent and of the model do not need to be stationary. However, we need to assume that for any control cycle, the actual inputs of agents and models are not correlated with each other. Note that the sum in Eq.~\eqref{eq:square_error_metric} comprises two square error terms: one for the linear speed of the agent, and the other for the angular speed. As our simulated agents employ a differential drive with unconstrained motor speeds, the linear and angular speeds are decoupled. In other words, the linear and angular speeds can be chosen independently of each other, and optimized separately. This means that Eq.~\eqref{eq:square_error_metric} can be thought of as two separate error functions: one pertaining to the linear speeds, and the other to the angular speeds.
%In particular, if the agent receives input $I=0$ with probability $p$ and input $I=1$ with probability $p-1$, then the optimal model parameters are $pv_{l0}+(1-p)v_{l1}$ for the left wheel and $pv_{r0}+(1-p)v_{r1}$ for the right wheel. 
%As $v_{l0}\ne v_{l1}$, if follows that the model that is globally optimal with respect to the cost function must be different from the agent (see Corollary 1.1).
%\end{remark}

\subsubsection{Comparison with \textit{Turing Learning}}

To verify whether the theoretical result (and its assumptions) holds in practice, we used an evolutionary algorithm with a single population of models. The algorithm was identical to the \finaltodo{model optimization} sub-algorithm in \textit{Turing Learning} except for the fitness calculation, where the metric of Eq.~\ref{eq:square_error_metric} was employed. We performed 30 evolutionary runs for each case study. Each evolutionary run lasted $1000$ generations. The \finaltodo{simulation setup and} number of fitness evaluations for the models \finaltodo{were} kept the same as in \textit{Turing Learning}.

Fig.~\subref*{fig:model_parameters_box_aggregation_evolution} shows the parameter distribution of the evolved models with highest fitness in the last generation over $30$ runs. The distributions of the evolved parameters corresponding to $I=0$ and $I=1$ are similar. This phenomenon can be explained as follows. \todo{In the identification problem that we consider, the method has
%metric-based evolutionary algorithm has 
no knowledge of the input, that is, whether the agent perceives another agent ($I=1$) or not ($I=0$). \textcolor{black}{This is consistent with \textit{Turing Learning} as the classifiers that are used to optimize the models also do not have any knowledge of the inputs.} The metric-based algorithm \todo{seems to} construct controllers that do not respond differently to either input, but work as good as it gets on average, that is, for the particular distribution of inputs, $0$ and $1$. For the left wheel speed both parameters are approximately $-0.54$. This is almost identical to the weighted mean ($-0.7*0.912 + 1.0*0.088=-0.5504$), which takes into account that parameter $v_{\ell0}=-0.7$ is observed around 91.2\% of the time, whereas parameter $v_{\ell1}=1$ is observed around 8.8\% of the time (see also Section~\ref{sec:analysis_evolved_models}).}
The parameters related to $I=1$ evolved well as the agent's parameters are identical regardless of the input \todo{($v_{r0}=v_{r1}=-1.0$)}. For both $I=0$ and $I=1$, the evolved parameters show good agreement with Theorem~1. As the
model and the agents are only observed for $\unit[10]{s}$ in
the simulation trials, the probabilities of seeing
a $0$ or a $1$ are nearly constant throughout
the trial. Hence, this scenario approximates very
well the conditions of Theorem 1, and the effects
of non-stationary probabilities on the optimal point (Corollary 3) are minimal. Similar results were found when inferring the object clustering behavior (see Fig.~\subref*{fig:model_parameters_box_clustering_evolution}). 

By comparing Figs.~\ref{fig:model_parameters_box} and~\ref{fig:model_parameters_box_evolution}, one can see that
%This means that 
\textit{Turing Learning} outperforms the metric-based evolutionary algorithm in terms of model accuracy in both case studies. As argued before, due to the unpredictable \finaltodo{interactions in swarms} the traditional metric-based method is \finaltodo{not suited for inferring the behaviors}. 
%Instead, \textit{Turing Learning} automatically synthesizes its own metric based on the data observed.
% The analysis of the obtained classifiers in in the following section \todo{suggests} that the overall classification is more sophisticated than the least square error \todo{approach}. 
%The means (standard deviations) of the AEs in the parameters were (from left to right in Fig.~\ref{fig:model_parameters_box_aggregation_evolution}): $0.13$ ($0.05$), $0.03$ ($0.02$), $1.52$ ($0.14$) and $0.08$ ($0.07$). 
%This may be due to the phenomenon that the data corresponding to $I=1$ is overwhelmed by the data corresponding to $I=0$, as an agent spends a much larger proportion of its time seeing nothing ($I=0$) than seeing other agents ($I=1$).

\subsection{Generality of \textit{Turing Learning}}\label{sec:generality_turing_learning}

\textcolor{black}{In the following, we present four orthogonal studies testing the generality of \textit{Turing Learning}. The experimental setup in each section is identical to that described previously (see Section~\ref{sec:simulation_setup}), except for the modifications discussed within each section.}

\subsubsection{Simultaneously inferring control and morphology}\label{sec:evolving_control_and_morphology}

\begin{figure}[!t]%
	\centering
	\hspace*{\fill}
		\subfloat[(a) \label{fig:Angle_I=0}]{%
			\includegraphics[width=1.6 in]{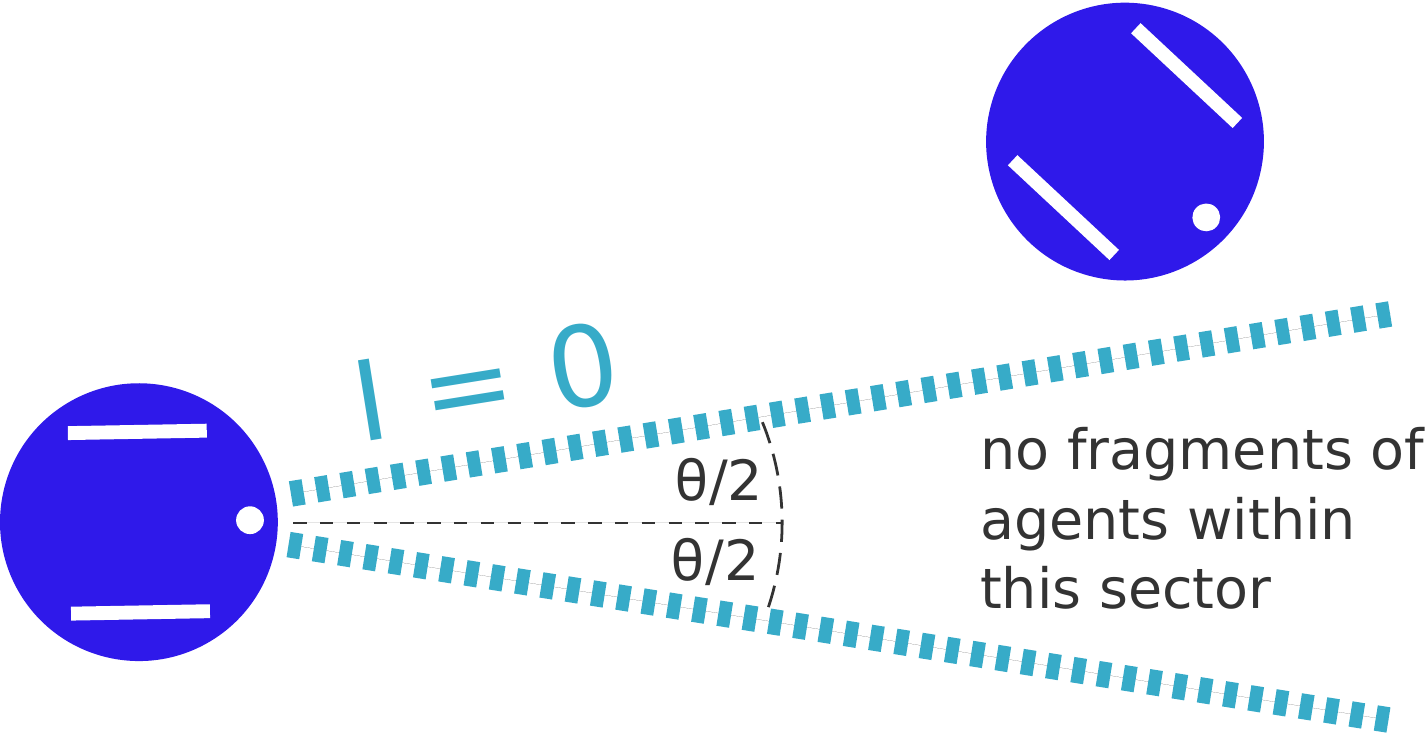}
		}\hfill
		\subfloat[(b) \label{fig:Angle_I=1}]{%
			\includegraphics[width=1.6 in]{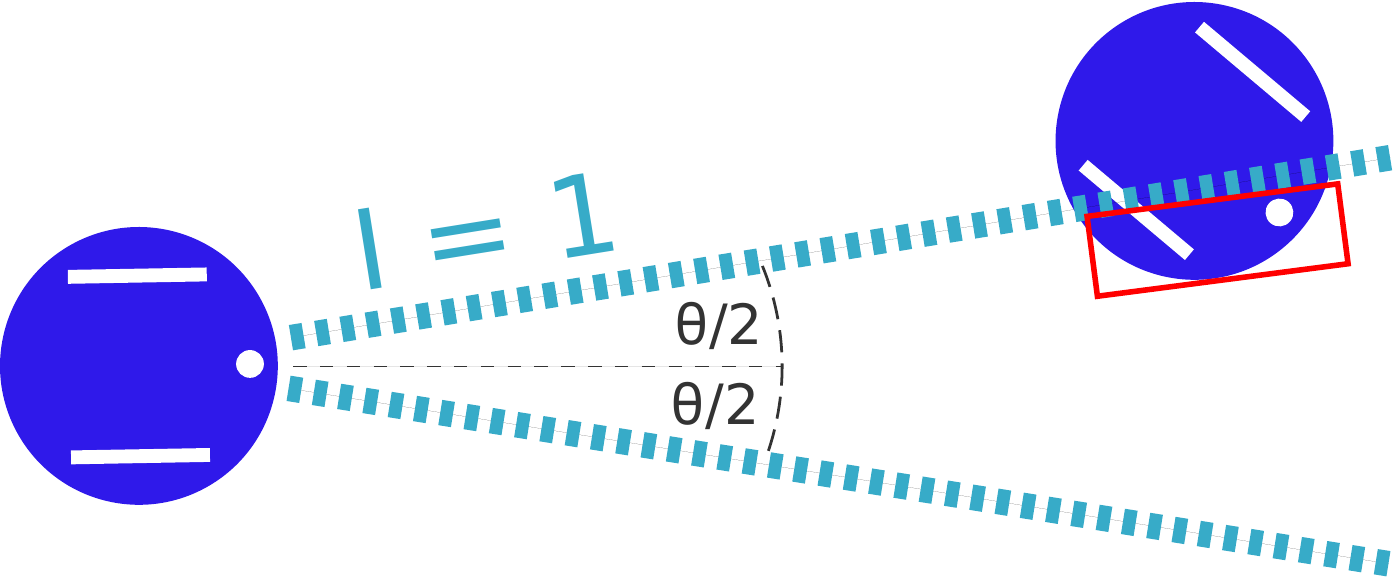}
		}%
	\hspace*{\fill}
		\caption{A diagram showing the angle of view of the agent's sensor investigated in Section~\ref{sec:evolving_control_and_morphology}.}
		\label{fig:Angle_I}
\end{figure}%
% TODO: change "agents" by "other individuals" in Figure 10a

In the previous sections, we assumed that we fully knew the agents' morphology, and only their behavior (controller) was to be identified. We now present a variation where one aspect of the morphology is also unknown. The replica, in addition to the four controller parameters, takes a parameter $\theta\in\left[0,2\pi\right]\unit{rad}$, which determines the horizontal field of view of its sensor, as shown in Fig.~\ref{fig:Angle_I} (the sensor is still binary). Note that the agents' line-of-sight sensors of the previous sections can be considered as sensors with \todo{a field} of view of $\unit[0]{rad}$.

The models now have five parameters. As before, we let \todo{\textit{Turing Learning}} run in an unbounded search space (i.e., now, $\mathbb{R}^5$). However, as $\theta$ is necessarily bounded, before a model is executed on a replica, the parameter corresponding to $\theta$ is mapped to the range $(0, 2\pi)$ using an appropriately scaled logistic sigmoid function. %(Eq.~\eqref{equ:logistic_sigmoid}). 
The controller parameters are directly passed to the replica. In this setup, the classifiers observe the individuals for $\unit[100]{s}$ in each trial (preliminary results indicated that this setup required a longer observation time). 

\begin{figure}[!t]%
	\centering
		\subfloat[(a) \label{fig:model_parameters_box_aggregation_0degree}]{
			\includegraphics[width=2.25 in]{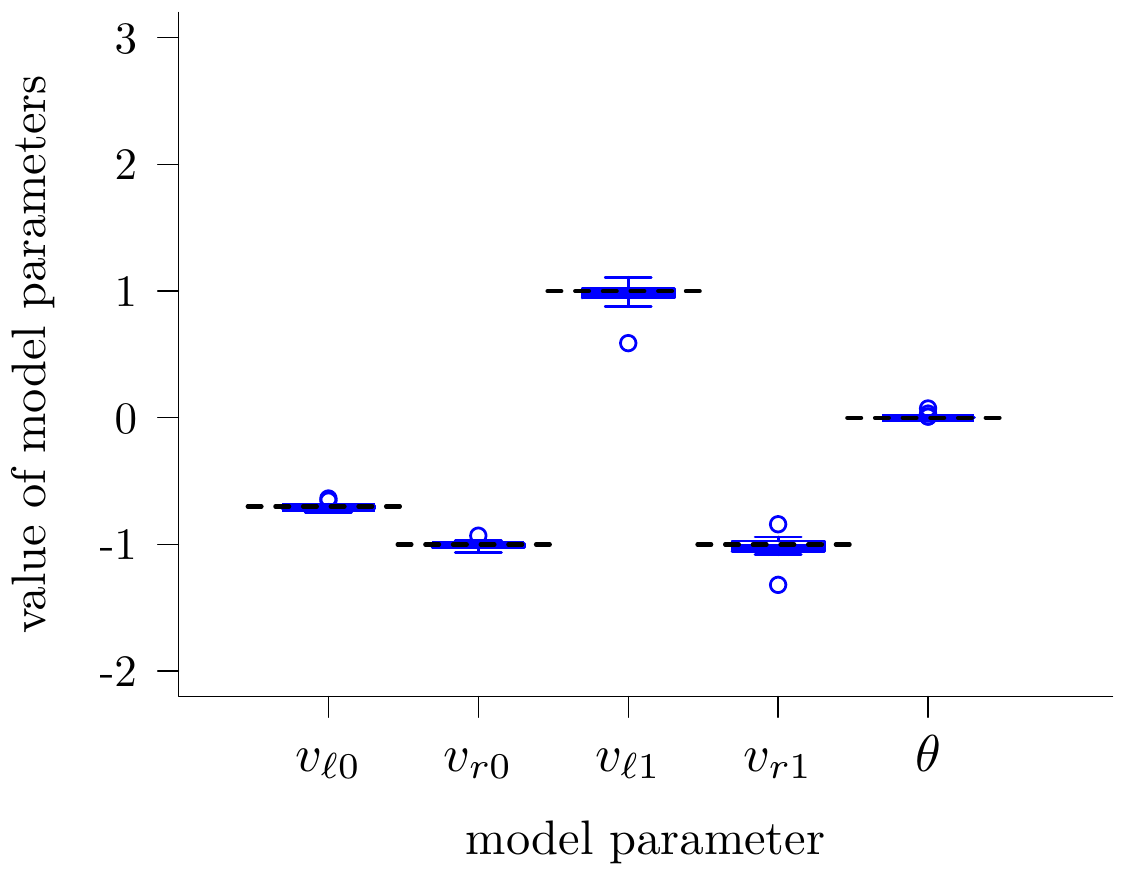}
		}%
		\subfloat[(b) \label{fig:model_parameters_box_aggregation_60degree}]{
			\includegraphics[width=2.25 in]{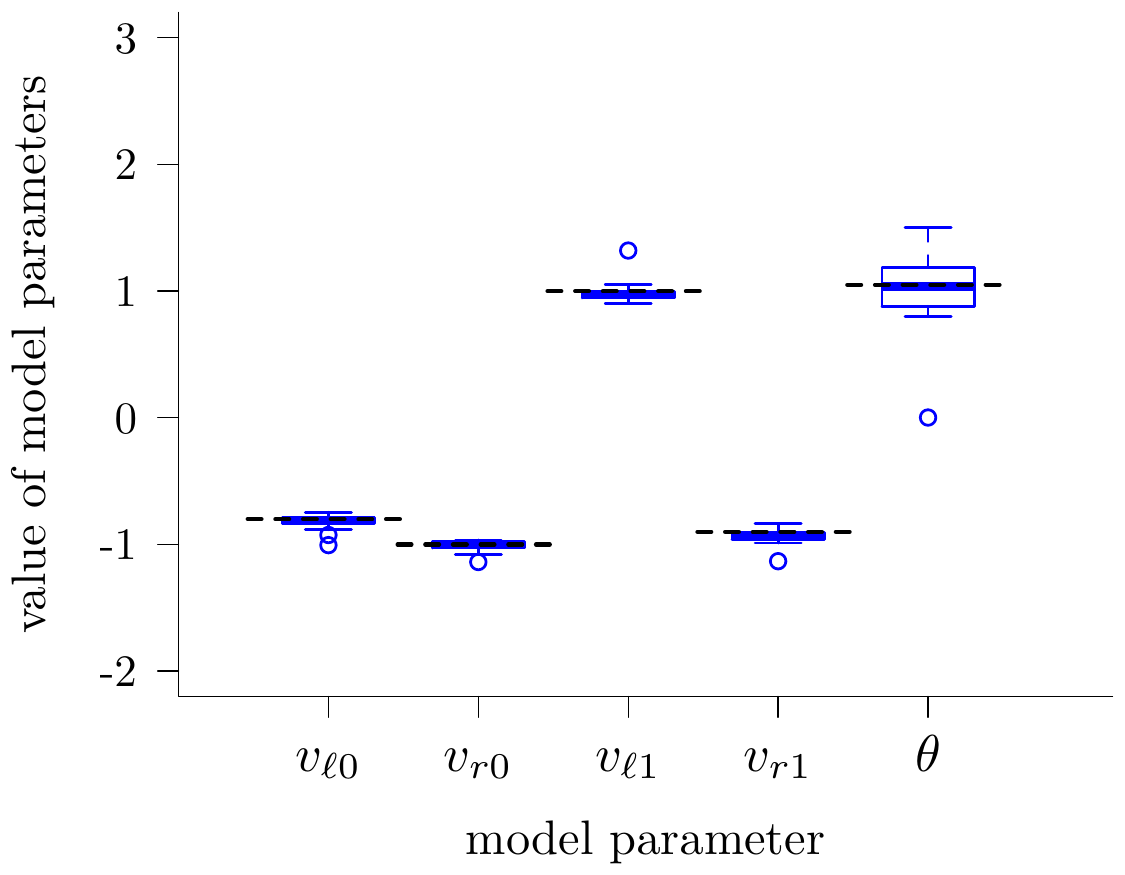}%
		}%
		\caption{\todo{\textit{Turing Learning} simultaneously inferring control and morphological parameters (field of view). The agents' field of view is} (a) $\unit[0]{rad}$ and (b) $\unit[\pi/3]{rad}$. Boxes show distributions for the models with the highest subjective fitness in the $1000$th generation over $30$ runs. Dashed black lines indicate \todo{the ground truth}.}
		\label{fig:model_parameters_box_aggregation_angle}
\end{figure}%

Fig.~\subref*{fig:model_parameters_box_aggregation_0degree} shows the parameters of the subjectively best models in the last ($1000$th) generations of $30$ runs. The means (standard deviations) of the AEs in each model parameter are as follows: $0.02$ ($0.01$), $0.02$ ($0.02$), $0.05$ ($0.07$), $0.06$ ($0.06$), and $0.01$ ($0.01$). All parameters including $\theta$ are still learned with high accuracy.

The case where the true value of $\theta$ is $\unit[0]{rad}$ is an edge case, because given an arbitrarily small $\epsilon>0$, the logistic sigmoid function maps an \todo{unbounded} domain of values onto $(0,\epsilon)$. 
%(i.e., $\forall\epsilon>0.\, \exists x_0.\, \forall x<x_0.\, \textrm{sig}\,x < \epsilon $). 
This makes it simpler for \todo{\textit{Turing Learning} to infer} this parameter. For this reason, we also consider another scenario where the agents' angle of view is $\unit[\pi/3]{rad}$ rather than $\unit[0]{rad}$. The controller parameters for achieving aggregation in this case are different from those in Eq.~\eqref{eq:aggregation_optimal_controller}. They were found by rerunning a grid search with the modified sensor. Fig.~\subref*{fig:model_parameters_box_aggregation_60degree} shows the results from $30$ runs with this setup. The means (standard deviations) of the AEs in each parameter are as follows: $0.04$ ($0.04$), $0.03$ ($0.03$), $0.05$ ($0.06$), $0.05$ ($0.05$), and $0.20$ ($0.19$). The controller parameters are still learned with good accuracy. The accuracy in the angle of view is noticeably lower, but still reasonable.

\subsubsection{\textcolor{black}{Inferring behavior without assuming a known control system structure}}\label{sec:evolve_neural_network_model}
\captionsetup[subfigure]{labelformat=empty}  
\begin{figure}[!t]
	\centering
	\subfloat[(a) One hidden neuron]{  
		\includegraphics[width = 1.49 in]{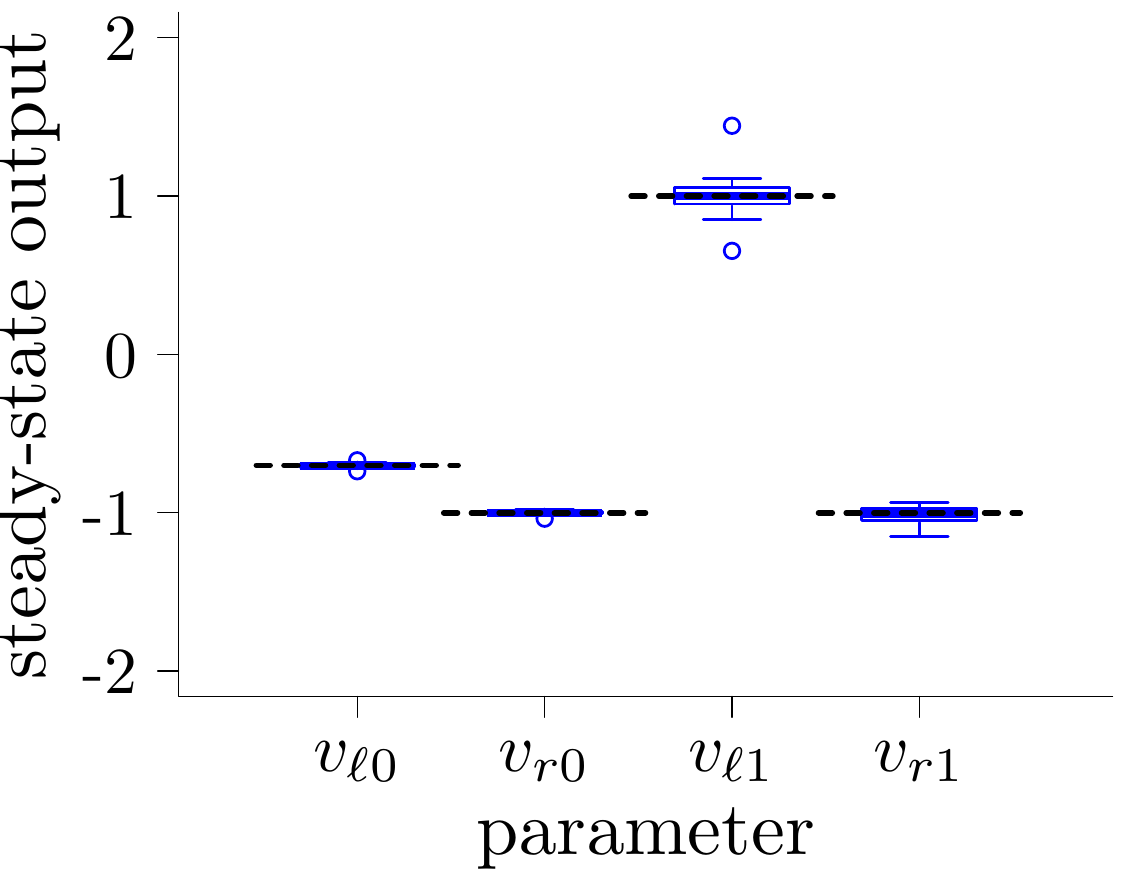}  %1.25
	}
	\subfloat[(b) Three hidden neurons]{
		\includegraphics[width = 1.49 in]{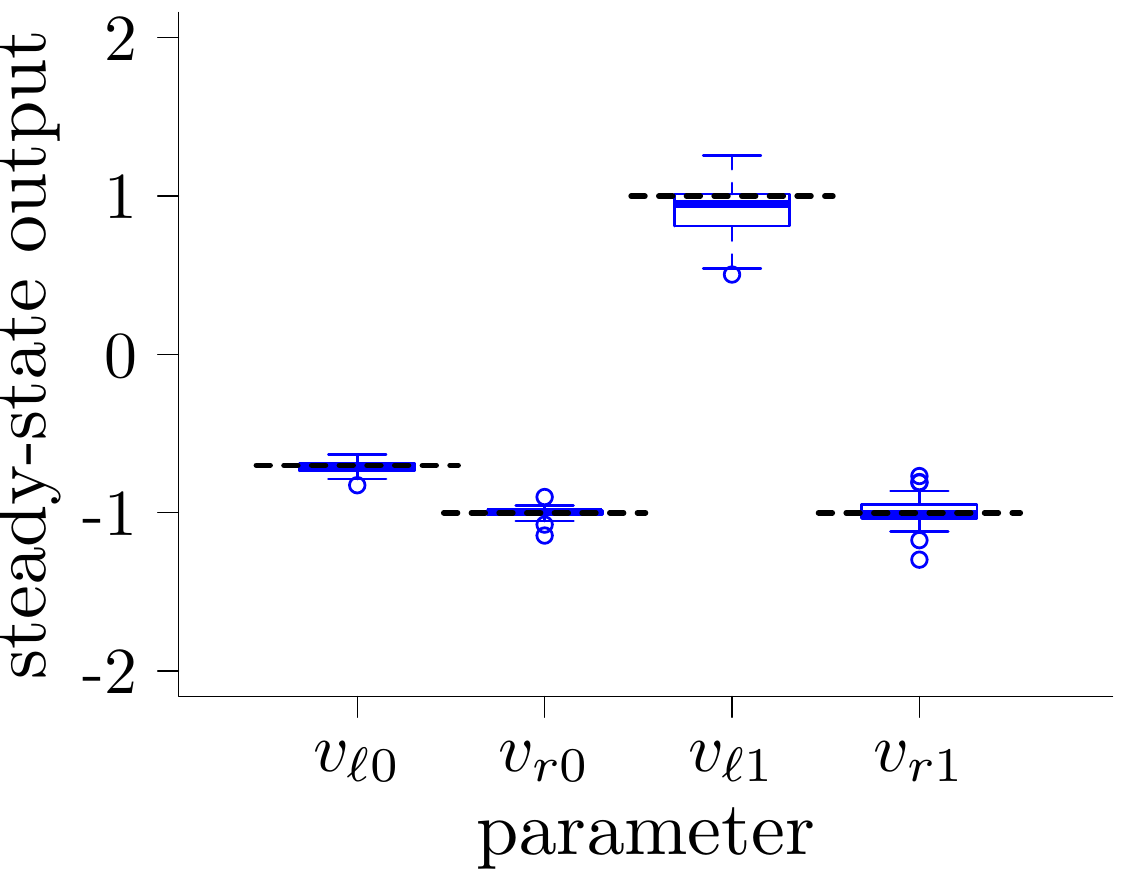}
	}
	\subfloat[(c) Five hidden neurons]{
		\includegraphics[width = 1.49 in]{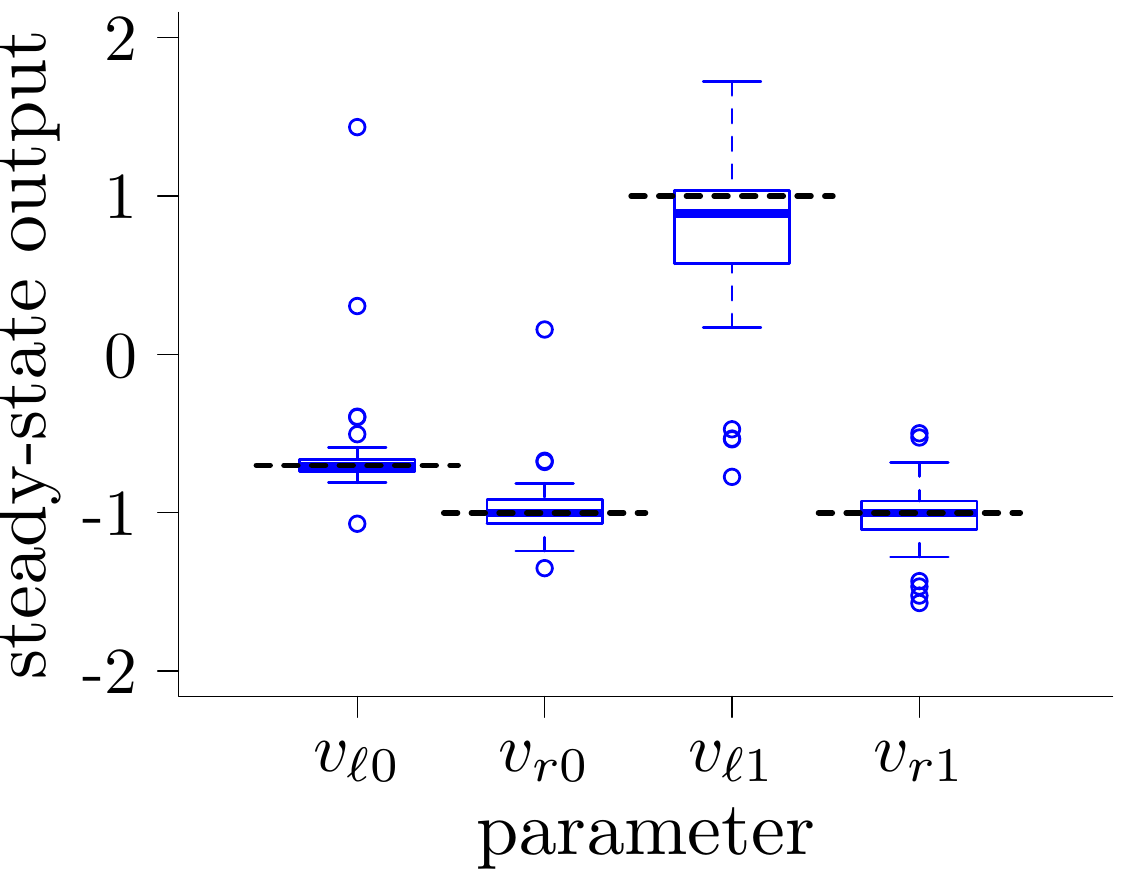}
	}
	\caption{\textit{Turing Learning} can infer an agent's behavior without assuming its control system structure to be known. These plots show the steady-state outputs (in the $20$th time step) of the inferred neural networks with the highest subjective fitness in the $1000$th generation of 30 simulation runs. Two outliers in (c) are not shown.}
	\label{fig:steady_output_RNN}
\end{figure}

\captionsetup[subfigure]{labelformat=empty}  
\begin{figure}[!t]
	\centering
	\subfloat[(a) One hidden neuron]{  
		\includegraphics[width = 1.49 in]{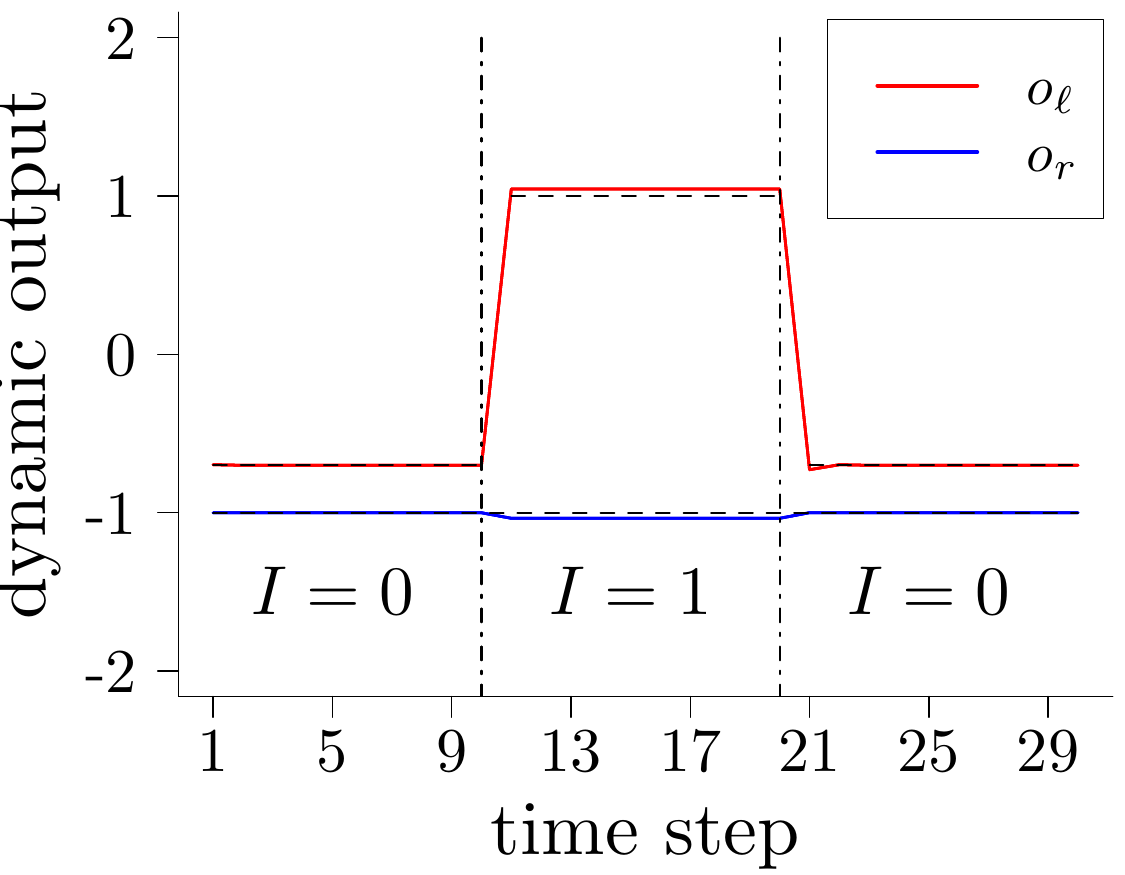}  %1.25
	}
	\subfloat[(b) Three hidden neurons]{
		\includegraphics[width = 1.49 in]{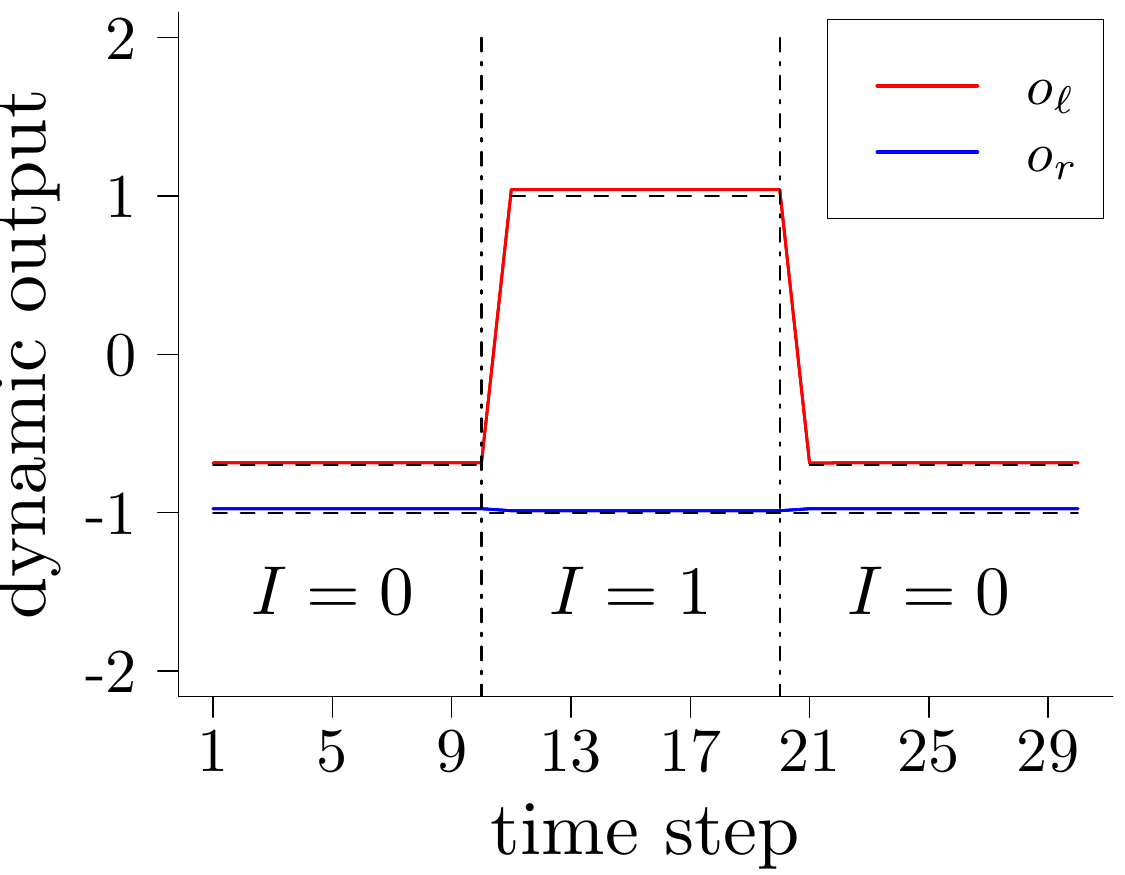}
	}
	\subfloat[(c) Five hidden neurons]{
		\includegraphics[width = 1.49 in]{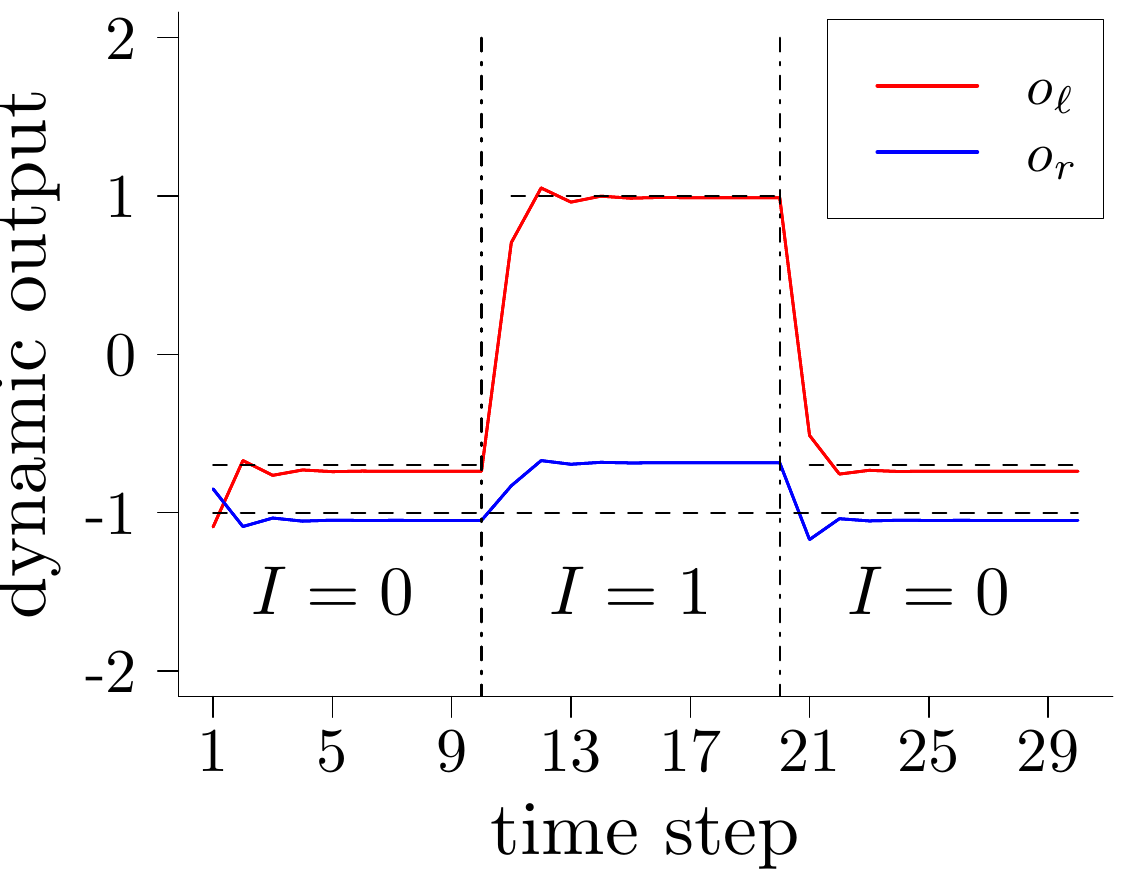}
	}
	\caption{Dynamic outputs of the inferred neural network with median performance. The network's input in each case was $I=0$ (time steps 1--10), $I=1$ (time steps 11-20) and $I = 0$ (time steps 21--30). See text for details.}
	\label{fig:dynamic_output_RNN}
\end{figure}

\textcolor{black}{In the previous sections, we assumed the agent's control system structure to be known and only inferred its parameters. To further investigate the generality of \textit{Turing Learning}, we now represent the model in a more general form, namely a (recurrent) Elman neural network~\citep{Elman1990}. The network inputs and outputs are identical to those used for our reactive models. In other words, the Elman network has one input ($I$) and two outputs representing the left and right wheel speed of the robot. A bias is connected to the input and hidden layers of the network, respectively. 
We consider three network structures with one, three, and five hidden neurons, which correspond, respectively, to $7$, $23$ and $47$ weights to be optimized. Except for a different number of parameters to be optimized, the experimental setup is equivalent in all aspects to that of Section~\ref{sec:simulation_setup}.}

\textcolor{black}{We first analyze the steady-state behavior of the inferred network models. To obtain their steady-state outputs, we fed them with a constant input ($I=0$ or $I=1$ depending on the parameters) for $20$ time steps. Fig.~\ref{fig:steady_output_RNN} shows the outputs in the final time step of the inferred models with the highest subjective fitness in the last generation in $30$ runs for the three cases. In all cases, the parameters of the swarming agent can be inferred correctly with reasonable accuracy. More hidden neurons lead to worse results, probably due to the larger search space.} 

\textcolor{black}{We now analyze the dynamic behavior of the inferred network models. Fig.~\ref{fig:dynamic_output_RNN} shows the dynamic output of $1$ of the $30$ neural networks. The chosen neural network is the one exhibiting the median performance according to metric $\sum \limits_{i=1}^{4} \sum \limits_{t=1}^{20} (o_{it} - p_{i})^2$, where $p_i$ denotes the $i$th parameter in Eq.~\eqref{eq:aggregation_optimal_controller}, and $o_{it}$ denotes the output of the neural network in the $t$th time step corresponding to the $i$th parameter in Eq.~\eqref{eq:aggregation_optimal_controller}.  
The inferred networks react to the inputs rapidly and maintain a steady-state output (with little disturbance). The results show that \textit{Turing Learning} can infer the behavior without assuming the agent's control system structure to be known.}

\subsubsection{\todo{Separating the replicas and the agents}}
\label{sec:separating_replicas_agents}

\todo{In our two case studies, the replica was mixed into \textcolor{black}{a} group of agents. \textcolor{black}{In the context of animal behavior studies, a robot replica may be introduced into a group of animals and recognized as a conspecific~\citep{J.Halloy2007, Faria2010}. However, if behaving abnormally, the replica may disrupt the behavior of the swarm~\citep{Bjerknes2013}.} For the same reason, the insertion of a replica that exhibits different behavior or is not recognized as conspecific may disrupt the \todo{behavior of the swarm} and hence the models obtained may be biased.}
%In this case, it is suggested to use the minimum number of replicas possible}. 
%For the case studies reported in this paper, \todo{additional tests revealed no such bias.
In this case, an alternative method would be to isolate the influence of the replica(s). We performed an additional simulation study where agents and replicas were never mixed. Instead, each trial focused on either a group of agents, or of replicas. All replicas in a trial executed the same model.
%, with the latter executing the same model. 
%In this setup, to evaluate a model, we performed two trials: one with only replicas each executing the model, and the other with only agents (i.e., no replicas). 
The group size was identical in both cases. The tracking data of the agents and the replicas from each sample were then fed into the classifiers for making judgments. 
%however the method itself is not sensitive to this number

\begin{figure}[!t]
	\centering
	\includegraphics[width=3.0 in]{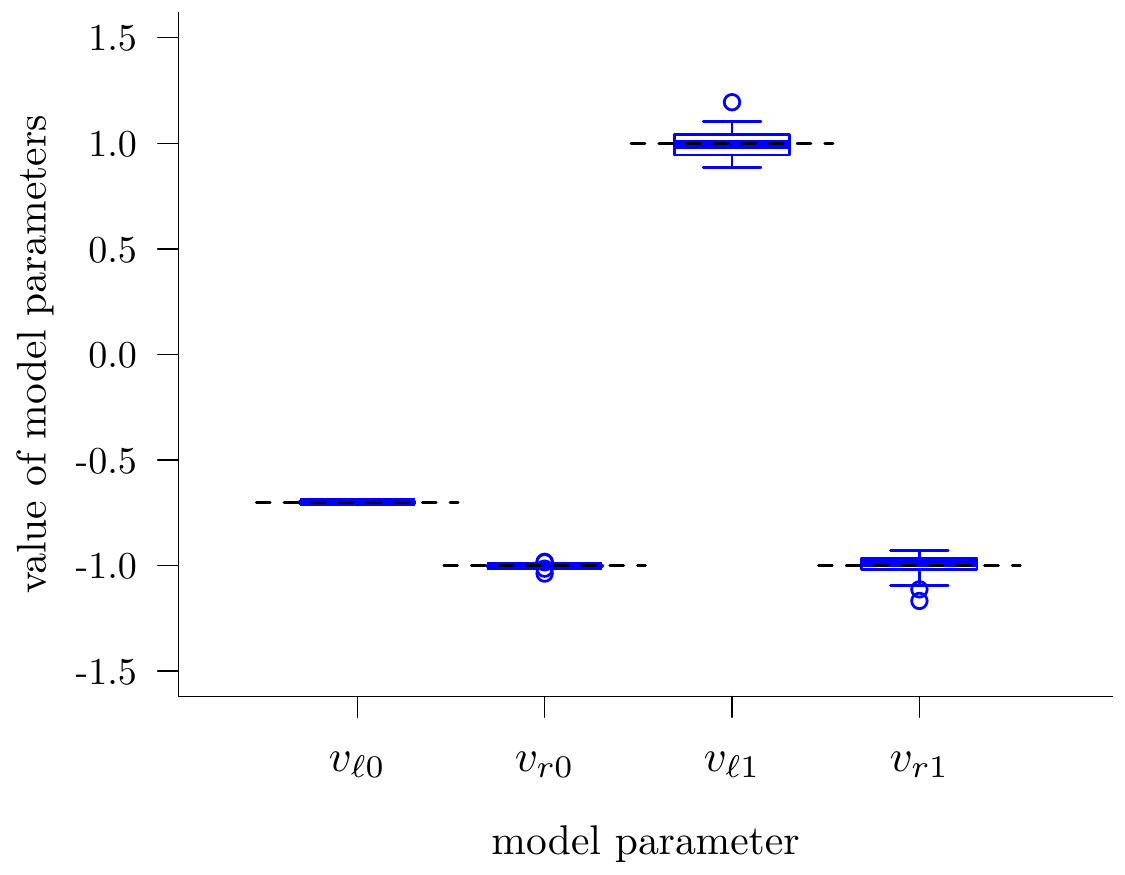}
	\caption{Model parameters inferred by a variant of \textit{Turing Learning} that observes swarms of aggregating agents and swarms of replicas in isolation, thereby avoiding potential bias. Each box corresponds to the models with the highest subjective fitness in the $1000$th generation of 30 simulation runs.\label{fig:model_parameters_box_aggregation_separate}}
\end{figure}

The distribution of the \todo{inferred} model parameters is shown in Fig.~\ref{fig:model_parameters_box_aggregation_separate}. The results show that \textit{Turing Learning} can still identify the model parameters well. There is no significant difference between either approach in the case studies considered in this paper. \todo{The method of separating replicas and agents is recommended if potential biases are suspected.}
%An appropriate strategy would be to isolate the influence of the replica. In particular, to evaluate a model one could perform two trials, one with only replicas each executing the model and the other with only agents. The data of the replicas and agents from each trial could then be fed into the classifiers for making judgments. Some preliminary results suggest that there is no significant difference between either approach for the case studies considered in this paper.

\subsubsection{Inferring other reactive behaviors}\label{sec:infer_other_behaviors}
% TODO: do we state how many agents are used?

The aggregation controller that agents used in our case study was originally synthesized by searching over \textcolor{black}{the parameter space defined in Eq.~\eqref{controller:form} with $n=2$},
% $\left[-1,1\right]^4$
using a metric to assess the swarm's global performance~\citep{Gauci2014_ijrr}. \textcolor{black}{Each of these points produces a global behavior. Some of these behaviors are particularly interesting, such as the circle formation behavior reported in~\citep{Melvin_DARS2014}.}
%It has been shown that there exists other} points in this space with global behaviors other than aggregation, but still `meaningful' to a human observer (e.g., circle formation~\citep{Melvin_DARS2014}). 
%Yet, other points
%\textcolor{black}{The points can be represented using Eq.~\eqref{controller:form} with $n =2$.} %\textcolor{black}{A visualization of the fitness was shown in~\citep{Gauci2014_ijrr}.}
%However, many other controllers may not lead to `meaningful' global behaviors.

We now investigate whether \todo{\textit{Turing Learning} can infer} arbitrary controllers in this space, irrespective of the global behaviors they lead to. We generated $1000$ controllers randomly in the \textcolor{black}{parameter space defined in Eq.~\eqref{controller:form}}, with uniform distribution. For each controller, we performed one run, and selected the subjectively best model in the last ($1000$th) generation. 

\begin{figure}[!t]%
	\centering
		\subfloat[(a) \label{fig:MAE_histgram_random_controllers}]{%
			\includegraphics[width=2.3 in]{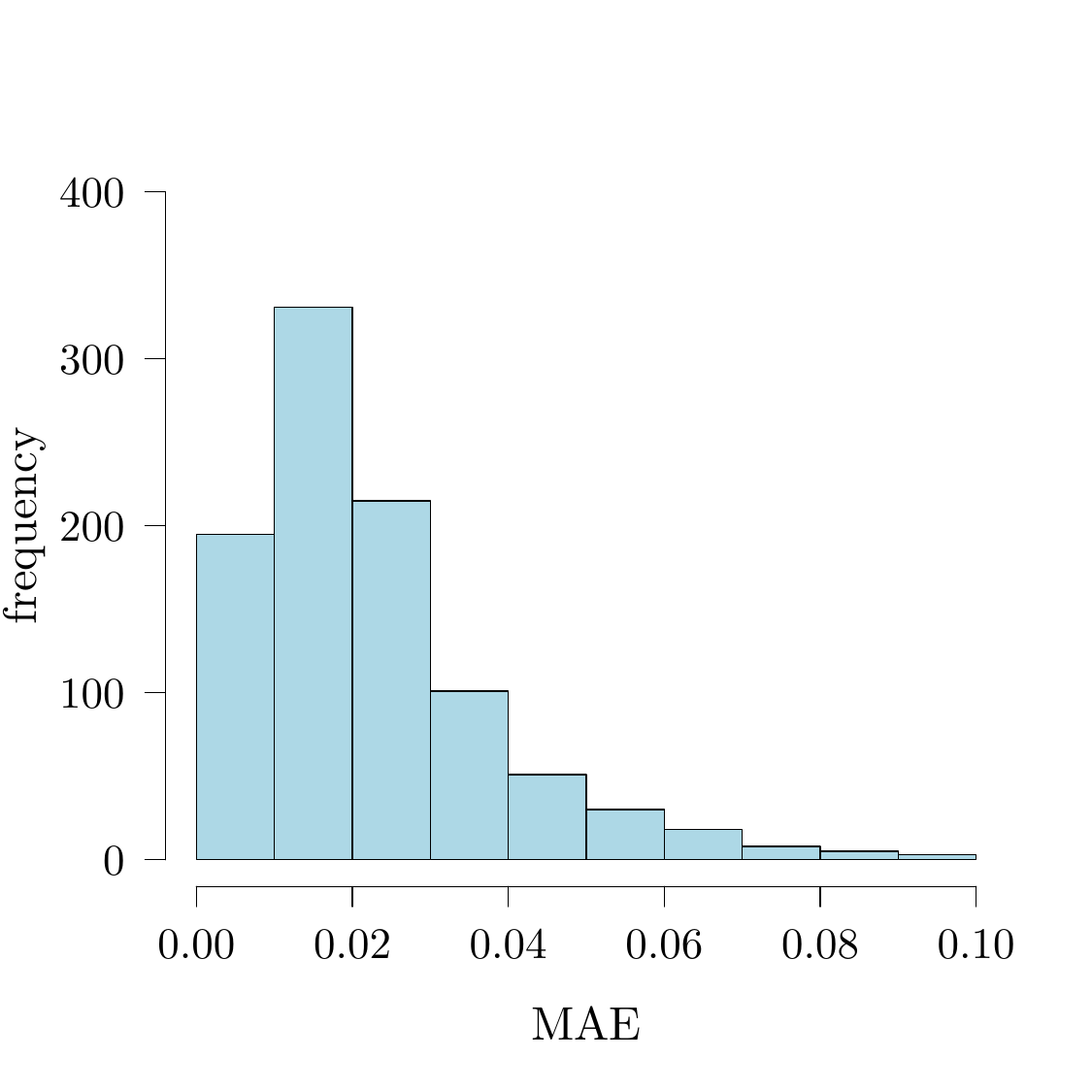}
		}
		\subfloat[(b) \label{fig:AE_box_random_controllers}]{%
			\includegraphics[width=2.3 in]{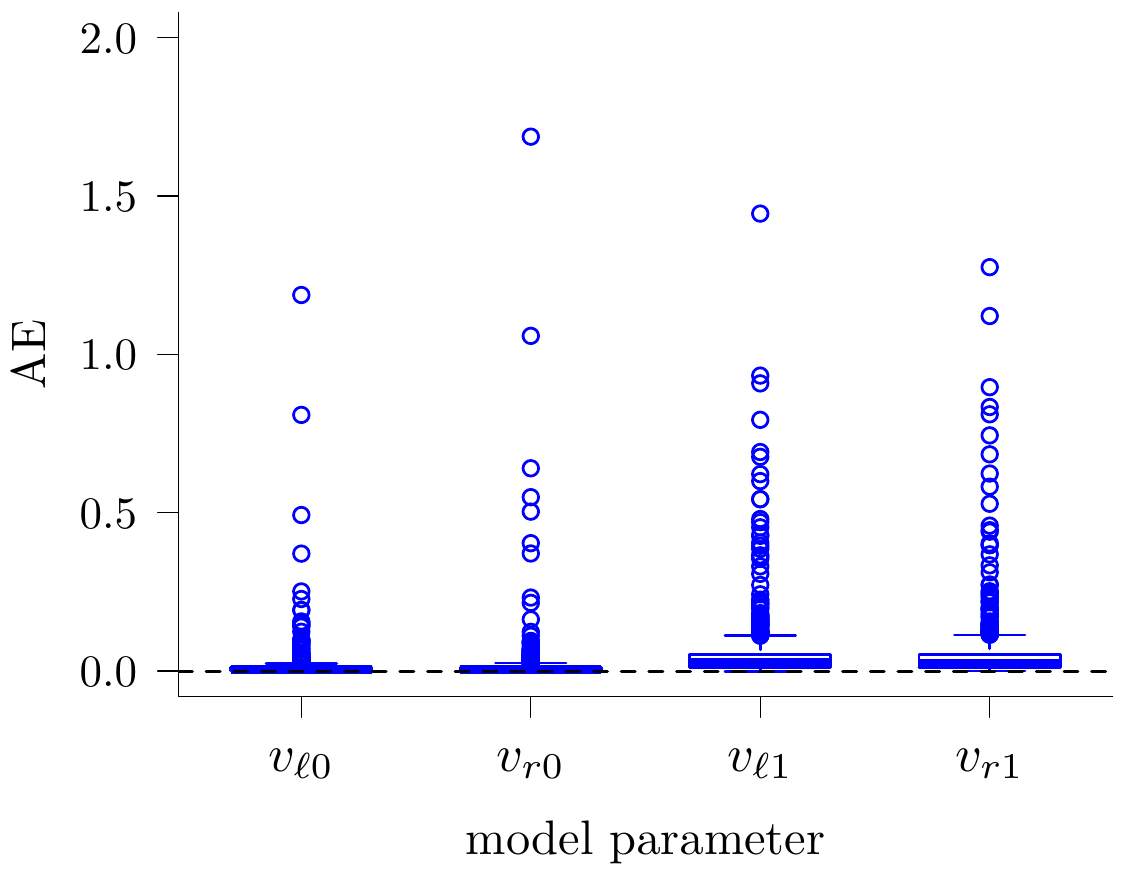}
		}
		\caption{\textit{Turing Learning} inferring the models for 1000 randomly generated agent behaviors. For each behavior, one run of \textit{Turing Learning} was performed and the model with the highest subjective fitness after $1000$ generations was considered. (a) Histogram of the models' MAE (defined in Eq.~\eqref{eq:MAE}; $43$ points that have an MAE larger than $0.1$ are not shown); and (b) AEs (defined in Eq.~\eqref{eq:AE}) for each model parameter.}
		\label{fig:model_parameters_random_controllers}
\end{figure}

Fig.~\subref*{fig:MAE_histgram_random_controllers} shows a histogram of the MAE of the \todo{inferred} models. The distribution has a single mode close to zero, and decays rapidly for increasing values. Over $89\%$ of the $1000$ cases have an error below $0.05$. This suggests that the accuracy of~\textit{Turing Learning} is not highly sensitive to the particular behavior under investigation (i.e., most behaviors are learned equally well). Fig.~\subref*{fig:AE_box_random_controllers} shows the AEs of each model parameter. The means (standard deviations) of the AEs in each parameter are as follows: $0.01$ ($0.05$), $0.02$ ($0.07$), $0.07$ ($0.6$), and $0.05$ ($0.2$). We performed a statistical test on the AEs between the model parameters corresponding to $I=0$ ($v_{\ell0}$ and $v_{r0}$) and $I=1$ ($v_{\ell1}$ and $v_{r1}$). The AEs of the \todo{inferred} $v_{\ell0}$ and  $v_{r0}$ are significantly lower than those of $v_{\ell1}$ and  $v_{r1}$. This is likely due to the reason reported in Section~\ref{sec:analysis_evolved_models}; that is, an agent \finaltodo{is likely to spend} more time seeing nothing ($I=0$) than seeing other agents ($I=1$) in each trial.

\section{Physical experiments}\label{sec:physical_implementation}
In this section, we present a real-world validation of \textit{Turing Learning}. We explain how it can be used to infer the behavior of a swarm of real agents. The agents and replicas are represented by physical robots. We use the same type of robot (e-puck) as in simulation. The agents execute the aggregation behavior described in Section~\ref{sec:aggregation_behavior}. The replicas execute the candidate models. We use two replicas to speed up the \todo{identification} process, as will be explained in Section~\ref{sec:experimental_protocol}.

\subsection{Physical platform}\label{sec:experimental_setup}

\begin{figure}[!t]
    \centering
    \includegraphics[width=3.45in]{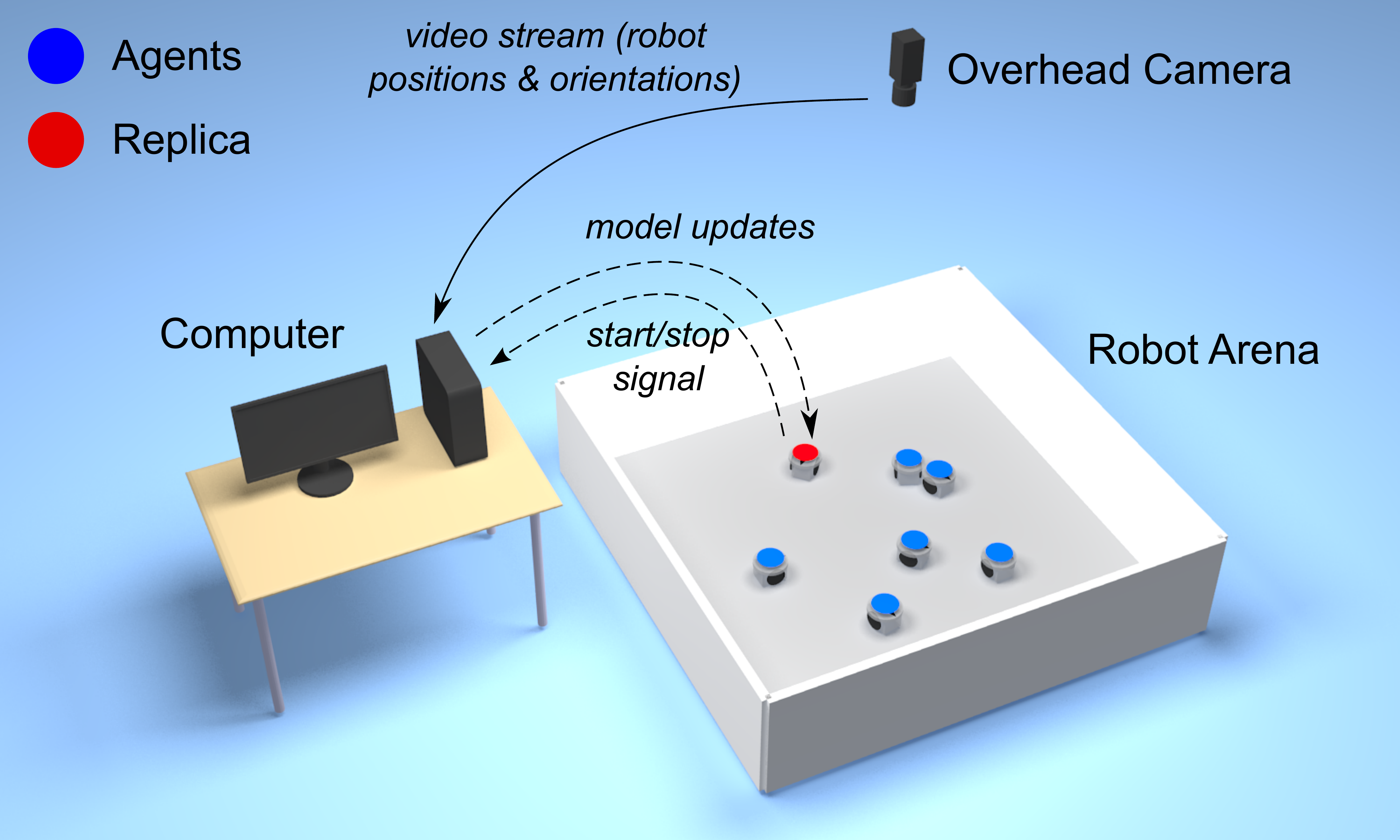}
    \caption{Illustration of the general setup for inferring the behavior of physical agents---e-puck robots (not to scale). The computer runs the \todo{\textit{Turing Learning}} algorithm, which produces models and classifiers. The models are uploaded and executed on the replica. The classifiers run on the computer. They are provided with the agents' and replica's motion data, extracted from the video stream of the overhead camera.}
    \label{fig:physical_system_setup}
\end{figure} 

The physical setup, shown in Fig.~\ref{fig:physical_system_setup}, consists of an arena with robots (representing agents or replicas), a personal computer (PC), and an overhead camera. The PC runs the \todo{\textit{Turing Learning}} algorithm. It communicates with the replicas, providing them models to be executed, but does not exert any control over the agents. The overhead camera supplies the PC with a video stream of the swarm. The PC performs video processing to obtain motion data about individual robots. We now describe the physical platform in more detail.
%\footnote{\todo{The evolution of the model population could in principle be conducted on the on-board micro-controller of the e-puck, but running it on the PC reduces experimental time~\citep{Floreano1996} and eases post-evaluation analysis.}}

\subsubsection{Robot arena}

The robot arena is rectangular with sides $\unit[200]{cm} \times \unit[225]{cm}$, and bounded by walls $\unit[50]{cm}$ high. The floor has a light gray color, and the walls are painted white.
\subsubsection{Robot platform and sensor implementations}\label{sec:robot_platform_sensor_implementation}
%In Section~\ref{sec:simulation_platform}, we presented the e-puck's shape and dimensions as a basis for the agents' embodiment in simulation. We now present further details about the e-puck relevant to our physical implementation. 

\begin{figure}[!t]
	\centering
	\includegraphics[width=3.0in]{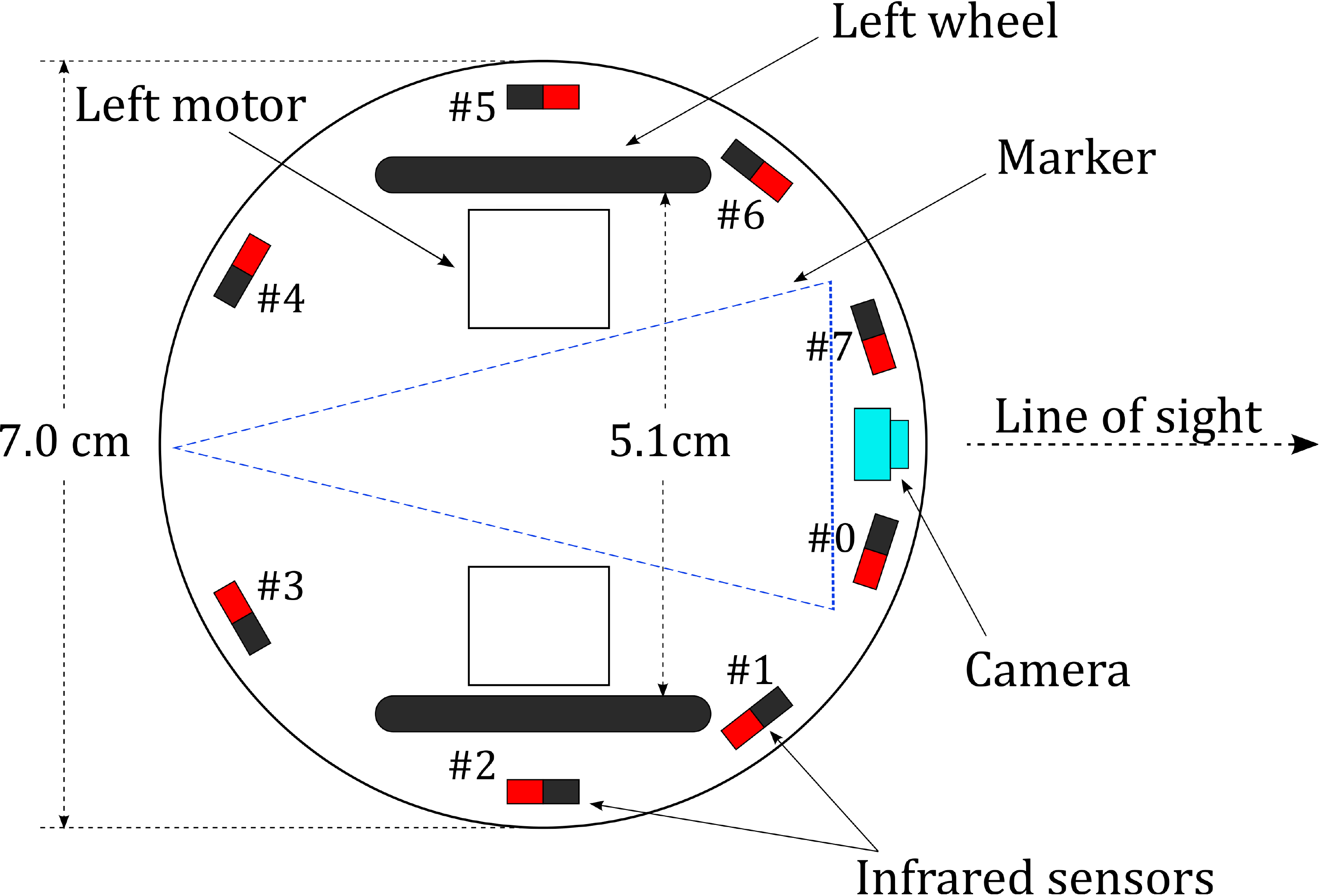}
	\caption{Schematic top view of an e-puck, indicating the locations of its motors, wheels, camera and infrared sensors. Note that the marker is pointing towards the robot's back.}
	\label{fig:e_puck_schematic}
\end{figure}

A schematic top view of the e-puck is shown in Fig.~\ref{fig:e_puck_schematic}. We implement the line-of-sight sensor using the e-puck's directional camera, located at its front. For this purpose, we wrap the robots in black `skirts' (see Fig.~\ref{fig:e-puck_body}) to make them distinguishable against the light-colored arena. While in principle the sensor could be implemented using one pixel, we use a column of pixels from a subsampled image to compensate for misalignment in the camera's vertical orientation. The gray values from these pixels are used to distinguish robots ($I=1$) against the arena ($I=0$). For more details about this sensor realization, see~\citep{Gauci2014_ijrr}.

We also use the e-puck's infrared sensors, in two cases. Firstly, before each trial, the robots disperse themselves within the arena. In this case, they use the infrared sensors to avoid both robots and walls, making the dispersion process more efficient. Secondly, we observe that using only the line-of-sight sensor can lead to robots becoming stuck against the walls of the arena, hindering the identification process. We therefore use the infrared sensors for wall avoidance, but in such a way as to not affect inter-robot interactions\footnote{To do so, the e-pucks determine whether a perceived object is a wall or another robot.}. Details of these two collision avoidance behaviors are provided in the online supplementary materials~\citep{online_supplementary_material_tevc2014}.

% However, we use the camera in monochrome mode, and sub-sample the image to $40 \times 15$ pixels, due to the e-puck's limited memory (which cannot even store a single full-resolution image). While in principle the sensor could be implemented using one pixel, we use the entire middle column of the sub-sampled image, to compensate for vertical misalignments in the camera~\citep{Gauci2014_ijrr}. If any pixel perceives black (according to a certain threshold), the sensor reads $I=1$; otherwise it reads $I=0$.

% There are eight infrared proximity sensors around the body of the robot. These are only used for collision/wall avoidance in the physical coevolutions where the environment is bounded. In the e-puck, the line-of-sight sensor is implemented using the middle column of the pixels from the camera to check if any pixel of that column exceeds a certain threshold in its gray scale~\citep{Gauci2014_ijrr}. %corresponding to the horizontal and vertical angle view of $56\degree$ and $48\degree$
%

\subsubsection{Motion capture}
To facilitate motion data extraction, we fit robots with markers on their tops, consisting of a colored isosceles triangle on a circular white background (see Fig.~\ref{fig:e-puck_body}). The triangle's color allows for distinction between robots; we use blue triangles for all agents, and orange and purple triangles for the two replicas. The triangle's shape eases extraction of robots' orientations.

% TODO: Wei, the ceiling was 270cm height. Would 300 cm (the value we had there before) then be too high? 
The robots' motion is captured using a camera mounted around $\unit[270]{cm}$ above the arena floor. The camera's frame rate is set to $\unit[10]{fps}$. The video stream is fed to the PC, which performs video processing to extract motion data about individual robots (position and orientation). The video processing software is written using OpenCV~\citep{Gary2008}.

\subsection{\todo{\textit{Turing Learning}} with physical robots}\label{sec:coevolutionary_process}
\finaltodo{Our objective is to} infer the agent's aggregation behavior. \finaltodo{We do not wish to infer} the agent's dispersion behavior, which is periodically executed to distribute already-aggregated agents. \finaltodo{To separate these two behaviors, the}
 robots (agents and replicas) \finaltodo{and the system} are
implicitly synchronized. This is realized by making each robot execute a fixed behavioral loop of constant duration. The PC also executes a fixed behavioral loop, but the timing is determined by the signals received from the replicas. Therefore, the PC is synchronized with the swarm. The PC communicates with the replicas via Bluetooth. At the start of a run, or after a human intervention (see Section~\ref{sec:experimental_protocol}), robots are initially synchronized using an infrared signal from a remote control.
%This enabled the PC to issue the replicas with models to be executed. The agents did not perform any communication. 

\begin{figure}[!t]
    \centering
    \includegraphics[width=3.0in]{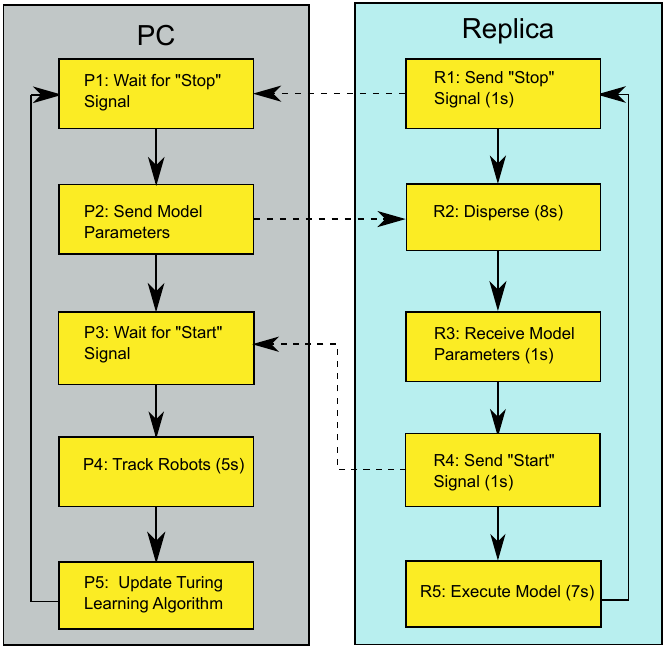}
    \caption{Flow diagram of the programs run by the PC and a replica in the physical experiments. Dashed arrows represent communication between the two units. See Section~\ref{sec:coevolutionary_process} for details. \todo{The PC does not exert any control over the agents.}}
    \label{fig:agent_pc_interation}
\end{figure} 

Fig.~\ref{fig:agent_pc_interation} shows a flow diagram of the programs run by the PC and the replicas, respectively. Dashed arrows indicate communication between the units. 
%The agents executed a similar behavioral loop to the replicas. In the following, we detail the states of the programs executed by the PC, replicas, and agents.

\todo{The program running on the PC has the following states:}
\begin{itemize}
\renewcommand{\labelitemi}{\scriptsize$\bullet$} 
	\item \textit{P1.} \textit{Wait for ``Stop'' Signal.} The program is paused until ``Stop'' signals are received from both replicas. These signals indicate that a trial has finished.
	
	\item \textit{P2.} \textit{Send Model Parameters.} The PC sends new model parameters to the buffer of each replica. % the replicas are still in state \textit{R2}; they are later read by the replicas in state \textit{R3}
	
	\item \textit{P3.} \textit{Wait for ``Start'' Signal.} The program is paused until ``Start'' signals are received from both replicas, indicating that a trial is starting.
	
	\item \textit{P4.} \textit{Track Robots.} The PC waits $\unit[1]{s}$ and then tracks the robots using the overhead camera for $\unit[5]{s}$. The tracking data contain the positions and orientations of the agents and replicas. 
	%During the trial, which lasts $\unit[7]{s}$, the PC tracks the robots using the overhead camera. Potentially misaligned data from the initial and final seconds of this period are discarded. The motion data thus contains the positions and orientations of the agents and replicas over $\unit[5]{s}$. 
	
% TODO: in later versions - consider to replace coevolutionary by optimization in both text and the diagram
	\item \textit{P5.} \textit{Update Turing Learning Algorithm.} The PC uses the motion data from the trial observed in \textit{P4} to update the solution quality (fitness values) of the corresponding two models and all classifiers. Once all models in the current iteration cycle (generation) have been evaluated, the PC also generates new model and classifier populations. The method for calculating the qualities of solutions and the optimization algorithm are described in \todo{Sections~\ref{sec:turing_learning} and~\ref{sec:optimization_algorithm}} respectively. The PC then goes back to~\textit{P1}. 
\end{itemize}

\todo{The program running on the replicas has the following states:}
\begin{itemize}
\renewcommand{\labelitemi}{\scriptsize$\bullet$} 
\item \textit{R1}. \textit{Send ``Stop'' Signal.} After a trial stops, the replica informs the PC by sending a ``Stop'' signal. The replica waits $\unit[1]{s}$ before proceeding with~\textit{R2}, so that all robots remain synchronized. Waiting $\unit[1]{s}$ in other states serves the same purpose. %(agents are programmed to restart $\unit[1]{s}$ after a trial stops)

\item\textit{R2}. \textit{Disperse.} The replica disperses in the environment, while avoiding collisions with other robots and the walls. This behavior lasts $\unit[8]{s}$.

\item\textit{R3}. \textit{Receive Model Parameters.} The replica reads new model parameters from its buffer (sent earlier by the PC). It waits $\unit[1]{s}$ before proceeding with~\textit{R4}.

\item\textit{R4}. \textit{Send ``Start'' Signal.} The replica sends a start signal to the PC to inform it that a trial is about to start. The replica waits $\unit[1]{s}$ before proceeding with~\textit{R5}.

\item\textit{R5}. \textit{Execute Model.} The replica moves within the swarm according to its model. This behavior lasts $\unit[7]{s}$ (the tracking data corresponds to the middle $\unit[5]{s}$, see~\textit{P4}). The replica then goes back to~\textit{R1}.
\end{itemize}

\todo{The program running on the agents has the same structure as the replica program.} However, in the states analogous to \textit{R1}, \textit{R3}, and \textit{R4}, they simply wait $\unit[1]{s}$ rather than communicate with the PC. In the state corresponding to \textit{R2}, they also execute the \textit{Disperse} behavior. In the state corresponding to \textit{R5}, they execute the agent's aggregation controller, rather than a model.

Each iteration (loop) of the program for the PC, replicas and agents lasts $\unit[18]{s}$. 

\subsection{Experimental setup}\label{sec:experimental_protocol}
As in simulation, we use a population size of $100$ for classifiers ($\mu = 50$,  $\lambda = 50$). However, the model population size is reduced from $100$ to $20$ ($\mu = 10$,  $\lambda = 10$), to shorten the experimentation time. We use $10$ robots: $8$ representing agents executing the original aggregation controller (Eq.~\eqref{eq:aggregation_optimal_controller}), and $2$ representing replicas that execute models. This means that in each trial, $2$ models from the population could be simultaneously evaluated; consequently, each generation consists of $20/2=10$ trials. 
%Note from the previous section that each trial lasted $\unit[18]{s}$; each generation therefore lasted $10\times 18 = \unit[180]{s}$. We chose to run coevolutions for $100$ generations, for a total time of around $5$ hours per run (excepting human interventions).

The \todo{\textit{Turing Learning}} algorithm is implemented without any modification to the code used in simulation (except for model population size and observation time in each trial). We still let the model parameters evolve unboundedly (i.e., in $\mathbb{R}^4$). However, as the speed of the physical robots is naturally bounded, we apply the hyperbolic tangent function ($\tanh{x}$) on each model parameter, before sending a model to a replica. This bounds the parameters to $\left(-1,1\right)^4$, with $-1$ and $1$ representing the maximum backwards and forwards wheel speeds, respectively.

The \todo{\textit{Turing Learning}} runs proceed autonomously. In the following cases, however, there is  intervention:
\begin{itemize}
\item The robots have been running continuously for $25$ generations. All batteries are replaced.

\item Hardware failure has occurred on a robot, for example because of a lost battery connection or because the robot has become stuck on the floor. Appropriate action is taken for the affected robot to restore its functionality.

\item A replica has lost its Bluetooth connection with the PC. The connection with both replicas is restarted.

\item A robot indicates a low battery status through its LED after running for only a short time. That robot's battery is changed.
\end{itemize}

After an intervention, the ongoing generation is restarted, to limit the impact on the \textit{identification} process.

We conducted $10$ runs of \todo{\textit{Turing Learning}} using the physical system. Each run lasted $100$ generations, corresponding to $5$ hours (excluding human intervention time). Video recordings of all runs can be found in the online supplementary materials~\citep{online_supplementary_material_tevc2014}.

%\subsection{\textcolor{black}{Experimental} Results}\label{sec:experimental_results}

\subsection{Analysis of \todo{inferred} models}

\begin{figure}[!t]
    \centering
    \includegraphics[width=3.0in]{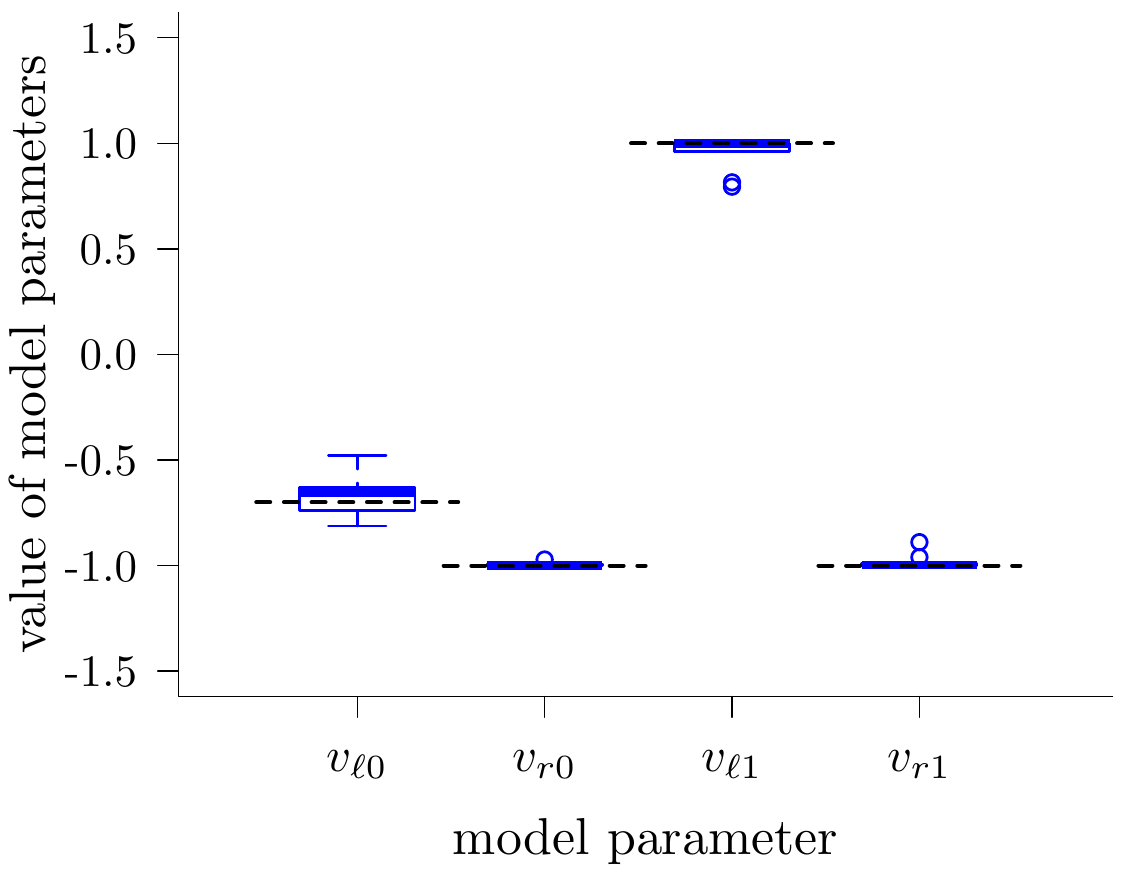}
    \caption{\todo{Model parameters \textit{Turing Learning} inferred from swarms of physical robots performing aggregation. The models are those with} the highest subjective fitness in the $100$th generation of $10$ runs. Dashed black lines indicate the ground truth\todo{, that is, the values of the parameters that the system is expected to learn.}}
    \label{fig:best_model_parameters_physical}
\end{figure}

We first investigate the quality of the models obtained. To select the `best' model from each run, we post-evaluated all models of the final generation $5$ times using all classifiers of that generation. The parameters of these models are shown in Fig.~\ref{fig:best_model_parameters_physical}. The means (standard deviations) of the AEs in each parameter are as follows: $0.08$ ($0.06$), $0.01$ ($0.01$), $0.05$ ($0.08$), and $0.02$ ($0.04$).

\begin{figure}[!t]%
	\centering
		\subfloat[(a) physical coevolutions\label{fig:model_parameters_convergence_compare_physical}]{%
			\includegraphics[width=2.25 in]{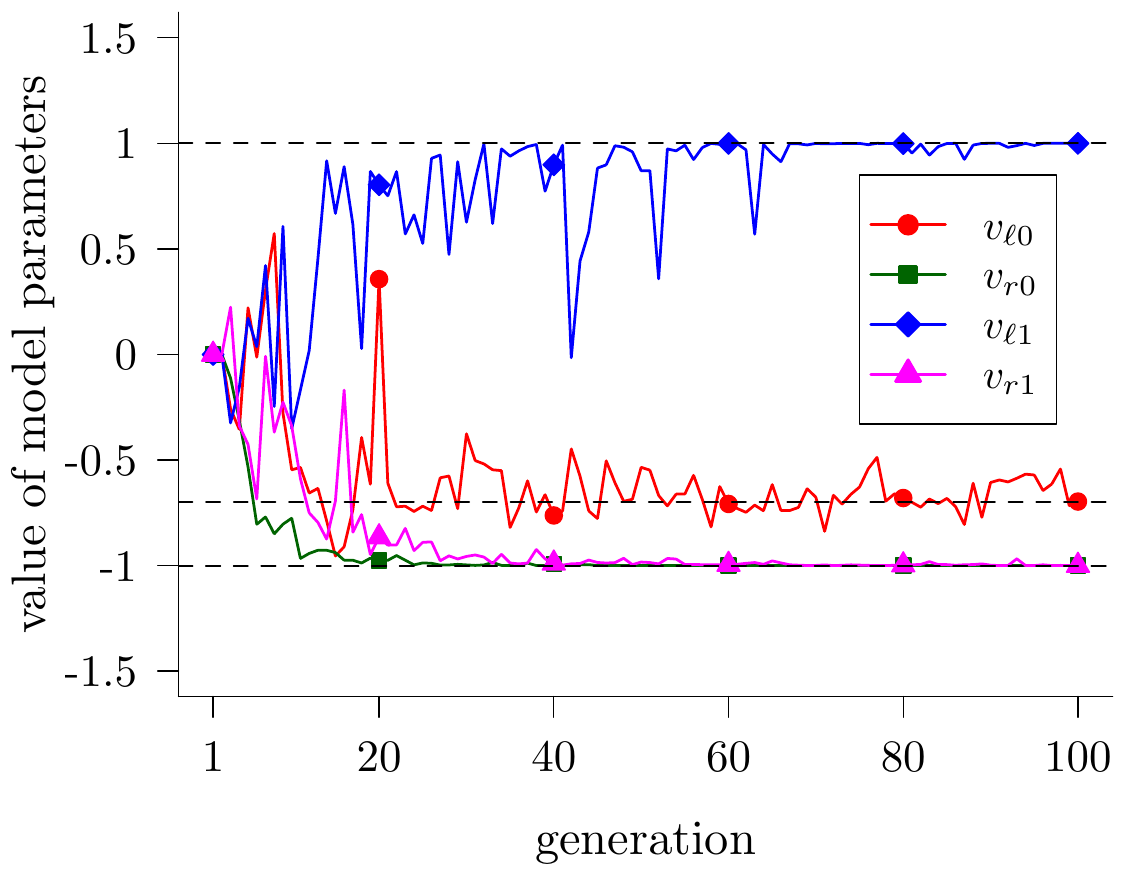} %grid_visualization
		}
		\subfloat[(b) simulated coevolutions\label{fig:model_parameters_convergence_compare_simulation}]{%
			\includegraphics[width=2.25 in]{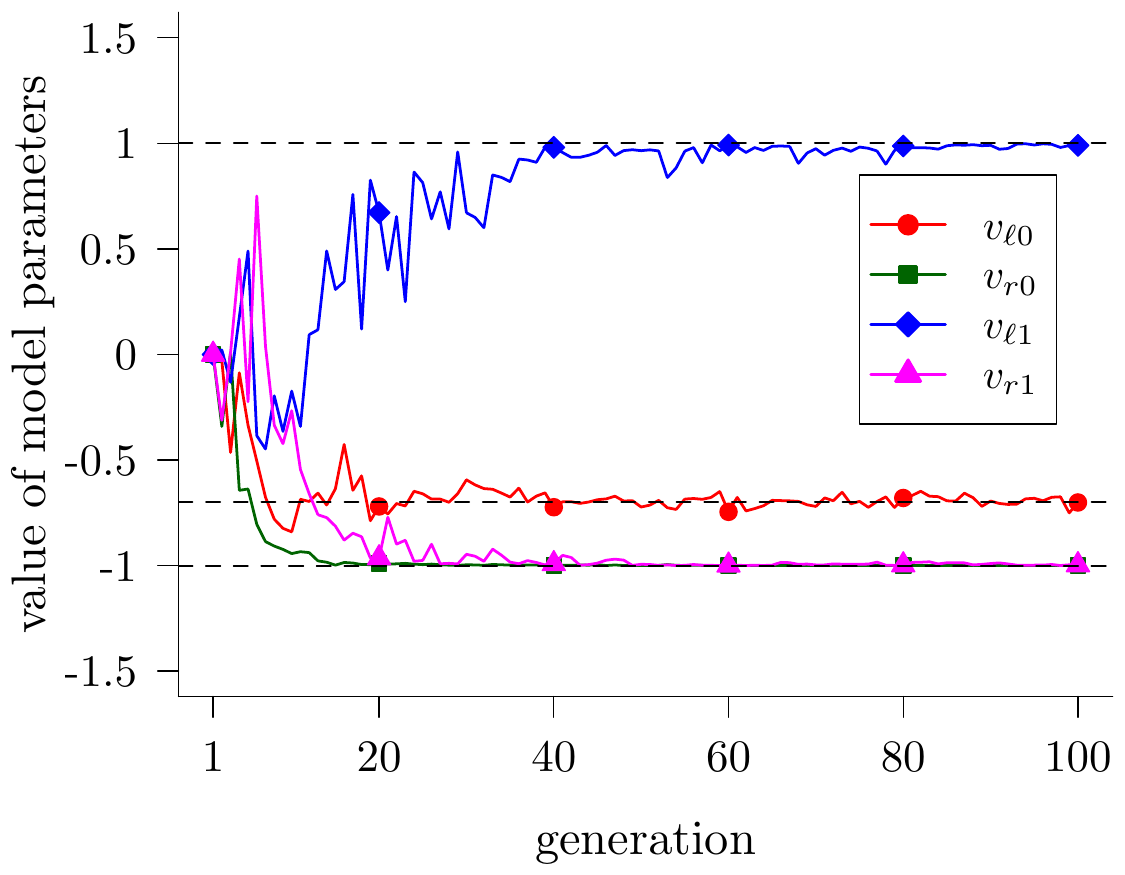}
		}
		\caption{Evolutionary dynamics of model parameters in (a) $10$ physical and (b) $10$ simulated runs of \textit{Turing Learning} (in both cases, equivalent setups were used). Curves represent median parameter values of the models with the highest subjective fitness across the $10$ runs. Dashed black lines indicate the ground truth.}
		\label{fig:model_parameters_convergence_compare}
\end{figure}

\begin{figure}[!t]%
	\centering
		\subfloat[(a) physical coevolutions\label{fig:MAE_simulation}]{%
			\includegraphics[width=2.3 in]{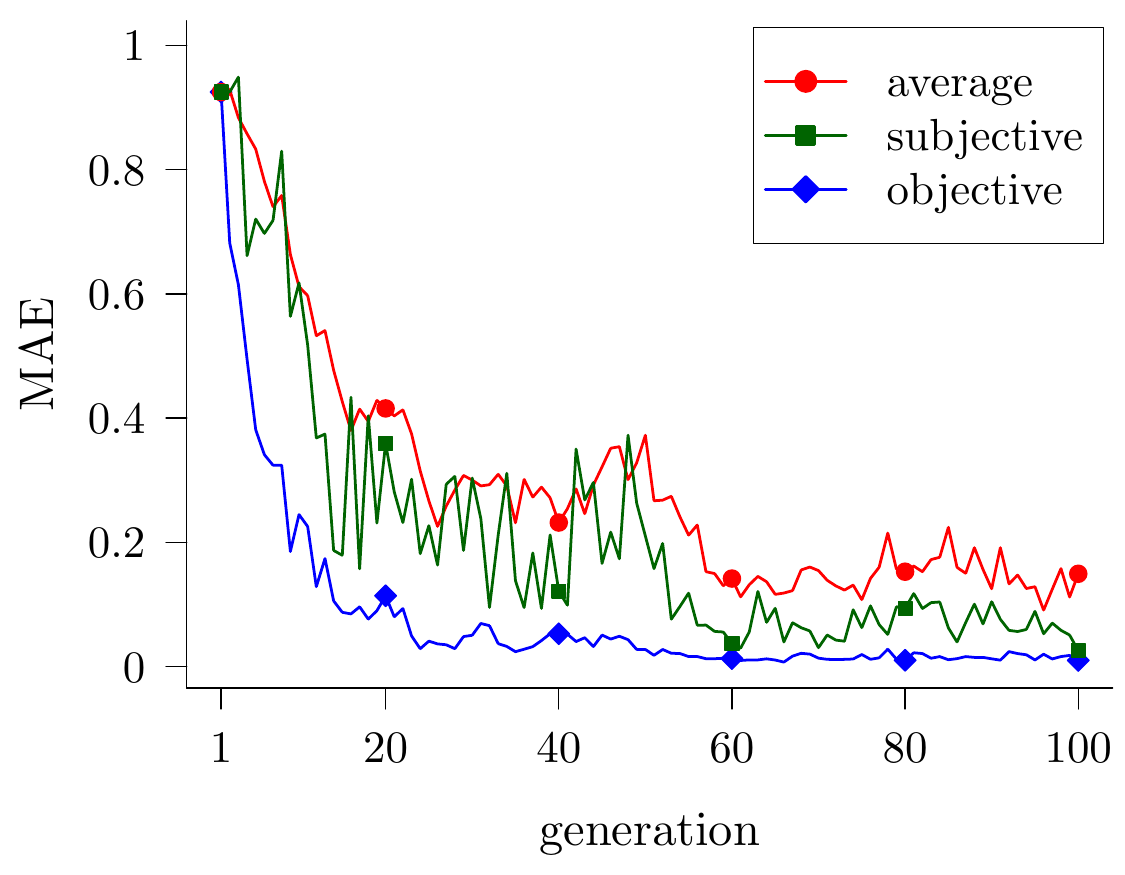} %grid_visualization
		}
		\subfloat[(b) simulated coevolutions\label{fig:MAE_physical}]{%
			\includegraphics[width=2.3 in]{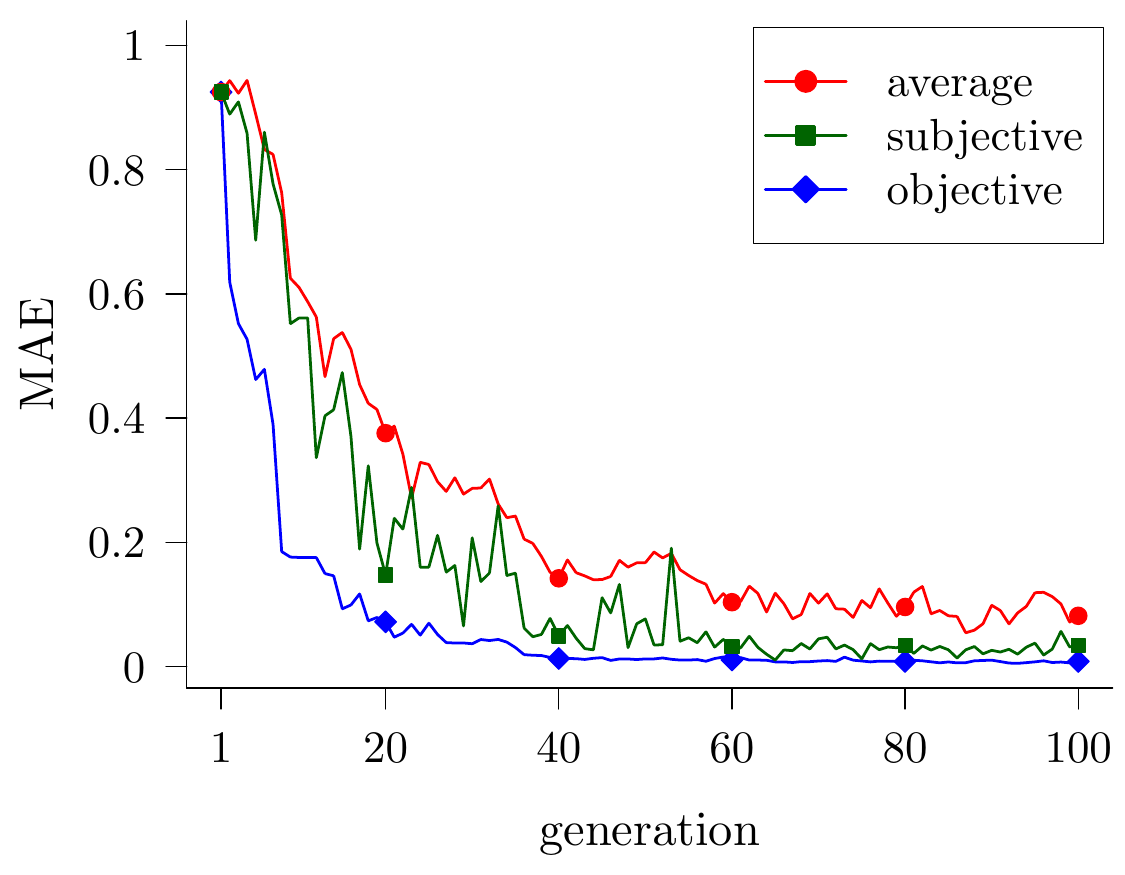}
		}
		\caption{Evolutionary dynamics of MAE (defined in Eq.~\eqref{eq:MAE}) for the candidate models in (a) $10$ physical and (b) $10$ simulated runs of \textit{Turing Learning}. Curves represent median values across $10$ runs. The red curve represents the average error of all models in a generation. The green and blue curves show, respectively, the errors of the models with the highest subjective and the highest objective fitness in a generation.}
		\label{fig:MAE_compare_simulation_physical}
\end{figure}

To investigate the effects of real-world conditions on the \todo{identification} process, we performed $10$ simulated runs \todo{of \textit{Turing Learning}} with the same setup as in the physical runs. Fig.~\ref{fig:model_parameters_convergence_compare} shows the evolutionary dynamics of the parameters of the \todo{inferred} models (with the highest subjective fitness) in the physical and simulated runs. The dynamics show good correspondence. However, the convergence in the physical runs is somewhat less smooth than that in the simulated ones (e.g., see spikes in $v_{\ell 0}$ and $v_{\ell 1}$). 
%One reason for this may be the limitations in motion data capture and extraction. In particular, given the relatively small diameter of the e-puck in our arena, inferring its orientation is particularly challenging.
In each generation of every run (physical and simulated), we computed the MAE of each model. We compared the error of the model with the highest subjective fitness with the average and lowest errors. The results are shown in Fig.~\ref{fig:MAE_compare_simulation_physical}. For both the physical and simulated runs, the subjectively best model (green) has an error in between the lowest error (blue) and the average error (red) in the majority of generations. 
%Also in both cases the gap between the error of the subjectively best model and the average error becomes wider as the coevolution proceeds, which means the decision-making ability of the classifiers is improving. This gap, in turn, forces the model population to evolve, as indicated by the downwards trend of the lowest error.
%and only misguided by increasingly good models.

As we argued before (Section~\ref{sec:analysis_evolved_models}), in swarm systems, good agreement between local behaviors (e.g., controller parameters) may not guarantee similar global behaviors. For this reason, we investigate both 
the original controller (Eq.~\eqref{eq:aggregation_optimal_controller}), and 
a controller obtained from the physical runs. This latter controller is constructed by taking the median values of the parameters over the $10$ runs, which are:
$$
\mathbf{p}=\left(-0.65, -0.99, 0.99, -0.99\right).
$$
The set of initial configurations of the robots is common to both controllers. As it is not necessary to extract the orientation of the robots, a red circular marker is attached to each robot so that its position can be extracted with higher accuracy in the offline analysis~\citep{Gauci2014_ijrr}.
\begin{figure}[!t]%
	\centering
		\subfloat[(a) Largest Cluster Dynamics \label{fig:aggregation_dynamics_proportion}]{
			\includegraphics[width=2.3 in]{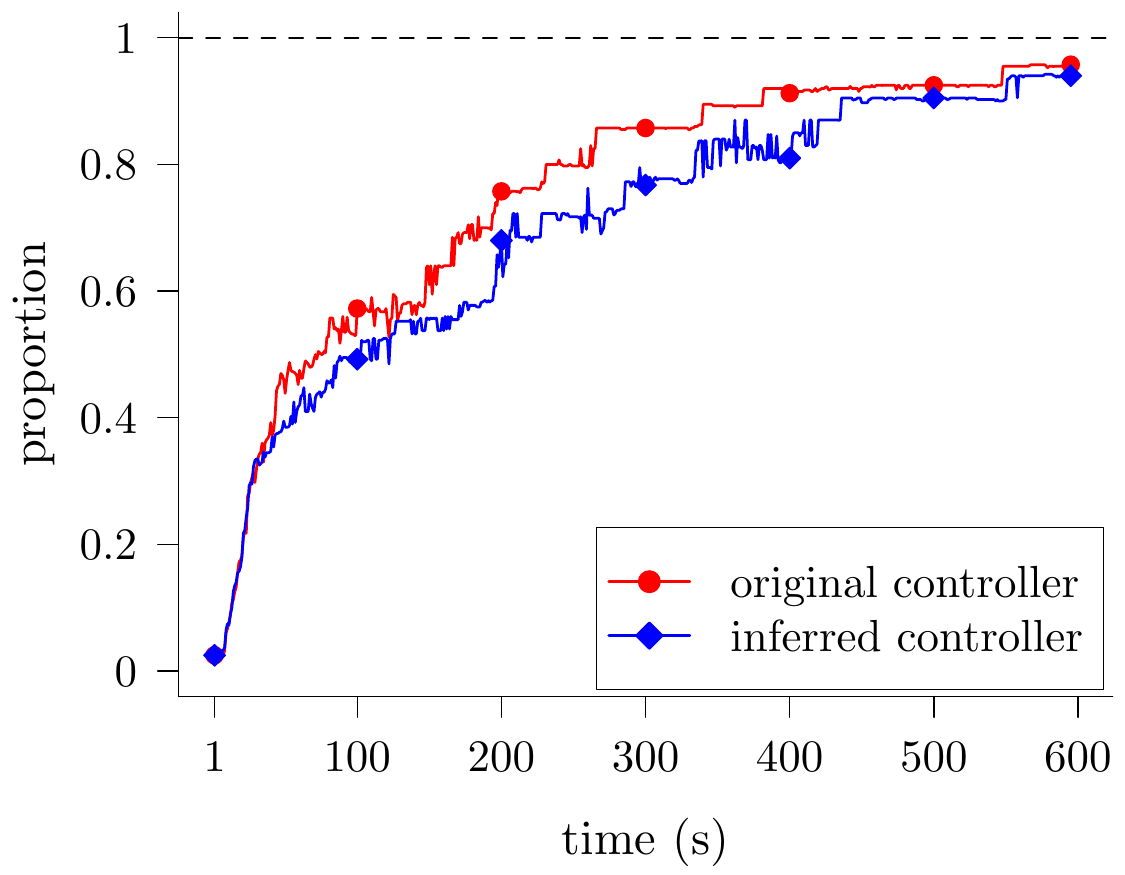}
		}
		\subfloat[(b) Dispersion Dynamics \label{fig:aggregation_dynamics_compactness}]{
			\includegraphics[width=2.3 in]{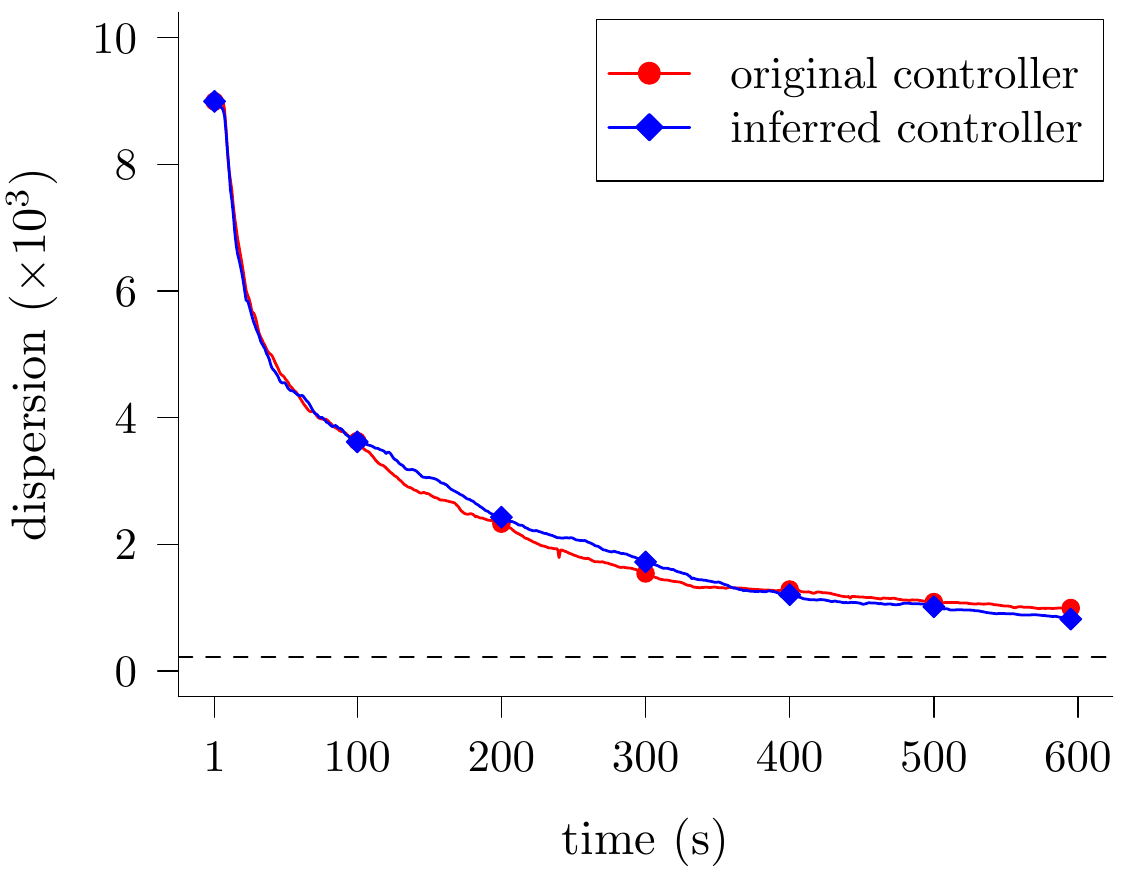}
		}
		\caption{Average aggregation dynamics in $10$ physical trials with $40$ physical e-puck robots executing the original agent controller (red) and the model controller (blue) \todo{inferred through observation of the physical system}. In (a), the vertical axis shows the proportion of robots in the largest cluster; in (b), it shows the robots' dispersion (see Section~\ref{sec:analysis_evolved_models}). Dashed lines in (a) and (b) respectively represent the maximum proportion and minimum dispersion that $40$ robots can achieve.}
		\label{fig:aggregation_dynamics_physical}
\end{figure}
\captionsetup[subfigure]{labelformat=empty}  
\begin{figure*}[!t]
	\centering
	\subfloat[initial configuration]{
		\includegraphics[width = 1.1 in]{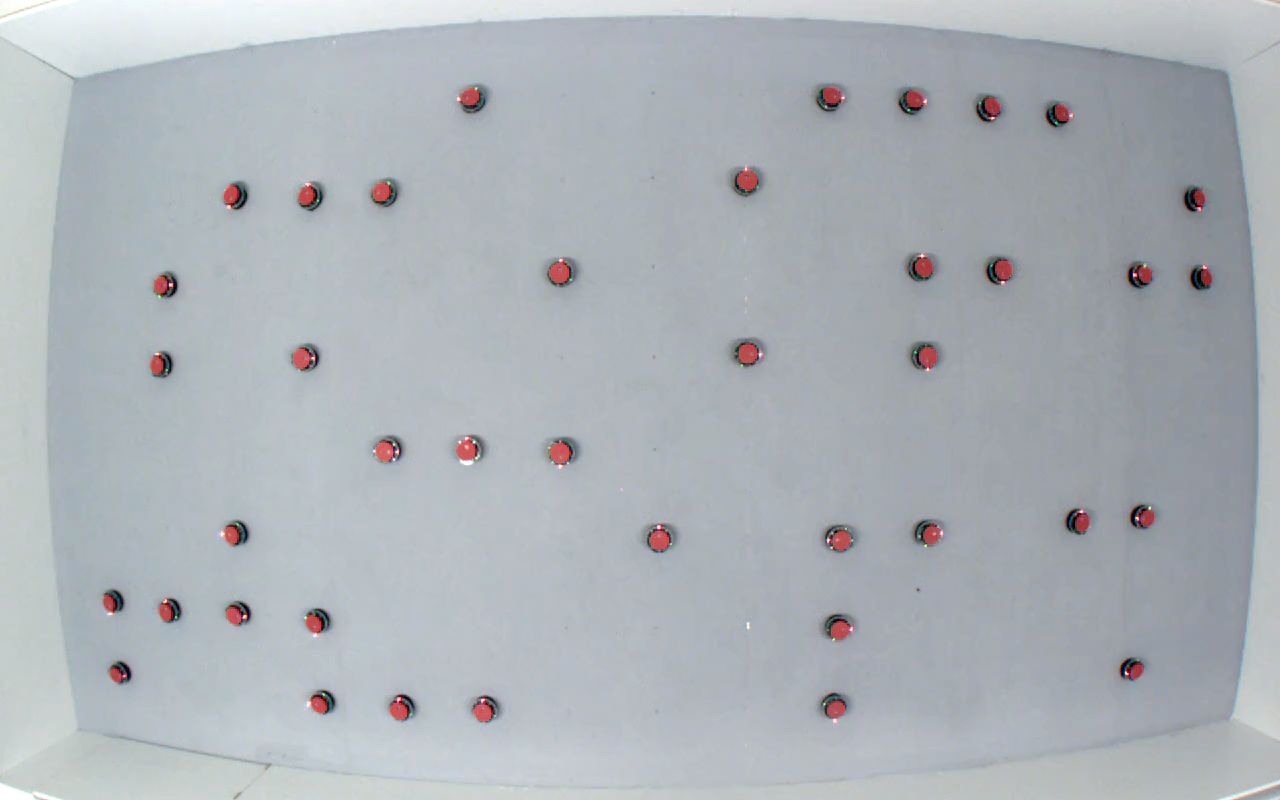}
	}
	\subfloat[after $20$ $\unit{s}$]{
		\includegraphics[width = 1.1 in]{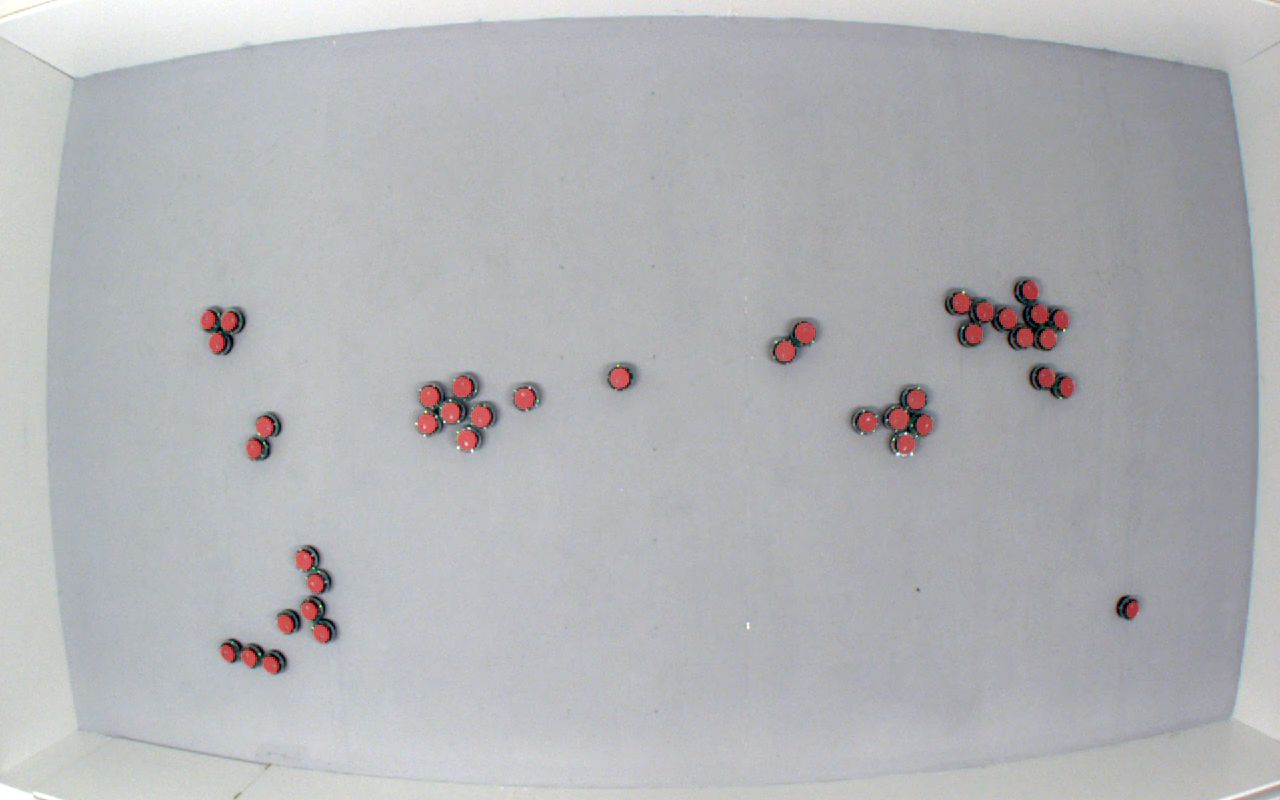}
	}
	\subfloat[after $40$ $\unit{s}$]{
		\includegraphics[width = 1.1 in]{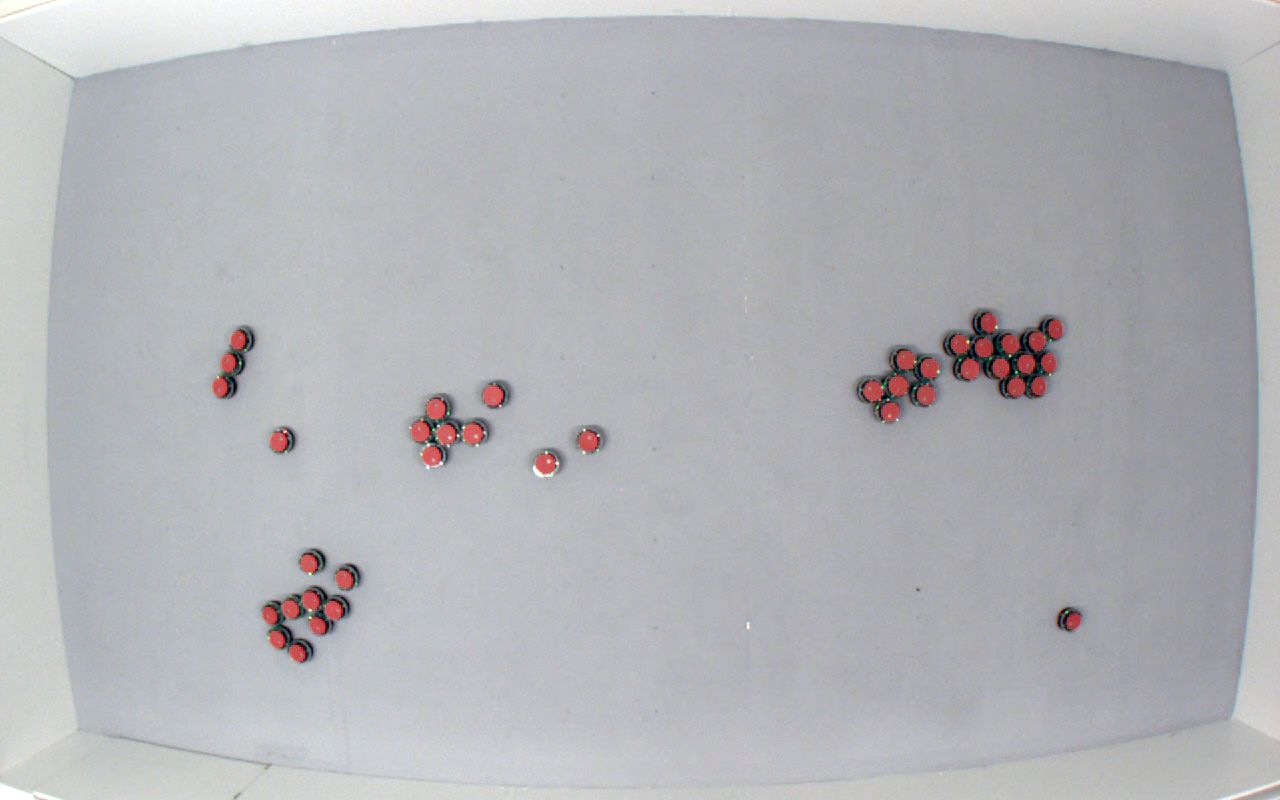}
	}
	\subfloat[after $180$ $\unit{s}$]{
		\includegraphics[width = 1.1 in]{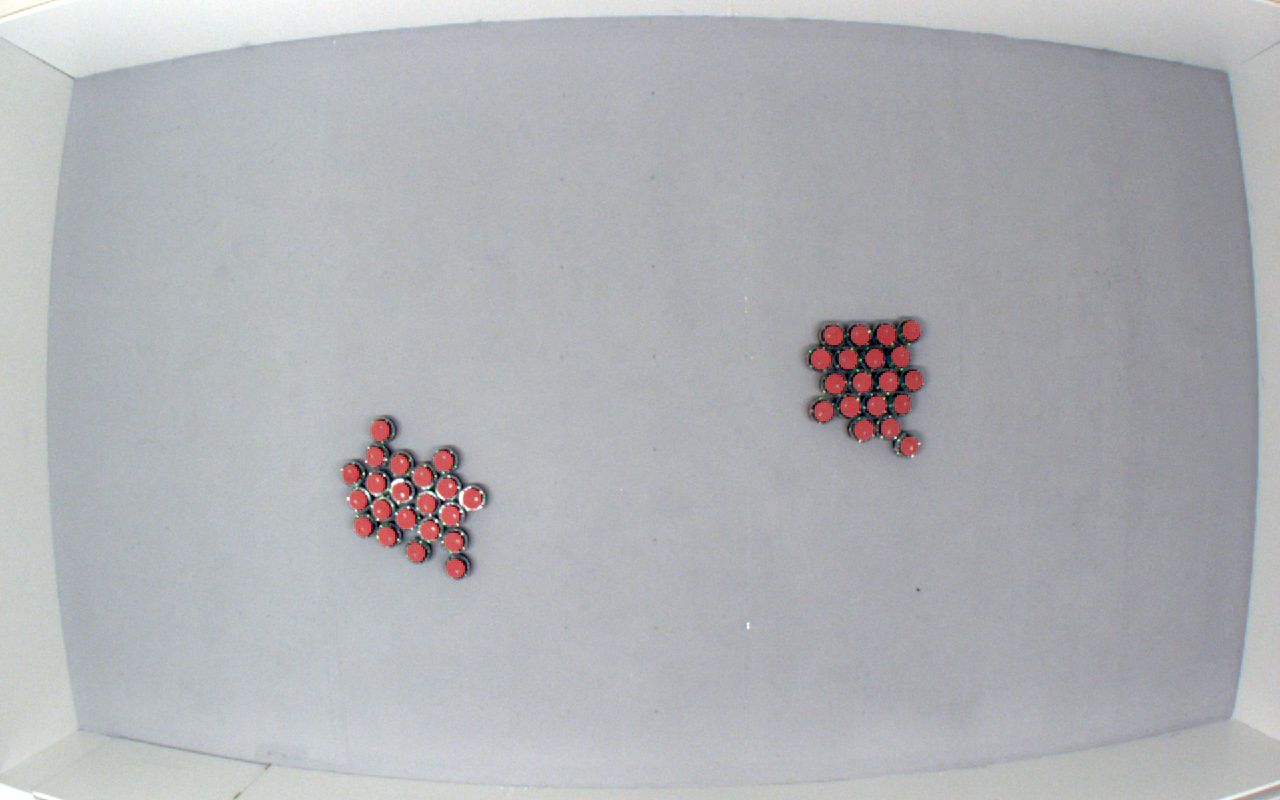}
	}\\
		\subfloat[after $360$ $\unit{s}$]{
		\includegraphics[width = 1.1 in]{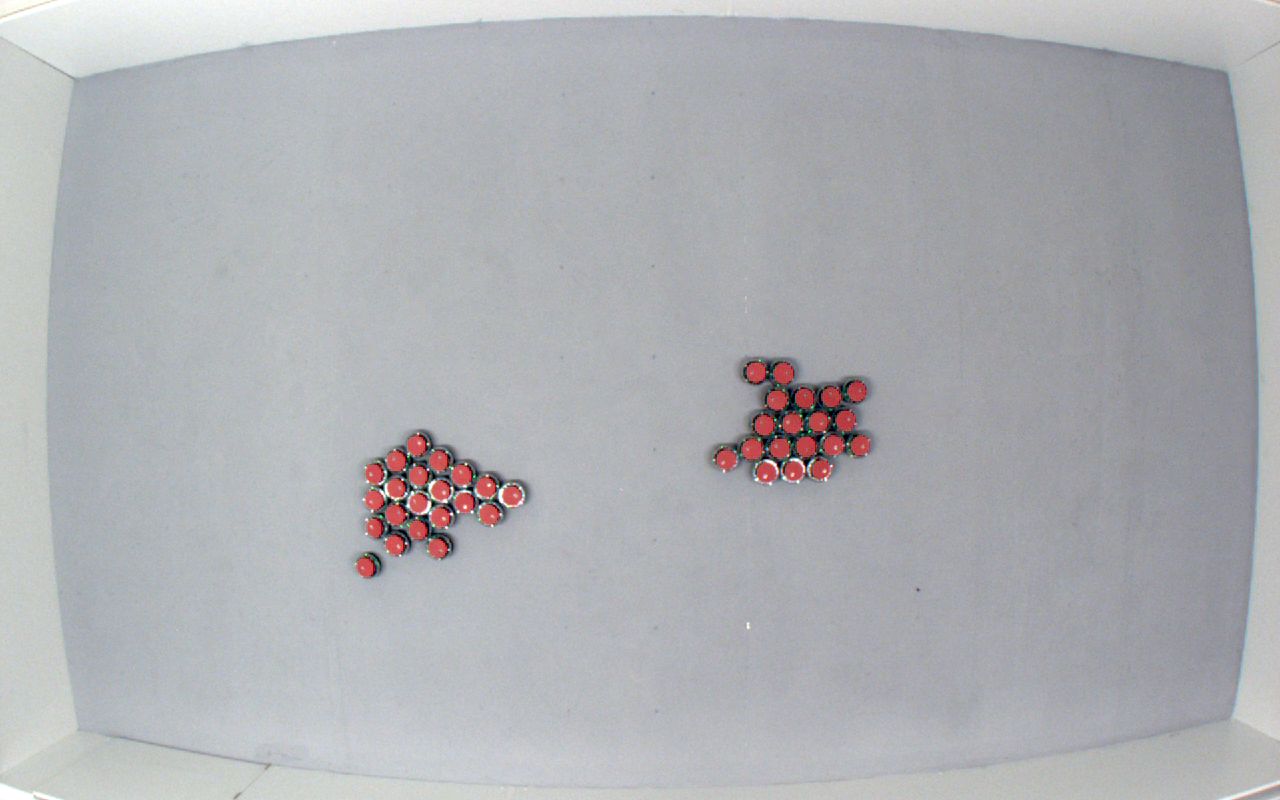}
	}
	\subfloat[after $420$ $\unit{s}$]{
		\includegraphics[width = 1.1 in]{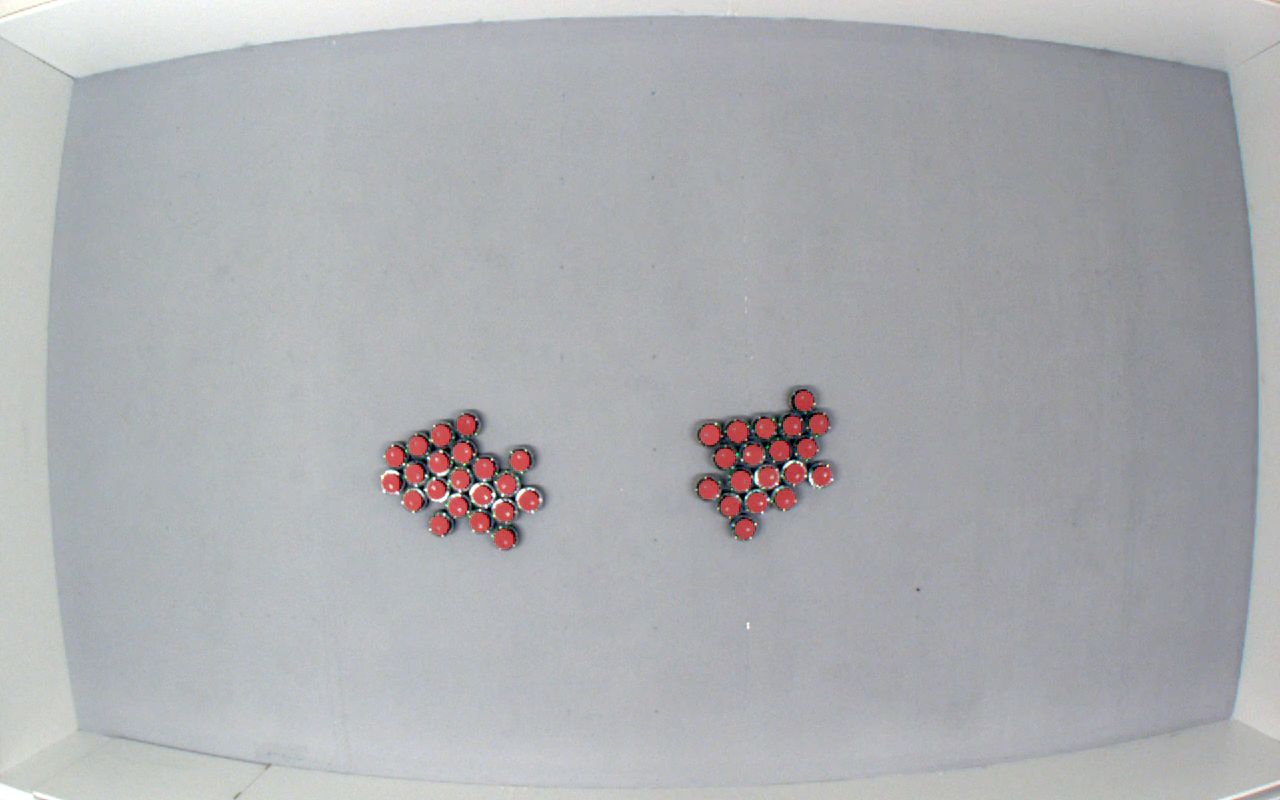}
	}
	\subfloat[after $480$ $\unit{s}$]{
		\includegraphics[width = 1.1 in]{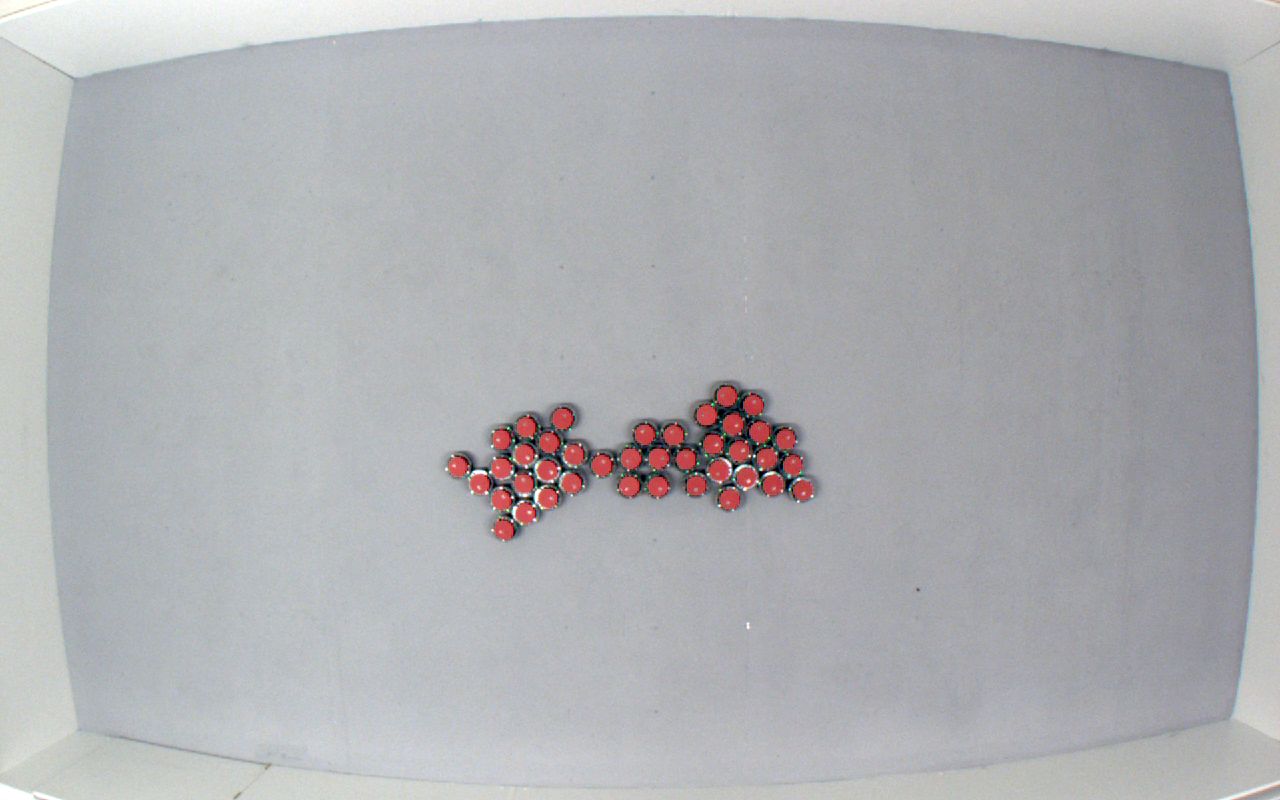}
	}
	\subfloat[after $600$ $\unit{s}$]{
		\includegraphics[width = 1.1 in]{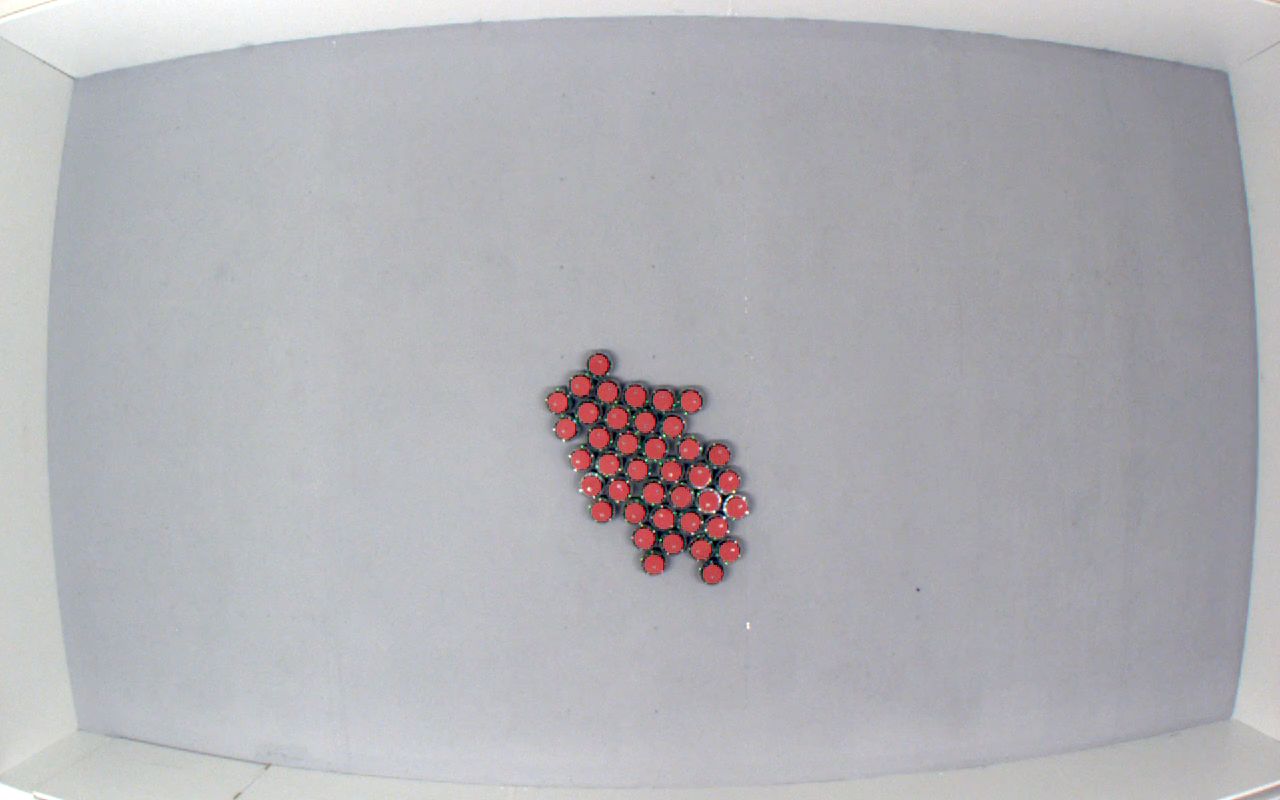}
	}
	\caption{Example of \todo{collective behavior produced by a model that was inferred} by \textit{Turing Learning} \todo{through the observation of swarms of physical e-puck robots}. A swarm of $40$ \finaltodo{physical} e-puck robots, each executing the inferred model, aggregates in a single spot.}%The behavior was automatically learned through observation of swarms of physical robots.
	\label{fig:aggregation_snapshoot_physical_validation}
\end{figure*}

For each controllers, we performed $10$ trials using $40$ physical e-pucks. Each trial lasted $10$ minutes.
%: $10$ trials with 
%$10$ trials with 
Fig.~\subref*{fig:aggregation_dynamics_proportion} shows the proportion of robots in the largest cluster\footnote{A cluster of robots is defined as a maximal connected subgraph of the graph defined by the robots' positions, where two robots are considered to be adjacent if another robot cannot fit between them~\citep{Gauci2014_ijrr}.} over time with the agent and model controllers. Fig.~\subref*{fig:aggregation_dynamics_compactness} shows the dispersion (as defined in Section~\ref{sec:analysis_evolved_models}) of the robots over time with the two controllers. The aggregation dynamics of the \todo{agents and the models} show good correspondence. Fig.~\ref{fig:aggregation_snapshoot_physical_validation} shows a sequence of snapshots from a trial with $40$ e-pucks executing the \todo{inferred model} controller.

A video accompanying this paper shows the \todo{\textit{Turing Learning} identification} process of the models (in a particular run) both in simulation and on the physical system. Additionally, videos of all $20$ post-evaluation trials with $40$ e-pucks, \textcolor{black}{are provided in the online supplementary materials}~\citep{online_supplementary_material_tevc2014}.

\subsection{Analysis of \todo{generated} classifiers}
%\todo{We will now investigate the classifiers' performance. First, we investigate whether the the classifiers' ability of making judgment improves over generations. This could be reflected by comparing the difference between the model with the highest subjective (selected by the classifiers) and the model with the highest objective fitness (the one with the least square error).} As we discussed in Section~\ref{sec:analysis_of_evolved_classifiers_simulation}, the classifier groups obtained in simulation have a good decision accuracy.

\begin{figure}[!t]
\centering
\includegraphics[width=2.8in]{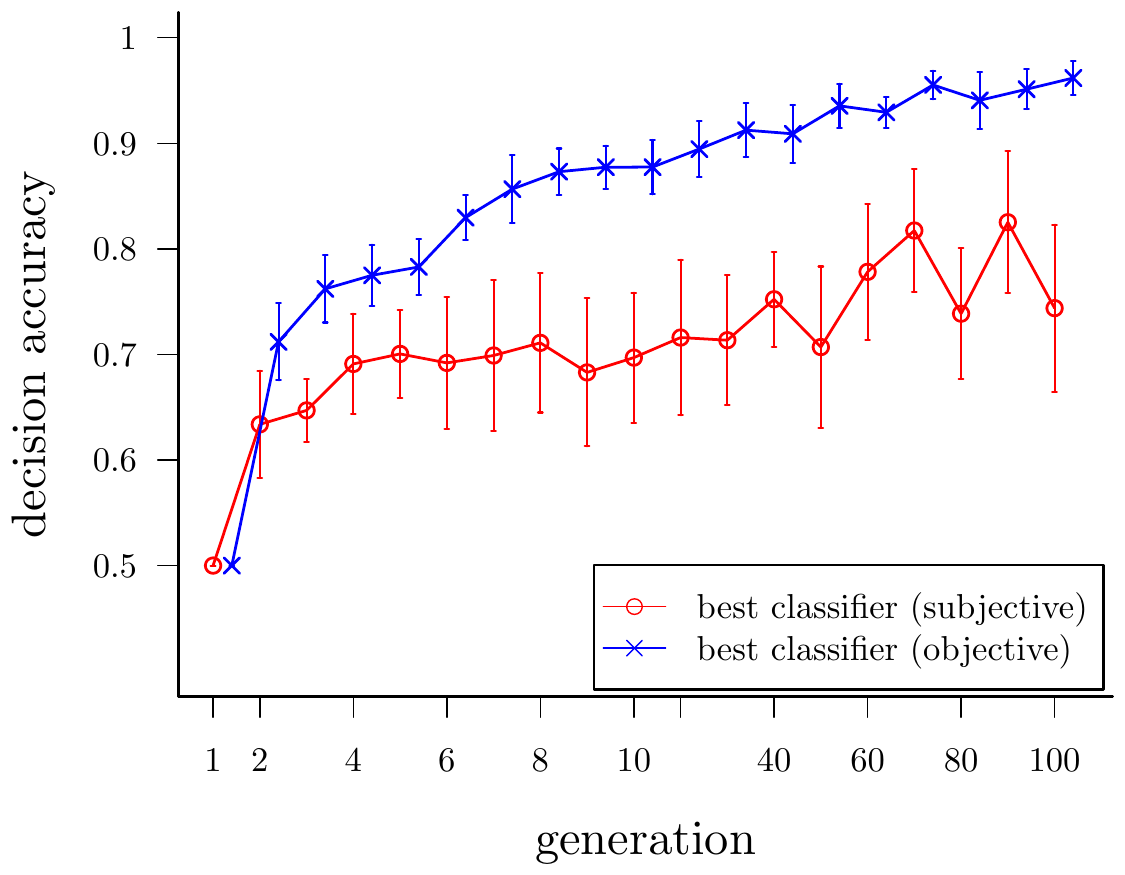}
\caption{Average decision accuracy of the best classifiers over $100$ generations (nonlinear scale) in $10$ runs of \textit{Turing Learning} with swarms of physical robots. Average over 100 data samples from a post-evaluation with physical robots executing random models. The error bars show standard deviations. See text for details.}
% TODO: Wei - please replace "decision" with "objective" in the figure legend?
\label{fig:classifier_decision_accuracy_physical}
\end{figure}

When post-evaluating the classifiers \todo{generated} in the physical \todo{runs of \textit{Turing Learning}}, we limited the number of candidate models to $100$, in order to reduce the \finaltodo{physical experimentation} time. Each candidate model was chosen randomly, \textcolor{black}{with uniform distribution, 
from the parameter space defined in Eq.~\eqref{controller:form}}. Fig.~\ref{fig:classifier_decision_accuracy_physical} shows the average decision accuracy of the best classifiers over the $10$ runs. Similar to the results in simulation, \textcolor{black}{the \textit{best classifier (objective)} still has a high decision accuracy.} However, in contrast to simulation, the decision accuracy of the \textit{best classifier (subjective)} does not drop within $100$ generations. This could be due to the \todo{noise present} in the physical runs, which may have prevented the classifiers from getting over-specialized in \todo{the comparatively short time provided (100 generations)}.
%Similar to the results in simulation, the classifier groups obtained in the physical runs still have a high decision accuracy. The \textit{classifier group (archive)} has a higher accuracy than that of the \textit{classifier group (subjective)} and \textit{best classifier (objective)} 
% TODO: becoming instead of getting?
% This could be due to the smaller model population size in the physical coevolution runs
%To see if we can still get an accurate classifier group in physical experiments, we did the same evaluation. However, as performing a grid search took so much time, we only evaluated the classifiers using a limited number of candidates. 
% Fig.~\subref*{fig:model_parameters_convergence_compare_physical}, the error of the models with the highest objective fitness is becoming smaller until about $60^\textrm{th}$ generation, and then it increases slightly. We suspect that this is due to the fact that when the replica behaves more like the agents (that is, the model parameters converge to their real value), it is more likely to collide with them (i.e., to aggregate). the chance of its collision with other agents is higher than those replicas that executes `worse' models.
%with uniform distribution from $[-1,1]^4$. 

\section{Conclusions}\label{sec:conclusion}

This paper presented a new system identification method---\textit{Turing Learning}---that can autonomously infer the behavior of a system from observations. To the best of our knowledge, \textit{Turing Learning} is the first system identification method that does not rely on any predefined metric to quantitatively gauge the difference between the system and the inferred models. This eliminates the need to choose a suitable metric and the bias that such metric may have on the identification process.

Through competitive and successive generation of models and classifiers, the system successfully learned two behaviors: self-organized aggregation and object clustering in swarms of mobile agents. Both the model parameters, which were automatically inferred, and emergent global behaviors closely mat\-ched those of the original swarm system. 

We also examined a conventional system identification method, which used a least-square error metric rather than classifiers. We proved that the metric-based method was fundamentally flawed for the case studies considered here. In particular, as the inputs to the agents and to the models were not correlated, the model solution that was globally optimal with respect to the metric was not identical to the agent solution. In other words, according to the metric, the parameter set of the agent itself scored worse than a different---and hence incorrect---parameter set.
%re parameter set of the original agent is considered by th agent itself is considered less good than a wr
Simulation results were in good agreement with these theoretical findings.
% and also demonstrated that 
% simulation results showed that 
%\textit{Turing Learning} clearly outperformes the metric-based system identification method.
% (which differed by using the least square metric instead of classifiers) in terms of the obtained model accuracy. 

\textcolor{black}{The classifiers generated by \textit{Turing Learning} can be a useful by-product.} 
Given a data sample (motion trajectory), \textcolor{black}{they can tell whether it is genuine, in other words, whether it originates from the reference system}. Such classifiers could \textcolor{black}{be used for} detecting abnormal behavior,---for example when faults occur in some members of the swarm---and are obtained without the need to define a priori what constitutes abnormal behavior. 

%\textcolor{black}{\st{A scalability study showed that the interactions in a swarm can be characterized by the effects on a subset of agents. In other words, when learning swarm behaviors especially with large number of agents, instead of considering the motion of all the agents in the group, we could focus on a subset of agents. This becomes critical when the available data about agents in the swarm is limited. Our approach was proven to work even if using only the motion data of a single agent and replica, as the data from this agent implicitly contained enough information about the interactions in the swarm.}} 
%\todo{Yet \textit{Turing Learning}, the interaction between members in the group or between members and the environment are implicitly considered by the classifiers, as the behavior is a result of such interactions.}

\textcolor{black}{In this paper, we presented the main results using a gray box model representation; in other words, the model structure was assumed to be known.} Consequently, the inferred model parameters could be compared against the ground truth, enabling us to objectively gauge the quality of the identification process.
%\todo{inferred} models in the two case studies as well as for $1000$ randomly sampled behaviors, and have a proof-of-concept test in both simulation and reality. 
Note that even though the search space for the models is small, identifying the parameters is challenging as the input values are unknown (consequently, the metric-based system identification method did not succeed in this respect). 
%Parameter estimation also plays an important part in the modeling of biological systems, where  the model structure is often assumed to be known~\citep{Gautrais:PloS:2012}.

The~\textit{Turing Learning} method was further validated using a physical system. We applied it to automatically infer the aggregation behavior of an observed swarm of e-puck robots. The behavior was learned successfully, and the results obtained in the physical experiments showed good correspondence with those obtained in simulation. This shows the robustness of our method with respect to noise and uncertainties typical of real-world experiments. To the best of our knowledge, this is also the first time that a system identification method was used to infer the behavior of \textcolor{black}{physical robots in a swarm}.

%how \textcolor{black}{one aspect of the agent's} morphology (their field of view) and brain (controller) can be simultaneously inferred

\textcolor{black}{We conducted further simulation experiments to test the generality of \textit{Turing Learning}. 
First, we showed that \textit{Turing Learning} can simultaneously infer the agent's brain (controller) as well as an aspect of the agent's morphology that determines its field of view. Second, we showed that \textit{Turing Learning} can infer the behavior without assuming the agent's control system structure to be known. The models were represented as fixed-structure recurrent neural networks, and the behavior could still be successfully inferred. For more complex behaviors, one could adopt other optimization algorithms such as NEAT~\citep{Stanley2002}, which gradually increases the complexity of the neural networks being evolved.}
%In another scenario, we have shown that \textit{Turing Learning} can also infer the morphology of the swarming agent through evolving its angle of view.}
%The results of learning the agent's angle of view showed that our method may even learn the morphology of the swarming agents. 
%We tested the generality of \textit{Turing Learning}. 
\textcolor{black}{Third, we assessed an alternative setup of \textit{Turing Learning}, in which
%for inferring agent behavior, where 
the replica---the robot used to test the models---is not in the same environment as the swarm of agents. While this requires a swarm of replicas, it has the advantage of ensuring that the agents are not affected by the replica's presence. In addition, it opens up the possibility of the replica not being a physical agent, but rather residing in a simulated world, which may lead to a less costly implementation. On the other hand, the advantage of using a physical replica is that it may help to address the reality gap issue. As the replica shares the same physics as the agent, its evolved behavior will rely on the same physics. This is not necessarily true for a simulated replica---for instance, when evolving a simulated fish it is hard (and computationally expensive) to fully capture the hydrodynamics of the real environment. As a final experiment, we showed that \textit{Turing Learning} is able to infer a wide range of randomly generated reactive behaviors.} 

%\textcolor{black}{In this paper, we provided two approaches for inferring agent behaviors: one that mixes the replicas into the agent group and the other that separates the replicas and agents. In general, using computer-based (non-robotic) implementation to separate the replicas and agents is preferable, as it is less costly than building and deploying a robot. The advantage of mixing the robot into the agent group is that the reality gap could be reduced as compared to computer-based simulation. The first approach is more commonly used in studying swarm behaviors.}  

In the future, we intend to use \textit{Turing Learning} to infer the complex behaviors exhibited in natural swarms, such as in shoals of fish or herds of land mammals.
We are interested in both reactive and non-reactive behaviors.
%Moreover, we intend to use \textit{Turing Learning} to infer non-reactive behaviors. 
As shown in~\citep{Li-etal2013:proc_gecco,LiThesis2016}, it can be beneficial if the classifiers are not restricted to observing the system passively. Rather, they could influence the process by which data samples are obtained, effectively choosing the conditions under which the system is to be observed.

%to infer the behavior of agents with internal memory.

%The method could also be extended that the classifiers have assess to information about the individuals' context and may even exert control over environmental parameters of the swarm~\citep{Li-etal2013:proc_gecco}.
%When the behaviors become more complex, instead of analyzing only the motion of individual agents, more information (such as number of the agent's neighbors or its internal states) may need to be provided to the classifiers. 
%In the two case studies presented, the model was explicitly represented by a set of parameters. The \todo{inferred} parameters could thus be compared against the ground truth, enabling us to objectively gauge the quality of \todo{inferred} models in the two case studies as well as for $1000$ randomly sampled behaviors. Although the search space for the models is relatively small, identifying the parameters is challenging as the input values are unknown \todo{(consequently, a metric-based evolutionary algorithm did not succeed in approximating the parameter values)}.
%for comparison of outputs using the traditional method does not lead to desirable results as shown in Section~\ref{sec:metric-based_EA}. 
%Parameter estimation plays an important part in the modeling of biological systems, as the model structure is often assumed to be known~\citep{Gautrais:PloS:2012}.

\section*{Acknowledgments}
%TODO - put back
%This research was funded by the Engineering and Physical Sciences
%Research Council (grant no. EP/K033948/1) and a Marie Curie European
%Reintegration Grant within the 7th European Community Framework
%Programme (grant no. PERG07-GA-2010-267354). \todo{
The authors are grateful for the support received by Jianing Chen, especially in relation to the physical implementation of \textit{Turing Learning}. The authors also thank the anonymous referees for their helpful comments.
%}
%\end{acknowledgements}

%\bibliographystyle{abbrvnat}      % spbasic: basic style, author-year citations 
%\bibliographystyle{spbasic}     % mathematics and  physical sciences  
%\bibliographystyle{spmpsci}     % mathematics and physical sciences
\bibliographystyle{ifacconf}       % APS-like style for physics
\bibliography{references}   % name your BibTeX data base

%\clearpage
%
%%Note: all texts should be in red color
%\appendix
%\section{\textcolor{black}{Appendix}}\label{appendix:a}
%\begin{table}[!h]
%  \centering
%  \caption{\textcolor{black}{A summary of different definitions in Section \ref{sec:analysis_of_evolved_classifiers_simulation}.}}
%  \label{tab:classifier_analysis_definition} 
%  \renewcommand{\arraystretch}{1.5}
%\begin{tabular}{l|c}
%	\hline
%	Category & Definition \\ \hline
%    \multirow{2}{*}{best classifier (subjective)} & \multirow{2}{*}{\parbox{7cm}{The classifier that has the highest fitness in a given generation.}} \\
%    & \\ \hline
%    \multirow{2}{*}{best classifier (archive)} & \multirow{2}{*}{\parbox{7cm}{The classifier that has the highest fitness when evaluated on the whole historical tracking data.}} \\
%    & \\ \hline
%    \multirow{2}{*}{best classifier (objective)} & \multirow{2}{*}{\parbox{7cm}{The classifier that has the highest fitness when evaluated on a range of models.}} \\
%    & \\ \hline
%   	\multirow{2}{*}{classifier group (subjective)} & \multirow{2}{*}{\parbox{7cm}{A classifier group that consists of $n$ classifiers with the highest fitness in a given generation.}} \\ 
%    & \\ \hline
%    \multirow{3}{*}{classifier group (archive)} & \multirow{3}{*}{\parbox{7cm}{A classifier group that consists of $n$ classifiers with the highest fitness when evaluated on the whole historical tracking data.}} \\ \\
%    & \\ \hline
%\end{tabular}
%\end{table}

\end{document}